\pdfoutput=1
\documentclass{amsart}
\usepackage{amssymb,amsfonts,amsmath}
\usepackage{array}
\usepackage{graphicx}
\usepackage[hidelinks]{hyperref}

\newtheorem{theorem}{Theorem}[section]
\newtheorem{proposition}[theorem]{Proposition}
\newtheorem{lemma}[theorem]{Lemma}
\newtheorem{corollary}[theorem]{Corollary}
\newtheorem{conjecture}[theorem]{Conjecture}
\theoremstyle{definition}
\newtheorem{definition}[theorem]{Definition}
\newtheorem{example}[theorem]{Example}

\theoremstyle{remark}
\newtheorem{remark}[theorem]{Remark}
\numberwithin{equation}{section}

\newcommand{\Zp}{\mathbb{Z}_{p}}

\newcommand{\R}{\mathbb{R}}
\newcommand{\bh}{\boldsymbol{h}}
\newcommand{\bW}{\boldsymbol{W}}
\newcommand{\bx}{\boldsymbol{x}}
\newcommand{\by}{\boldsymbol{y}}
\newcommand{\bu}{\boldsymbol{u}}
\newcommand{\bxi}{\boldsymbol{\xi}}
\newcommand{\btheta}{\boldsymbol{\theta}}
\newcommand{\D}{\mathcal{D}}
\newcommand{\X}{\mathcal{X}}
\newcommand{\Lip}{\mathrm{Lip}_{0}(\mathbb{R})}
\newcommand{\Win}{\boldsymbol{W}_{\mathrm{in}}}
\newcommand{\Wout}{\boldsymbol{W}_{\mathrm{out}}}
\newcommand{\xiout}{\boldsymbol{\xi}_{\mathrm{out}}}
\newcommand{\T}{\boldsymbol{T}}
\newcommand{\opnorm}[1]{\left\Vert #1\right\Vert }

\newcommand{\Om}[2]{\Omega\left(p^{#1}\left\vert #2\right\vert _{p}\right)}

\begin{document}
\emergencystretch=3em

\title[Critical Organization of $p$-Adic DNNs]{Critical Organization of $p$-Adic Deep Neural Networks}
\author{W. A. Z\'{u}\~{n}iga-Galindo}
\address{University of Texas Rio Grande Valley\\
School of Mathematical \& Statistical Sciences\\
One West University Blvd\\
Brownsville, TX 78520, United States}
\email{wilson.zunigagalindo@utrgv.edu}
\thanks{The author was partially supported by the Debnath Endowed Professorship in Mathematical Sciences.}

\begin{abstract}
We study continuous neural networks whose neurons are the points of the ring of $p$-adic integers, a fractal tree-like space that encodes hierarchical architectures. The hidden states solve a nonlinear integral equation, which describes the equilibria of recurrent networks and deep equilibrium models; feedforward architectures appear as the special case of kernels without feedback. We prove that the state is unique and depends Lipschitz-continuously on the parameters whenever the product of the Lipschitz constant of the activation function and the operator norm of the connectivity is less than one, and that, when the parameters are locally constant, every state is exactly the state of a finite network with a tree-like topology. We introduce $p$-adic discrete deep networks, whose layers are the levels of the tree and whose weights are coarse-grainings of a single kernel, and show that, in the stable region, they converge to the continuous network as the depth tends to infinity. Feedforward networks always have a unique state; hence a multiplicity of states, and therefore critical organization in our sense, requires recurrence. For recurrent networks with radial kernels we prove a sharp transition: the zero state is the only state if and only if the total coupling does not exceed one. For hierarchical kernels, above the threshold, the states are indexed by sign patterns on the leaves of the tree, their number is computed exactly, and they form a tree-indexed partially ordered set graded by instability. Numerical experiments illustrate these results and test conjectures on general radial kernels and on random recurrent networks on the tree. Finally, we interpret the critical organization hypothesis for deep neural networks, originally formulated for brain organization, through the lens of multistability, and we discuss the continuous limits of families of finite networks.
\end{abstract}

\maketitle

\section{Introduction}

\subsection{Motivation}
There are two central hypotheses about the brain's functioning. The first is that the brain has a hierarchical organization; see, e.g., \cite{Brain-orga-1,Brain-orga-2}. The second is that it operates near a phase transition, where ordered (synchronized) and disordered (random) states balance; this is the critical brain hypothesis. Since artificial deep neural networks were originally inspired by the structure and function of the brain, analogous questions have been raised for artificial networks, both for the dynamics of signal propagation in deep networks \cite{Poole et al}, \cite{Schoenholz-DIP} and for the self-organization of trained networks \cite{Criticall-0,Criticall-1,Criticall-2}.

Rigorous infinite-width limits of standard deep neural networks (DNNs) are known in two regimes. Under the standard scaling $W_{ik}\sim\mathcal{N}(0,g^{2}/n)$, randomly initialized networks converge to Gaussian processes \cite{Lee et al}, \cite{Matthews et al}, \cite{Neal}, \cite{Yang}; under the mean-field scaling $1/n$, the empirical distribution of neurons converges to a deterministic limit \cite{Mei et al}, \cite{Sirignano-Spiliopoulos}. In neither regime does the hierarchical organization of the neurons play a role. The purpose of this paper is to study a class of continuous networks in which it does: the neurons are the points of the ring $\Zp$ of $p$-adic integers, where $p$ is a fixed prime number; the finite truncations $G_{l}=\Zp/p^{l}\Zp$ of $\Zp$ are rooted trees.

\subsection{\label{Section_model}The model}
We consider the integral equation
\begin{equation}
\bh(x)=\int_{\Zp}\bW(x,y)\phi(\bh(y))\,dy+\int_{\Zp}\Win(x,y)\bx(y)\,dy+\bxi(x),\qquad x\in\Zp, \label{Eq-B}
\end{equation}
with output $\by(x)=\int_{\Zp}\Wout(x,y)\varphi(\bh(y))\,dy+\xiout(x)$, where $\bW,\Win,\Wout\in L^{2}(\Zp\times\Zp)$, $\bW(x,y)$ is the weight of the connection from the neuron $y$ to the neuron $x$, $\bxi,\xiout\in L^{2}(\Zp)$, and $\phi,\varphi$ are Lipschitz with $\phi(0)=\varphi(0)=0$. Equation (\ref{Eq-B}) is a neural field equation on an ultrametric space \cite{Amari}, \cite{Veltz-Faugeras}. Its solutions are the stationary states of $p$-adic cellular neural networks \cite{Zambrano-Zuniga-1,Zambrano-Zuniga-2}, \cite{Zuniga-2}, and functionals of the type (\ref{Eq-B}) appear in the Hamiltonians of several types of statistical field theories attached to neural networks \cite{Zuniga et al,Zuniga-1}. Equation (\ref{Eq-B}) also describes the equilibria of a recurrent network, or the output of a deep equilibrium model \cite{Bai et al}, \cite{Winston-Kolter}, whose neurons form a continuum indexed by the points of $\Zp$; the neurons are thus organized in a tree, but the actual connections between them are determined by the kernel $\bW$, so the network has a tree-like structure. When $\Zp$ is partitioned into open compact sets $U_{0},\ldots,U_{\Delta}$ and $\bW$ is supported on $\bigsqcup_{0\leq j<k\leq\Delta}U_{k}\times U_{j}$, so that the neurons of each $U_{k}$ receive connections only from the neurons of the previous sets $U_{0},\ldots,U_{k-1}$, (\ref{Eq-B}) describes a \emph{feedforward} network with input layer $U_{0}$, hidden layers $U_{1},\ldots,U_{\Delta-1}$ and output layer $U_{\Delta}$. It is \emph{layered} if, moreover, the neurons of each $U_{k}$ receive connections only from those of $U_{k-1}$. Otherwise the network has feedback, that is, connections inside a layer, from a neuron to itself, or from a layer to an earlier one, and we call it \emph{recurrent}; see Section \ref{Section_6}.

We stress that (\ref{Eq-B}) is a \emph{model}: we do not claim that standard DNNs with independent weights converge to (\ref{Eq-B}) as their width grows. Section \ref{Section_continuous_limits} explains why this cannot hold for the Hilbert--Schmidt topology under the standard scaling, and discusses continuous limits of families of networks.

\subsection{\label{Section_terminology}Terminology}
Before discussing the main results and the numerical simulations, we briefly discuss the terminology used in this paper. A \emph{$p$-adic continuous DNN} is the network (\ref{Eq-B}) (Definition \ref{Definition_DNN}); when no confusion arises, we simply say \emph{$p$-adic DNN}, or \emph{continuous network}. We use the term DNN also for recurrent kernels because, in the stable region, the state of (\ref{Eq-B}) is the limit of the hidden states of $p$-adic discrete DNNs as the depth tends to infinity (Theorem \ref{Theorem_depth}), as for deep equilibrium models. A $p$-adic continuous DNN is \emph{feedforward} if its kernel $\bW$ has the support described in Section \ref{Section_model} for some partition of $\Zp$, \emph{layered} if, moreover, only consecutive layers are connected, and \emph{recurrent} if it is not feedforward, that is, if it has feedback (Section \ref{Section_6}). Feedforward networks may have skip connections, as residual networks do; layered networks may not. A \emph{finite $p$-adic DNN}, or simply a \emph{finite network}, is a $p$-adic continuous DNN whose parameters are test functions; its states are the solutions of a finite system on a tree $G_{l}$ (Remark \ref{Theorem_exact}). A \emph{$p$-adic discrete DNN} is a deep feedforward network whose layers are the levels of the tree (the connections between the neurons are determined by the kernels) and whose weights are coarse-grainings of a single kernel (Definition \ref{Definition-3}). We reserve the terms \emph{standard DNN} and \emph{standard RNN} for the networks of the machine-learning literature, recalled in (\ref{Net-1}).

\subsection{\label{Section_main_results}Main results}
Theorem \ref{Theorem_A} shows that, if $q:=L_{\phi}\opnorm{\T_{\bW}}<1$, where $\T_{\bW}$ is the integral operator with kernel $\bW$, then (\ref{Eq-B}) has a unique state in $L^{2}(\Zp)$, and by Theorem \ref{Theorem_continuity} this state depends locally Lipschitz-continuously on all parameters, including the activation function. If the parameters are test functions, every state of (\ref{Eq-B}) lies in a finite-dimensional space $\D^{l}(\Zp)$ and is exactly a state of a finite $p$-adic DNN on the tree $G_{l}$ (Remark \ref{Theorem_exact}); consequently, states for general parameters are limits of states of finite $p$-adic DNNs (Remark \ref{Theorem_exact}(iv)).

On the stable region $\X_{\mathrm{stable}}$, formed by the parameters with $q<1$, the network computes a function of its input, robustly with respect to the input and the parameters, with constants containing the factor $(1-q)^{-1}$, which diverges at the boundary $q=1$; this is the analogue of the divergence of the susceptibility at a critical point (Theorem \ref{Theorem_continuity}, Corollary \ref{Cor_robust}).

By \emph{critical organization} we mean a transition, as a parameter crosses a threshold, from a unique state to several states. The condition $q<1$ guarantees the existence and uniqueness of the state, but $q\geq1$ does not guarantee a bifurcation: Lemma \ref{Prop_counterexample} gives families of continuous networks that have exactly one state for every value of $q$. Thus crossing $q=1$ need not create new states, and conditions on the class of kernels are necessary. Feedforward networks, in particular layered ones, always have exactly one state (Proposition \ref{Theorem_feedforward}), so critical organization requires recurrence (Corollary \ref{Cor_feedforward}); the radial kernels studied in Section \ref{Section_7} are recurrent (Lemma \ref{Lemma_nilpotent}).

For activation functions such as $\tanh$ (odd, bounded, strictly concave on $[0,\infty)$, with $\phi^{\prime}(0)=1$), zero input, and a nonnegative radial kernel $\bW(x,y)=w(|x-y|_{p})$, Theorem \ref{Theorem_radial} shows that the state $\bh=0$ is unique if and only if $\lambda=\int_{\Zp}w(|y|_{p})dy\leq1$; for $\lambda>1$ there are at least three states. For the hierarchical kernels of Theorem \ref{Theorem_hierarchical}, above the threshold and for a weak global coupling between all neurons, there are exactly $3^{p^{l}}$ states, indexed by the sign patterns on the leaves of the tree $G_{l}=\Zp/p^{l}\Zp$, and exactly $2^{p^{l}}$ of them are stable. In Theorems \ref{Theorem_toy} and \ref{Theorem_order}, a corrected description of the states of the $p$-adic cellular edge detector of \cite{Zambrano-Zuniga-2} is presented.

In this framework, the hierarchical and the critical organizations are connected in a precise sense: above the threshold, the states of hierarchical recurrent networks are indexed by the leaves of a tree, and their order by instability is a tree-indexed product poset (Theorems \ref{Theorem_hierarchical} and \ref{Theorem_order}). In Section \ref{Section_brain}, we use these results to visualize the critical brain hypothesis from the point of view of multistability: below the critical point there is a unique attractor, at the critical point the relaxation time and the susceptibility diverge, and above it the network stores patterns in many attractors, organized hierarchically for hierarchical kernels.

\emph{Numerical experiments.} In Section \ref{Section_experiments}, we illustrate these results on finite $p$-adic DNNs. For hierarchical kernels, an exact enumeration confirms the counts of Theorem \ref{Theorem_hierarchical} and shows that they hold for a much larger coupling between all neurons than the one covered by the proof. For general radial kernels, the experiments show that the hypothesis of Conjecture \ref{Conj_tree} on the coupling between neurons in different balls is essential. For random recurrent networks on the tree, the number of equilibria grows above the threshold, in agreement with Conjecture \ref{Conj_random}; with symmetric couplings the dynamics converges to equilibria, whereas with asymmetric couplings it may approach non-stationary attractors. Finally, we compute the relaxation time, the susceptibility and the number of stable states across the critical point (Figure \ref{Fig_regimes}).

\subsection{Organization}
Section \ref{Section_related} reviews related work. Section \ref{Section_3} collects $p$-adic preliminaries. Section \ref{Section_4} proves existence, uniqueness and continuity of states, and describes the attractors in the stable region. Section \ref{Section_5} introduces the $p$-adic discrete DNNs, proves their infinite-depth limit, and gives the exact finite-dimensional reduction. Section \ref{Section_6} proves uniqueness for feedforward networks and gives a criterion for recurrence. Section \ref{Section_8} describes critical organization in the edge detector of \cite{Zambrano-Zuniga-2}, and introduces the order on sign patterns that organizes the states; Section \ref{Section_7} studies critical organization for recurrent networks; and Section \ref{Section_experiments} presents numerical experiments. Section \ref{Section_conjectures} discusses continuous limits of finite networks, the place of $p$-adic DNNs among standard DNNs and RNNs, and the relation with the $p$-adic Buice--Cowan theory \cite{Zuniga-BC}, and states conjectures and open problems.

\section{\label{Section_related}Related work}

\subsection{Infinite-width limits and criticality}
For randomly initialized networks with weights of variance $g^{2}/n$, the pre-activations of a wide network converge to a Gaussian process \cite{Lee et al,Matthews et al}, \cite{Neal}; the result extends to essentially all architectures, including recurrent ones \cite{Yang}. Under the mean-field scaling, the training dynamics of wide two-layer networks converges to a deterministic evolution of a measure \cite{Mei et al}, \cite{Sirignano-Spiliopoulos}. Effective field theory corrections in $1/n$ are developed in \cite{Roberts et al}.

Criticality of deep networks at initialization is understood through signal propagation: the correlation map between two inputs has a fixed point whose stability is governed by $\chi_{1}=\sigma_{w}^{2}\mathbb{E}[\phi^{\prime}(h)^{2}]$, and $\chi_{1}=1$ separates an ordered from a chaotic phase \cite{Poole et al}, \cite{Schoenholz-DIP}. For random recurrent networks $\dot{h}=-h+M\phi(h)$ with $M_{ij}\sim\mathcal{N}(0,g^{2}/n)$ and $\phi=\tanh$, the transition to chaos occurs at $g=1$ \cite{Sompolinsky et al}, and the number of equilibria grows exponentially in $n$ beyond it \cite{Wainrib-Touboul}. These thresholds depend on the derivative of $\phi$ along the signal and on the spectrum of the weights, not on a global Lipschitz constant; the threshold of Theorem \ref{Theorem_A} is a sufficient condition of a different nature (Remark \ref{Remark_chi}).

\subsection{$p$-Adic networks and neural fields}
Ultrametricity entered physics through spin glasses and protein folding; see \cite[Chapter 4]{KKZuniga}, \cite{Spin-Glasses}. There are two types of $p$-adic NNs: with real-valued states $\bh:\Zp\rightarrow\R$, as here, and with $p$-adic-valued states, which are $p$-adic dynamical systems \cite{Albeverio et al,Khrennikov-Nilson,Khrennikov-Tirozzi}. Networks of the first type include the $p$-adic cellular neural networks \cite{Zambrano-Zuniga-1}, \cite{Zambrano-Zuniga-2}, \cite{Zuniga-2}, generalizing those of Chua and Yang \cite{Chua-Tamas}, \cite{Slavova}, and a $p$-adic Wilson--Cowan model \cite{Zuniga-Entropy}. In \cite{Non-ARCH-DNNs}, $p$-adic continuous perceptrons $Y=\sigma(\T_{w}X+\theta)$ and their discrete layered versions are introduced; for a fixed input, the output of a discrete network can approximate any function in $L^{\rho}$, uniformly for weights and biases in a ball \cite[Theorem 6.1]{Non-ARCH-DNNs}.

Equation (\ref{Eq-B}) on a Euclidean domain is a stationary neural field equation \cite{Amari}; Theorem \ref{Theorem_A} is the $p$-adic analogue of the uniqueness result of Veltz and Faugeras \cite{Veltz-Faugeras}, who also studied the bifurcations past the threshold (see Remark \ref{Remark_sufficient}(ii)). What is specific to the $p$-adic setting is that the spaces $\D^{l}(\Zp)$ of locally constant functions are exactly invariant under integral operators with locally constant kernels (Remark \ref{Theorem_exact}), and that radial kernels act diagonally along the tree (Section \ref{Section_7}).

\subsection{Attractors and network geometry}
The idea that attractors organize neural computation is classical \cite{Brain-1,IEEE}. In \cite{Villegas}, embeddings of complex networks built from eigenfunctions of discrete Laplacians reveal hidden geometries that emerge at the transition from chaos to order; discrete network Laplacians are $p$-adic operators \cite{Laplacian1}, \cite{Laplacian2}.

\section{\label{Section_3}$p$-Adic numbers and tree-like structures}

\subsection{The field of $p$-adic numbers}

In this section, we collect some basic results on $p$-adic analysis that we
use throughout the article. For a detailed exposition, the reader may consult
\cite{A-K-S}, \cite{Taibleson}, \cite{V-V-Z}, \cite{Zuniga-Textbook}.

We use $p$ to denote a fixed prime number. Any nonzero $p$-adic number $x$ has
a unique expansion
\[
x=x_{-k}p^{-k}+x_{-k+1}p^{-k+1}+\ldots+x_{0}+x_{1}p+\ldots,
\]
with $x_{-k}\neq0$, $k\in\mathbb{Z}$, and digits $x_{j}\in\{0,1,\ldots,p-1\}$.
These series form the field $\mathbb{Q}_{p}$. The norm is $|x|_{p}=p^{k}$ for
$x$ as above, and $|0|_{p}=0$; it satisfies the ultrametric inequality
$|x+y|_{p}\leq\max\{|x|_{p},|y|_{p}\}$, with equality for $|x|_{p}\neq|y|_{p}
$. The $p$-adic integers, the series with $k\leq0$, form the unit ball
$\mathbb{Z}_{p}$, an infinite rooted tree with fractal structure. As a
topological space $\mathbb{Q}_{p}$\ is homeomorphic to a Cantor-like subset of
the real line; see, e.g., \cite{A-K-S}, \cite{V-V-Z}.

\subsection{The space of test functions and the Lebesgue spaces}

For $r\in\mathbb{Z}$, denote by $B_{r}(a)=\{x\in\mathbb{Q}_{p};|x-a|_{p}\leq
p^{r}\}$ the ball of radius $p^{r}$ with center at $a\in\mathbb{Q}_{p}$. Two
balls in $\mathbb{Q}_{p}$ are either disjoint or one is contained in the
other. A subset of $\mathbb{Q}_{p}$ is compact if and only if it is closed and
bounded in $\mathbb{Q}_{p}$, see e.g., \cite[Section 1.8]{A-K-S} or
\cite[Section 1.3]{V-V-Z}. Since $(\mathbb{Q}_{p},+)$ is a locally compact
topological group, there exists a Haar measure $dx$, which is invariant under
translations, i.e., $d(x+a)=dx$. If we normalize this measure by the condition
$\int_{\mathbb{Z}_{p}}dx=1$, then $dx$ is unique. We denote by $\Omega\left(
p^{l}\left\vert x-a\right\vert _{p}\right)  $ the characteristic function of
the ball $a+p^{l}\mathbb{Z}_{p}$, so that
\[
\int_{\mathbb{Z}_{p}}\Omega\left(  p^{l}\left\vert x-I\right\vert _{p}\right)
\,dx=p^{-l}.
\]

A function $\varphi:\mathbb{Q}_{p}\rightarrow\mathbb{R}$ is called a test
function (or a Bruhat--Schwartz function) if it is a linear combination of
characteristic functions of balls. In this paper, we work exclusively with
test functions supported in the unit ball. We denote by $\mathcal{D}
(\mathbb{Z}_{p})$ the real vector space of test functions supported in
$\mathbb{Z}_{p}$.

The space $\mathcal{D}(\mathbb{Z}_{p})$ of test functions is dense in
\[
L^{\rho}\left(  \mathbb{Z}_{p}\right)  =\left\{  \varphi:\mathbb{Z}
_{p}\rightarrow\mathbb{R};\left\Vert \varphi\right\Vert _{\rho}=\left\{
{\displaystyle\int\limits_{\mathbb{Z}_{p}}}
\left\vert \varphi\left(  x\right)  \right\vert ^{\rho}dx\right\}
^{\frac{1}{\rho}}<\infty\right\}  ,
\]
for $1\leq\rho<\infty$, see, e.g., \cite[Section 4.3]{A-K-S}. We need only the
spaces $L^{2}\left(  \mathbb{Z}_{p}\right)  \subset L^{1}\left(
\mathbb{Z}_{p}\right)  $.

\subsection{Tree-like structures}

For $l\geq1$, $G_{l}:=\mathbb{Z}_{p}/p^{l}\mathbb{Z}_{p}\simeq\mathbb{Z}
/p^{l}\mathbb{Z}$, with representatives $I=I_{0}+I_{1}p+\ldots+I_{l-1}p^{l-1}
$, $I_{j}\in\{0,\ldots,p-1\}$. The points of $G_{l}$ are the leaves of a
rooted tree with $l+1$ levels, the root being at level $0$ and the leaves at
level $l$; $G_{0}=\{0\}$. For a non-negative integer $m$ and a positive
integer $l$, we denote by $\Lambda_{l,m}$ the group homomorphism (reduction
modulo $p^{m}$)
\[
\begin{array}
[c]{cccc}
\Lambda_{l,m}: & G_{l} & \rightarrow & G_{m}\\
& I & \mapsto & I\bmod p^{m}\mathbb{Z}_{p},
\end{array}
\]
where
\[
I\bmod p^{m}\mathbb{Z}_{p}=\left\{
\begin{array}
[c]{lll}
I_{0}+I_{1}p+\ldots+I_{m-1}p^{m-1} & \text{if} & l>m\geq1,\\
I & \text{if} & m\geq l\geq1,\\
0 & \text{if} & m=0.
\end{array}
\right.
\]
For $l>m$, $\Lambda_{l,m}(I)$ is the ancestor at level $m$ of the leaf $I$; in
particular, $\Lambda_{l,l-1}(I_{0}+\ldots+I_{l-1}p^{l-1})=I_{0}+\ldots
+I_{l-2}p^{l-2}$ sends a vertex to its parent. Furthermore, $\Lambda
_{m,k}\circ\Lambda_{l,m}=\Lambda_{l,k}$ for $l\geq m\geq k$.

On the other hand, for each $z\in\{0,1,\ldots,p-1\}$, the map
\[
\begin{array}
[c]{ccc}
G_{l} & \rightarrow & G_{l+1}\\
a_{0}+\ldots+a_{l-1}p^{l-1} & \mapsto & a_{0}+\ldots+a_{l-1}p^{l-1}+zp^{l}
\end{array}
\]
identifies the rooted tree $G_{l}$ with a subtree of $G_{l+1}$; it is a
section of $\Lambda_{l+1,l}$ but it is not a group homomorphism. These
identifications show that $G_{l+1}$ is a tree-like structure, i.e., a rooted
tree made of trees. These $p$-adic tree-like structures have been used to
construct hierarchical versions of deep Boltzmann networks \cite{Zuniga et
al,Zuniga-1,Zuniga-DBNs} and of cellular neural networks
\cite{Zambrano-Zuniga-1,Zambrano-Zuniga-2}.

\subsection{\label{Section_Dl}The spaces $\mathcal{D}^{l}$}

For $l\geq0$, $\mathcal{D}^{l}(\mathbb{Z}_{p})$ denotes the space spanned by
$\{\Omega\left(  p^{l}\left\vert x-I\right\vert _{p}\right)  \}_{I\in G_{l}}$,
the functions constant on the balls of radius $p^{-l}$; then $\dim
\mathcal{D}^{l}(\mathbb{Z}_{p})=p^{l}$, $\mathcal{D}^{l}(\mathbb{Z}
_{p})\subset\mathcal{D}^{l+1}(\mathbb{Z}_{p})$, and $\mathcal{D}
(\mathbb{Z}_{p})=\bigcup_{l}\mathcal{D}^{l}(\mathbb{Z}_{p})$. Every ball
$I+p^{f}\mathbb{Z}_{p}$ is the disjoint union of the balls $K+p^{m}
\mathbb{Z}_{p}$ with $K\in G_{m}$, $\Lambda_{m,f}(K)=I$, for $m\geq f$:
\begin{equation}
\Omega\left(  p^{f}\left\vert x-I\right\vert _{p}\right)  =\sum_{K\in
G_{m},\ \Lambda_{m,f}(K)=I}\Omega\left(  p^{m}\left\vert x-K\right\vert
_{p}\right)  . \label{Identity-Basic}
\end{equation}
Similarly, $\mathcal{D}^{l}(\mathbb{Z}_{p}\times\mathbb{Z}_{p})$ is spanned by
$\Omega\left(  p^{l}\left\vert x-I\right\vert _{p}\right)  \Omega\left(
p^{l}\left\vert y-J\right\vert _{p}\right)  $, $I,J\in G_{l}$, and
$\mathcal{D}(\mathbb{Z}_{p}\times\mathbb{Z}_{p})=\bigcup_{l}\mathcal{D}
^{l}(\mathbb{Z}_{p}\times\mathbb{Z}_{p})$ is dense in $L^{2}(\mathbb{Z}
_{p}\times\mathbb{Z}_{p})$. For $f\in\mathcal{D}^{l}(\mathbb{Z}_{p})$ we write
$f(I)$ for its value on $I+p^{l}\mathbb{Z}_{p}$, and for $F\in\D^{l}(\Zp\times\Zp)$ we write $F(I,K)$ for its value on $(I+p^{l}\Zp)\times(K+p^{l}\Zp)$. Throughout the paper, the elements of $\D^{l}(\Zp)$ and $\D^{l}(\Zp\times\Zp)$ are written in the basis of characteristic functions of balls, whose coefficients are these values:
\begin{equation}
\begin{gathered}
f(x)=\sum_{I\in G_{l}}f(I)\,\Om{l}{x-I},\\
F(x,y)=\sum_{I,K\in G_{l}}F(I,K)\,\Om{l}{x-I}\Om{l}{y-K}.
\end{gathered} \label{Expansion_values}
\end{equation}
In the orthonormal basis $e_{I}(x)=p^{l/2}\Om{l}{x-I}$ of Lemma \ref{Lemma_ON} below, the coefficients are $p^{-l/2}f(I)$ and $p^{-l}F(I,K)$, respectively:
\[
f=\sum_{I\in G_{l}}p^{-l/2}f(I)\,e_{I},\qquad F(x,y)=\sum_{I,K\in G_{l}}p^{-l}F(I,K)\,e_{I}(x)e_{K}(y).
\]
We use the values, and not the coefficients in the orthonormal basis, because a function $\phi:\R\rightarrow\R$ acts on $\D^{l}(\Zp)$ through the values: $(\phi\circ f)(I)=\phi(f(I))$, whereas in general $\phi(p^{-l/2}f(I))\neq p^{-l/2}\phi(f(I))$. All the finite networks of this paper are written in terms of values.

We denote by $\boldsymbol{P}_{l}$ the orthogonal projection of $L^{2}
(\mathbb{Z}_{p})$ onto $\mathcal{D}^{l}(\mathbb{Z}_{p})$:
\begin{equation}
\boldsymbol{P}_{l}f(x)=p^{l}\int_{x+p^{l}\mathbb{Z}_{p}}f(y)\,dy=\int
_{\mathbb{Z}_{p}}p^{l}\Omega\left(  p^{l}\left\vert x-y\right\vert
_{p}\right)  f(y)\,dy,\label{Projection}
\end{equation}
the average of $f$ over the ball of radius $p^{-l}$ containing $x$. It is an
integral operator with kernel $p^{l}\Omega\left(  p^{l}\left\vert
x-y\right\vert _{p}\right)  $. Now, if $x\in I+p^{l}\mathbb{Z}_{p}$, then
\[
\boldsymbol{P}_{l}f(x)=p^{l}\int_{I+p^{l}\mathbb{Z}_{p}}f(y)\,dy=:c_{I},
\]
and thus
\[
\boldsymbol{P}_{l}f(x)=\sum_{I\in G_{l}}c_{I}\,\Omega\left(  p^{l}\left\vert
x-I\right\vert _{p}\right)  .
\]

Throughout the paper, operators are denoted by boldface letters, and for a bounded linear operator $\boldsymbol{A}$ on $L^{2}(\mathbb{Z}_{p})$, $\Vert\boldsymbol{A}\Vert$ denotes its operator norm; for functions and kernels we write $\Vert\cdot\Vert_{2}$ for the $L^{2}$ norm.

For $f\in L^{1}(\Zp)$ and $F\in L^{1}(\Zp\times\Zp)$ we use the approximations by averages
\begin{equation}
f^{(l)}(x)=\sum_{I\in G_{l}}c_{I}\,\Om{l}{x-I},\qquad c_{I}=p^{l}\int_{I+p^{l}\Zp}f(y)\,dy,\label{Average_1}
\end{equation}
\begin{multline}
F^{(l)}(x,y)=\sum_{I,J\in G_{l}}C_{I,J}\,\Om{l}{x-I}\Om{l}{y-J},\\
C_{I,J}=p^{2l}\iint_{(I+p^{l}\Zp)\times(J+p^{l}\Zp)}F(u,v)\,du\,dv,\label{Average_2}
\end{multline}
which belong to $\D^{l}(\Zp)$ and $\D^{l}(\Zp\times\Zp)$, respectively; in the notation (\ref{Expansion_values}), $c_{I}=f^{(l)}(I)$ and $C_{I,J}=F^{(l)}(I,J)$. Here it is crucial to recall that $L^{2}(\Zp)\hookrightarrow L^{1}(\Zp)$ and $L^{2}(\Zp\times\Zp)\hookrightarrow L^{1}(\Zp\times\Zp)$, where the arrow denotes a continuous embedding. Thus $f^{(l)}=\boldsymbol{P}_{l}f$, and $F^{(l)}=(\boldsymbol{P}_{l}\otimes\boldsymbol{P}_{l})F$ is the orthogonal projection of $F\in L^{2}(\Zp\times\Zp)$ onto $\D^{l}(\Zp\times\Zp)$. If $f\in\D^{m}(\Zp)$ and $l\geq m$, then $f^{(l)}=f$, and similarly for $F$; and $(f^{(l)})^{(k)}=f^{(k)}$, $(F^{(l)})^{(k)}=F^{(k)}$ for $k\leq l$. Moreover, for $F\in L^{2}(\Zp\times\Zp)$ the integral operators satisfy
\begin{equation}
\T_{F^{(l)}}=\boldsymbol{P}_{l}\T_{F}\boldsymbol{P}_{l}. \label{Projected_operator}
\end{equation}

\begin{lemma}
\label{Lemma_Pl}For every $f\in L^{2}(\mathbb{Z}_{p})$, $f^{(l)}=\boldsymbol{P}_{l}f\rightarrow f$ in $L^{2}(\mathbb{Z}_{p})$, and for every $F\in L^{2}(\Zp\times\Zp)$, $F^{(l)}\rightarrow F$ in $L^{2}(\Zp\times\Zp)$, as $l\rightarrow\infty$.
\end{lemma}

\begin{proof}
Let $\epsilon>0$. Since $\mathcal{D}(\mathbb{Z}_{p})=\bigcup_{m}
\mathcal{D}^{m}(\mathbb{Z}_{p})$ is dense in $L^{2}(\mathbb{Z}_{p})$, there
are $m$ and $g\in\mathcal{D}^{m}(\mathbb{Z}_{p})$ with $\Vert f-g\Vert
_{2}<\epsilon$. For $l\geq m$, $g\in\mathcal{D}^{l}(\mathbb{Z}_{p})$, so
$\boldsymbol{P}_{l}g=g$, and since $\Vert\boldsymbol{P}_{l}\Vert=1$,
\[
\Vert\boldsymbol{P}_{l}f-f\Vert_{2}\leq\Vert\boldsymbol{P}_{l}(f-g)\Vert
_{2}+\Vert g-f\Vert_{2}\leq2\epsilon.
\]
The same argument, using the density of $\D(\Zp\times\Zp)=\bigcup_{m}\D^{m}(\Zp\times\Zp)$ in $L^{2}(\Zp\times\Zp)$, applies to $F$.
\end{proof}

\begin{remark}
The averages (\ref{Average_1})--(\ref{Average_2}) also converge almost everywhere: if $f\in L^{1}(\Zp)$, then $c_{I}\rightarrow f(x)$ as $l\rightarrow\infty$, where $I\in G_{l}$ is determined by $x\in I+p^{l}\Zp$, for almost every $x$, as a consequence of the Hardy--Littlewood maximal theorem; see \cite[Chapter II, Theorem 1.14]{Taibleson}. The same holds in several variables. The approximations (\ref{Average_1})--(\ref{Average_2}) give an explicit construction of approximations by test functions, valid for all $L^{2}$ functions, and not only for continuous ones.
\end{remark}

\begin{lemma}
\label{Lemma_ON}Let $l\geq0$. The functions $e_{I}(x)=p^{l/2}\Omega\left(
p^{l}\left\vert x-I\right\vert _{p}\right)  $, $I\in G_{l}$, form an
orthonormal basis of $\mathcal{D}^{l}(\mathbb{Z}_{p})$. If $\boldsymbol{W}
\in\mathcal{D}^{l}(\mathbb{Z}_{p}\times\mathbb{Z}_{p})$, the integral operator
$\boldsymbol{T}_{\boldsymbol{W}}f(x)=\int_{\mathbb{Z}_{p}}\boldsymbol{W}
(x,y)f(y)\,dy$ maps $L^{2}(\mathbb{Z}_{p})$ into $\mathcal{D}^{l}
(\mathbb{Z}_{p})$, and its matrix on $\mathcal{D}^{l}(\mathbb{Z}_{p})$ in the
basis $\{e_{I}\}$ is
\[
M_{I,K}=p^{-l}\boldsymbol{W}(I,K),\qquad I,K\in G_{l}.
\]
In particular, $\Vert\boldsymbol{W}\Vert_{L^{2}(\mathbb{Z}_{p}\times
\mathbb{Z}_{p})}^{2}=\sum_{I,K}M_{I,K}^{2}$ is the squared Frobenius norm of
$M$, and $\left\Vert \boldsymbol{T}_{\boldsymbol{W}}\right\Vert =\Vert
M\Vert$, the norm of $M$ as a linear map of $\mathbb{R}^{p^{l}}$ with the Euclidean norm. Since $e_{I}$ and $\Om{l}{x-I}$ differ by the factor $p^{l/2}$, which does not depend on $I$, $M$ is also the matrix of $\T_{\bW}$ in the basis $\{\Om{l}{x-I}\}_{I\in G_{l}}$, that is, $(\T_{\bW}f)(I)=\sum_{K\in G_{l}}M_{I,K}f(K)$ for $f\in\D^{l}(\Zp)$.
\end{lemma}

\begin{proof}
The balls $I+p^{l}\mathbb{Z}_{p}$ and $K+p^{l}\mathbb{Z}_{p}$ are disjoint for $I\neq K$, and the measure of each of them is $p^{-l}$; then $p^{l}\int\Omega\left(  p^{l}\left\vert x-I\right\vert
_{p}\right)  \Omega\left(  p^{l}\left\vert x-K\right\vert _{p}\right)
dx=\delta_{I,K}$, that is, $\langle e_{I},e_{K}\rangle=\delta_{I,K}$. Since the $p^{l}$ functions $e_{I}$ span $\D^{l}(\Zp)$, they form an orthonormal basis. Writing, by (\ref{Expansion_values}), $\boldsymbol{W}(x,y)=\sum_{I,K}\boldsymbol{W}
(I,K)\Omega\left(  p^{l}\left\vert x-I\right\vert _{p}\right)  \Omega\left(
p^{l}\left\vert y-K\right\vert _{p}\right)  $, we get $\boldsymbol{T}
_{\boldsymbol{W}}f(x)=\sum_{I}\bigl(\sum_{K}\boldsymbol{W}(I,K)\int
_{K+p^{l}\mathbb{Z}_{p}}f\bigr)\Omega\left(  p^{l}\left\vert x-I\right\vert
_{p}\right)  \in\mathcal{D}^{l}(\mathbb{Z}_{p})$, and $\boldsymbol{T}
_{\boldsymbol{W}}e_{K}=\sum_{I}p^{-l}\boldsymbol{W}(I,K)e_{I}$. The norm
identity is $\Vert\boldsymbol{W}\Vert_{2}^{2}=\sum_{I,K}\boldsymbol{W}
(I,K)^{2}p^{-2l}$. Since $\boldsymbol{T}_{\boldsymbol{W}}$ vanishes on the
orthogonal complement of $\mathcal{D}^{l}(\mathbb{Z}_{p})$ (each
$\int_{K+p^{l}\mathbb{Z}_{p}}f$ vanishes there), its operator norm equals that
of $M$.
\end{proof}

\begin{lemma}
\label{Lemma-1A}Let $l\geq1$, $\boldsymbol{W}\in\mathcal{D}^{l}(\mathbb{Z}
_{p}\times\mathbb{Z}_{p})$ and $\boldsymbol{h}\in\mathcal{D}^{m}
(\mathbb{Z}_{p})$ with $m\leq l$. Then for every $\phi:\mathbb{R}
\rightarrow\mathbb{R}$,
\[
\int_{\mathbb{Z}_{p}}\boldsymbol{W}(x,y)\phi(\boldsymbol{h}(y))\,dy=\sum_{I\in
G_{l}}\Bigl(\sum_{K\in G_{l}}p^{-l}\boldsymbol{W}(I,K)\,\phi
\bigl(\boldsymbol{h}(\Lambda_{l,m}(K))\bigr)\Bigr)\Omega\left(  p^{l}
\left\vert x-I\right\vert _{p}\right)  .
\]

\end{lemma}

\begin{proof}
By (\ref{Identity-Basic}), $\phi(\boldsymbol{h}(y))=\sum_{K\in G_{l}}
\phi(\boldsymbol{h}(\Lambda_{l,m}(K)))\Omega\left(  p^{l}\left\vert
y-K\right\vert _{p}\right)  $. Substituting the expansion of $\boldsymbol{W}$
and using $\int_{\mathbb{Z}_{p}}\Omega\left(  p^{l}\left\vert y-K\right\vert
_{p}\right)  dy=p^{-l}$ gives the formula.
\end{proof}

\begin{lemma}
\label{Lemma-2A}Let $\boldsymbol{W}_{\mathrm{in}}\in\mathcal{D}^{L}
(\mathbb{Z}_{p}\times\mathbb{Z}_{p})$, $\boldsymbol{x}\in L^{2}(\mathbb{Z}
_{p})$, and set $\bar{\boldsymbol{x}}(K)=p^{L}\int_{K+p^{L}\mathbb{Z}_{p}
}\boldsymbol{x}(y)dy$ for $K\in G_{L}$. Then
\[
\int_{\mathbb{Z}_{p}}\boldsymbol{W}_{\mathrm{in}}(x,y)\boldsymbol{x}
(y)\,dy=\sum_{I\in G_{L}}\Bigl(\sum_{K\in G_{L}}p^{-L}\boldsymbol{W}
_{\mathrm{in}}(I,K)\bar{\boldsymbol{x}}(K)\Bigr)\Omega\left(  p^{L}\left\vert
x-I\right\vert _{p}\right)  ,
\]
and for every $\Delta\geq0$ this function belongs to $\mathcal{D}
^{L}(\mathbb{Z}_{p})\subset\mathcal{D}^{L+\Delta}(\mathbb{Z}_{p})$; its value
at $R\in G_{L+\Delta}$ is the value at $\Lambda_{L+\Delta,L}(R)$.
\end{lemma}

\begin{proof}
As in the proof of Lemma \ref{Lemma_ON}, followed by (\ref{Identity-Basic}).
\end{proof}

\section{\label{Section_4}$p$-Adic continuous DNNs}

\begin{definition}
We denote by $\Lip$ the real Banach space of functions $\omega:\R\rightarrow\R$ with $\omega(0)=0$ and
\[
L_{\omega}:=\sup_{s\neq t}\frac{|\omega(s)-\omega(t)|}{|s-t|}<\infty,
\]
with norm $\Vert\omega\Vert_{\mathrm{Lip}}=L_{\omega}$; see \cite[Chapter 1]{Weaver}. An activation function is an element of $\Lip$. Examples are $\tanh$, $\mathrm{ReLU}$, and the saturating function $\phi(s)=\frac{1}{2}(|s+1|-|s-1|)$, all with $L_{\phi}=1$.
\end{definition}

Since $|\omega(s)|\leq L_{\omega}|s|$, the operator $\bh\mapsto\omega\circ\bh$ maps $L^{2}(\Zp)$ into itself and satisfies
\begin{equation}
\Vert\omega\circ\bh_{1}-\omega\circ\bh_{2}\Vert_{2}\leq L_{\omega}\Vert\bh_{1}-\bh_{2}\Vert_{2},\qquad\Vert\omega_{1}\circ\bh-\omega_{2}\circ\bh\Vert_{2}\leq\Vert\omega_{1}-\omega_{2}\Vert_{\mathrm{Lip}}\Vert\bh\Vert_{2}. \label{Nemytskii}
\end{equation}

\begin{definition}
\label{Definition_DNN}Let
\[
\btheta=\{\phi,\varphi,\Win,\Wout,\bW,\bxi,\xiout\}\in\X:=\Lip^{2}\times L^{2}(\Zp\times\Zp)^{3}\times L^{2}(\Zp)^{2},
\]
with the product topology. A \emph{$p$-adic continuous DNN} with parameters $\btheta$ assigns to an input $\bx\in L^{2}(\Zp)$ the set of its \emph{hidden states} (pre-activations), the solutions $\bh\in L^{2}(\Zp)$ of
\begin{equation}
\bh(x)=\int_{\Zp}\bW(x,y)\phi(\bh(y))\,dy+\int_{\Zp}\Win(x,y)\bx(y)\,dy+\bxi(x), \label{State_equation}
\end{equation}
and to each hidden state $\bh$ the output
\begin{equation}
\by(x)=\int_{\Zp}\Wout(x,y)\varphi(\bh(y))\,dy+\xiout(x). \label{output}
\end{equation}
For brevity, we call the hidden states simply \emph{states}. Throughout the paper, $\bW(x,y)$ is the weight of the connection from the neuron $y$ to the neuron $x$: the first variable is the neuron that receives the signal and the second one the neuron that sends it, as for the weight matrices of standard networks, whose entry $(i,j)$ is the weight from neuron $j$ to neuron $i$. The same convention applies to $\Win$ and $\Wout$.
We write $\T_{\bW}$ for the integral operator with kernel $\bW$ and $\bu:=\T_{\Win}\bx+\bxi$, so that (\ref{State_equation}) reads $\bh=\T_{\bW}(\phi\circ\bh)+\bu$.
\end{definition}

\begin{remark}
Equation (\ref{State_equation}) describes the equilibria of the dynamics $\partial_{t}\bh=-\bh+\T_{\bW}(\phi\circ\bh)+\bu$, which is a $p$-adic neural field. When the kernel is feedforward, in particular when it is layered, the network has no feedback; otherwise it is recurrent (Section \ref{Section_6}). Taking kernels depending on $x-y$ or on $|x-y|_{p}$ gives convolutional networks.
\end{remark}

\begin{lemma}
\label{Lemma-0A}For $\bW\in L^{2}(\Zp\times\Zp)$, $\T_{\bW}$ is a Hilbert--Schmidt operator on $L^{2}(\Zp)$ and $\Vert\T_{\bW}\Vert\leq\Vert\bW\Vert_{2}$. For $\bh\in L^{2}(\Zp)$ and $\omega\in\Lip$,
\[
\Vert\T_{\bW}(\omega\circ\bh)\Vert_{2}\leq L_{\omega}\Vert\T_{\bW}\Vert\,\Vert\bh\Vert_{2}\leq L_{\omega}\Vert\bW\Vert_{2}\Vert\bh\Vert_{2}.
\]
\end{lemma}

\begin{proof}
The first assertion follows from the Cauchy--Schwarz inequality; for the fact that an integral operator with a square-integrable kernel is Hilbert--Schmidt, with Hilbert--Schmidt norm $\Vert\bW\Vert_{2}$, see, e.g., \cite[Theorem VI.23]{Reed-Simon}. The first inequality in the display follows from (\ref{Nemytskii}), using $\omega(0)=0$. The second follows from the Cauchy--Schwarz inequality as follows:
\begin{gather*}
\Vert\T_{\bW}(\omega\circ\bh)\Vert_{2}^{2}=\int_{\Zp}\left\vert\int_{\Zp}\bW(x,y)\omega(\bh(y))\,dy\right\vert^{2}dx\\
\leq\int_{\Zp}\left\{\int_{\Zp}\left\vert\bW(x,y)\right\vert\,\left\vert\omega(\bh(y))\right\vert\,dy\right\}^{2}dx\leq L_{\omega}^{2}\int_{\Zp}\left\{\int_{\Zp}\left\vert\bW(x,y)\right\vert\,\left\vert\bh(y)\right\vert\,dy\right\}^{2}dx\\
\leq L_{\omega}^{2}\left\{\int_{\Zp}\int_{\Zp}\left\vert\bW(x,y)\right\vert^{2}dx\,dy\right\}\int_{\Zp}\left\vert\bh(y)\right\vert^{2}dy.
\end{gather*}
\end{proof}

\begin{theorem}
\label{Theorem_A}Let $\btheta\in\X$ and $\bx\in L^{2}(\Zp)$, and assume
\begin{equation}
q:=L_{\phi}\opnorm{\T_{\bW}}<1. \label{Stability}
\end{equation}
Then (\ref{State_equation}) has a unique state $\bh\in L^{2}(\Zp)$, and the Picard iterates $\bh_{n+1}=\T_{\bW}(\phi\circ\bh_{n})+\bu$ converge to it from any $\bh_{0}\in L^{2}(\Zp)$ with
\[
\Vert\bh_{n}-\bh\Vert_{2}\leq\frac{q^{n}}{1-q}\Vert\bh_{1}-\bh_{0}\Vert_{2};
\]
moreover,
\begin{equation}
\Vert\bh\Vert_{2}\leq\frac{\opnorm{\T_{\Win}}\Vert\bx\Vert_{2}+\Vert\bxi\Vert_{2}}{1-q}. \label{Bound_h}
\end{equation}
In particular this holds if $L_{\phi}\Vert\bW\Vert_{2}<1$.
\end{theorem}

\begin{proof}
For $\btheta\in\X$ and $\bx\in L^{2}(\Zp)$ we define the map $\Phi_{\btheta,\bx}:L^{2}(\Zp)\rightarrow L^{2}(\Zp)$ by
\begin{equation}
\Phi_{\btheta,\bx}(\bh)=\T_{\bW}(\phi\circ\bh)+\T_{\Win}\bx+\bxi, \label{Fixed_point_map}
\end{equation}
so that the states are the fixed points of $\Phi_{\btheta,\bx}$. By Lemma \ref{Lemma-0A}, $\Phi_{\btheta,\bx}$ maps $L^{2}(\Zp)$ into itself, and by (\ref{Nemytskii})
\begin{align*}
\Vert\Phi_{\btheta,\bx}(\bh_{1})-\Phi_{\btheta,\bx}(\bh_{2})\Vert_{2} & =\Vert\T_{\bW}(\phi\circ\bh_{1}-\phi\circ\bh_{2})\Vert_{2}\\
 & \leq\opnorm{\T_{\bW}}\Vert\phi\circ\bh_{1}-\phi\circ\bh_{2}\Vert_{2}\leq q\Vert\bh_{1}-\bh_{2}\Vert_{2}.
\end{align*}
The Banach fixed point theorem gives the existence and uniqueness of $\bh$, and the convergence of the Picard iterates $\bh_{n+1}=\Phi_{\btheta,\bx}(\bh_{n})$ to $\bh$ for every $\bh_{0}\in L^{2}(\Zp)$. The estimate for $\Vert\bh_{n}-\bh\Vert_{2}$ follows from the contraction property: since $\bh_{k+1}-\bh_{k}=\Phi_{\btheta,\bx}(\bh_{k})-\Phi_{\btheta,\bx}(\bh_{k-1})$, we have $\Vert\bh_{k+1}-\bh_{k}\Vert_{2}\leq q\Vert\bh_{k}-\bh_{k-1}\Vert_{2}\leq q^{k}\Vert\bh_{1}-\bh_{0}\Vert_{2}$, hence, for $m>n$,
\[
\Vert\bh_{m}-\bh_{n}\Vert_{2}\leq\sum_{k=n}^{m-1}\Vert\bh_{k+1}-\bh_{k}\Vert_{2}\leq\sum_{k=n}^{m-1}q^{k}\Vert\bh_{1}-\bh_{0}\Vert_{2}\leq\frac{q^{n}}{1-q}\Vert\bh_{1}-\bh_{0}\Vert_{2},
\]
and letting $m\rightarrow\infty$ gives the estimate. The bound (\ref{Bound_h}) follows from $\Vert\bh\Vert_{2}\leq\Vert\Phi_{\btheta,\bx}(\bh)-\Phi_{\btheta,\bx}(0)\Vert_{2}+\Vert\Phi_{\btheta,\bx}(0)\Vert_{2}\leq q\Vert\bh\Vert_{2}+\Vert\bu\Vert_{2}$, since $\Phi_{\btheta,\bx}(0)=\bu$.
\end{proof}

\begin{remark}
\label{Remark_sufficient}(i) Condition (\ref{Stability}) is sufficient, but not necessary, for uniqueness; see Lemma \ref{Prop_counterexample}.

(ii) Comparison with \cite{Veltz-Faugeras}. Veltz and Faugeras study the stationary equation $V=J\cdot S(\lambda V)+I_{\mathrm{ext}}$ of Euclidean neural fields on a bounded domain $\Omega\subset\R^{d}$, $d\leq3$, where $S$ is a smooth sigmoid with values in $(0,1)$ and $\lambda\geq0$ is its slope; the connectivity $J$ and the solutions belong to Sobolev spaces. They prove that a stationary solution exists for every $\lambda$, by Leray--Schauder degree, and that it is unique if $\lambda\Vert J\Vert_{\mathcal{F}}<1$, where $\Vert\cdot\Vert_{\mathcal{F}}$ is a Sobolev norm, by the Picard theorem \cite[Proposition 2.9]{Veltz-Faugeras}. Theorem \ref{Theorem_A} is the analogue of the uniqueness part on $\Zp$, with three differences. First, here the activation is only assumed Lipschitz with $\phi(0)=0$, so unbounded activations such as the ReLU are included, while in \cite[Proposition 2.9]{Veltz-Faugeras} the sigmoid is bounded. Second, no regularity of the kernel is required, only $\bW\in L^{2}(\Zp\times\Zp)$; on $\Zp$ the finite-dimensional reductions come from the invariance of the spaces $\D^{l}(\Zp)$ (Remark \ref{Theorem_exact}), not from Sobolev embeddings. Third, the condition uses the operator norm, and on a Euclidean domain $\Vert\T_{J}\Vert\leq\Vert J\Vert_{L^{2}}\leq\Vert J\Vert_{\mathcal{F}}$, so (\ref{Stability}) is weaker than the Hilbert--Schmidt condition $L_{\phi}\Vert\bW\Vert_{2}<1$ and than the Sobolev condition of \cite{Veltz-Faugeras}. The difference can be large: the kernel $\bW(x,y)=\lambda p^{l}\Om{l}{x-y}$ of $\lambda\boldsymbol{P}_{l}$, $\lambda>0$, satisfies $\Vert\T_{\bW}\Vert=\lambda$ and $\Vert\bW\Vert_{2}=\lambda p^{l/2}$, so, for $\phi=\tanh$, the Hilbert--Schmidt condition guarantees uniqueness only for $\lambda<p^{-l/2}$, which tends to zero as $l\rightarrow\infty$, while (\ref{Stability}) gives $\lambda<1$, which is sharp by Theorem \ref{Theorem_radial}. Moreover, under (\ref{Stability}) the state is globally exponentially stable (Proposition \ref{Prop_attractor}), whereas in \cite{Veltz-Faugeras} global stability is obtained under a different condition on the spectral radius of the symmetric part of $J$ \cite[Proposition 2.8]{Veltz-Faugeras}.
\end{remark}

We set
\[
\X_{\mathrm{stable}}:=\{\btheta\in\X;\ L_{\phi}\opnorm{\T_{\bW}}<1\}.
\]

\begin{lemma}
\label{Lemma_open}The map $q:\X\rightarrow[0,\infty)$, $q(\btheta)=L_{\phi}\Vert\T_{\bW}\Vert$, is continuous. Consequently, $\X_{\mathrm{stable}}=q^{-1}([0,1))$ is open in $\X$.
\end{lemma}

\begin{proof}
Let $\btheta,\btheta^{\prime}\in\X$, with activation functions $\phi,\phi^{\prime}$ and recurrent kernels $\bW,\bW^{\prime}$, respectively. Since $L_{\phi}=\Vert\phi\Vert_{\mathrm{Lip}}$ is a norm, $|L_{\phi^{\prime}}-L_{\phi}|\leq\Vert\phi^{\prime}-\phi\Vert_{\mathrm{Lip}}$. By Lemma \ref{Lemma-0A}, applied to $\bW^{\prime}-\bW$,
\[
\bigl|\Vert\T_{\bW^{\prime}}\Vert-\Vert\T_{\bW}\Vert\bigr|\leq\Vert\T_{\bW^{\prime}}-\T_{\bW}\Vert=\Vert\T_{\bW^{\prime}-\bW}\Vert\leq\Vert\bW^{\prime}-\bW\Vert_{2}.
\]
Hence
\begin{align*}
|q(\btheta^{\prime})-q(\btheta)| & \leq|L_{\phi^{\prime}}-L_{\phi}|\,\Vert\T_{\bW^{\prime}}\Vert+L_{\phi}\,\bigl|\Vert\T_{\bW^{\prime}}\Vert-\Vert\T_{\bW}\Vert\bigr|\\
 & \leq\Vert\phi^{\prime}-\phi\Vert_{\mathrm{Lip}}\bigl(\Vert\bW\Vert_{2}+\Vert\bW^{\prime}-\bW\Vert_{2}\bigr)+L_{\phi}\Vert\bW^{\prime}-\bW\Vert_{2}.
\end{align*}
The right-hand side tends to zero when $\btheta^{\prime}\rightarrow\btheta$ in the product topology of $\X$, since this convergence implies $\Vert\phi^{\prime}-\phi\Vert_{\mathrm{Lip}}\rightarrow0$ and $\Vert\bW^{\prime}-\bW\Vert_{2}\rightarrow0$. Thus $q$ is continuous, and $\X_{\mathrm{stable}}$ is the preimage of the open set $[0,1)$ of $[0,\infty)$, hence an open subset of $\X=\Lip^{2}\times L^{2}(\Zp\times\Zp)^{3}\times L^{2}(\Zp)^{2}$.
\end{proof}

For $\btheta\in\X_{\mathrm{stable}}$ and $\bx\in L^{2}(\Zp)$, we denote by $\bh(\btheta,\bx)$ the unique state corresponding to $\btheta$ and $\bx$, and by $\by(\btheta,\bx)$ the corresponding unique output.

\begin{theorem}
\label{Theorem_continuity}The maps $(\btheta,\bx)\mapsto\bh(\btheta,\bx)$ and $(\btheta,\bx)\mapsto\by(\btheta,\bx)$ are locally Lipschitz from $\X_{\mathrm{stable}}\times L^{2}(\Zp)$ into $L^{2}(\Zp)$. More precisely, if $\btheta,\btheta^{\prime}\in\X_{\mathrm{stable}}$ and $q^{\prime}=L_{\phi^{\prime}}\opnorm{\T_{\bW^{\prime}}}$, then
\begin{multline*}
(1-q^{\prime})\Vert\bh(\btheta^{\prime},\bx^{\prime})-\bh(\btheta,\bx)\Vert_{2}\leq\bigl(L_{\phi^{\prime}}\Vert\bW^{\prime}-\bW\Vert_{2}+\Vert\bW\Vert_{2}\Vert\phi^{\prime}-\phi\Vert_{\mathrm{Lip}}\bigr)\Vert\bh(\btheta,\bx)\Vert_{2}\\
+\Vert\Win^{\prime}-\Win\Vert_{2}\Vert\bx^{\prime}\Vert_{2}+\Vert\Win\Vert_{2}\Vert\bx^{\prime}-\bx\Vert_{2}+\Vert\bxi^{\prime}-\bxi\Vert_{2}.
\end{multline*}
\end{theorem}

\begin{proof}
Let $\bh=\bh(\btheta,\bx)$ and $\bh^{\prime}=\bh(\btheta^{\prime},\bx^{\prime})$, so that $\bh=\Phi_{\btheta,\bx}(\bh)$ and $\bh^{\prime}=\Phi_{\btheta^{\prime},\bx^{\prime}}(\bh^{\prime})$, see (\ref{Fixed_point_map}). Adding and subtracting $\Phi_{\btheta^{\prime},\bx^{\prime}}(\bh)$,
\[
\bh^{\prime}-\bh=\bigl[\Phi_{\btheta^{\prime},\bx^{\prime}}(\bh^{\prime})-\Phi_{\btheta^{\prime},\bx^{\prime}}(\bh)\bigr]+\bigl[\Phi_{\btheta^{\prime},\bx^{\prime}}(\bh)-\Phi_{\btheta,\bx}(\bh)\bigr].
\]
In the first bracket both terms are computed with the same parameters $\btheta^{\prime}$ and the same input $\bx^{\prime}$, so the input and bias terms $\T_{\Win^{\prime}}\bx^{\prime}+\bxi^{\prime}$ cancel:
\[
\Phi_{\btheta^{\prime},\bx^{\prime}}(\bh^{\prime})-\Phi_{\btheta^{\prime},\bx^{\prime}}(\bh)=\T_{\bW^{\prime}}(\phi^{\prime}\circ\bh^{\prime}-\phi^{\prime}\circ\bh),
\]
whose norm is at most $q^{\prime}\Vert\bh^{\prime}-\bh\Vert_{2}$ by Lemma \ref{Lemma-0A} and (\ref{Nemytskii}). The second bracket contains all the differences between the parameters and between the inputs:
\begin{multline*}
\Phi_{\btheta^{\prime},\bx^{\prime}}(\bh)-\Phi_{\btheta,\bx}(\bh)=\T_{\bW^{\prime}}(\phi^{\prime}\circ\bh)-\T_{\bW}(\phi\circ\bh)+\T_{\Win^{\prime}}\bx^{\prime}-\T_{\Win}\bx+\bxi^{\prime}-\bxi\\
=\T_{\bW^{\prime}-\bW}(\phi^{\prime}\circ\bh)+\T_{\bW}(\phi^{\prime}\circ\bh-\phi\circ\bh)+\T_{\Win^{\prime}-\Win}\bx^{\prime}+\T_{\Win}(\bx^{\prime}-\bx)+\bxi^{\prime}-\bxi.
\end{multline*}
Therefore
\[
(1-q^{\prime})\Vert\bh^{\prime}-\bh\Vert_{2}\leq\Vert\Phi_{\btheta^{\prime},\bx^{\prime}}(\bh)-\Phi_{\btheta,\bx}(\bh)\Vert_{2},
\]
and estimating each of the five terms of the second bracket by Lemma \ref{Lemma-0A} and (\ref{Nemytskii}) gives the inequality in the statement. Since $q^{\prime}=q(\btheta^{\prime})$ depends continuously on $\btheta^{\prime}$ (Lemma \ref{Lemma_open}) and $\Vert\bh(\btheta,\bx)\Vert_{2}$ is locally bounded by (\ref{Bound_h}), the estimate gives local Lipschitz continuity of $\bh$. The output map is then locally Lipschitz by the same estimates applied to (\ref{output}).
\end{proof}

\begin{corollary}
\label{Cor_robust}Let $\btheta\in\X_{\mathrm{stable}}$ and $q=L_{\phi}\Vert\T_{\bW}\Vert$.
\begin{enumerate}
\item[(i)] For all inputs $\bx,\bx^{\prime}\in L^{2}(\Zp)$,
\begin{align*}
\Vert\bh(\btheta,\bx^{\prime})-\bh(\btheta,\bx)\Vert_{2} & \leq\frac{\Vert\T_{\Win}\Vert}{1-q}\,\Vert\bx^{\prime}-\bx\Vert_{2},\\
\Vert\by(\btheta,\bx^{\prime})-\by(\btheta,\bx)\Vert_{2} & \leq\frac{L_{\varphi}\Vert\T_{\Wout}\Vert\,\Vert\T_{\Win}\Vert}{1-q}\,\Vert\bx^{\prime}-\bx\Vert_{2}.
\end{align*}
\item[(ii)] Let $(\bx_{k},\boldsymbol{f}_{k})\in L^{2}(\Zp)\times L^{2}(\Zp)$, $k=1,\ldots,K$, be a finite set of training data. The loss
\[
\mathcal{L}(\btheta)=\sum_{k=1}^{K}\Vert\by(\btheta,\bx_{k})-\boldsymbol{f}_{k}\Vert_{2}^{2}
\]
is locally Lipschitz on $\X_{\mathrm{stable}}$.
\end{enumerate}
\end{corollary}

\begin{proof}
(i) Write $\bh=\bh(\btheta,\bx)$ and $\bh^{\prime}=\bh(\btheta,\bx^{\prime})$. Subtracting the state equations,
\[
\bh^{\prime}-\bh=\T_{\bW}(\phi\circ\bh^{\prime}-\phi\circ\bh)+\T_{\Win}(\bx^{\prime}-\bx),
\]
so, by (\ref{Nemytskii}), $\Vert\bh^{\prime}-\bh\Vert_{2}\leq q\Vert\bh^{\prime}-\bh\Vert_{2}+\Vert\T_{\Win}\Vert\,\Vert\bx^{\prime}-\bx\Vert_{2}$, which gives the first inequality. The second follows from $\by(\btheta,\bx^{\prime})-\by(\btheta,\bx)=\T_{\Wout}(\varphi\circ\bh^{\prime}-\varphi\circ\bh)$ and (\ref{Nemytskii}).

(ii) By Theorem \ref{Theorem_continuity}, each map $\btheta\mapsto\by(\btheta,\bx_{k})$ is locally Lipschitz on $\X_{\mathrm{stable}}$. It is also locally bounded: by (\ref{output}), (\ref{Nemytskii}) and (\ref{Bound_h}),
\[
\Vert\by(\btheta,\bx_{k})\Vert_{2}\leq L_{\varphi}\Vert\Wout\Vert_{2}\frac{\Vert\Win\Vert_{2}\Vert\bx_{k}\Vert_{2}+\Vert\bxi\Vert_{2}}{1-q(\btheta)}+\Vert\xiout\Vert_{2},
\]
and the right-hand side is bounded on a neighborhood of each point of $\X_{\mathrm{stable}}$, because $q$ is continuous (Lemma \ref{Lemma_open}). The assertion follows from $\bigl|\Vert\boldsymbol{a}\Vert_{2}^{2}-\Vert\boldsymbol{b}\Vert_{2}^{2}\bigr|\leq(\Vert\boldsymbol{a}\Vert_{2}+\Vert\boldsymbol{b}\Vert_{2})\Vert\boldsymbol{a}-\boldsymbol{b}\Vert_{2}$, applied with $\boldsymbol{a}=\by(\btheta^{\prime},\bx_{k})-\boldsymbol{f}_{k}$ and $\boldsymbol{b}=\by(\btheta,\bx_{k})-\boldsymbol{f}_{k}$.
\end{proof}

\subsection{\label{Section_attractors}Attractors}
The states are the equilibria of the dynamics
\begin{equation}
\partial_{t}\bh=-\bh+\T_{\bW}(\phi\circ\bh)+\bu,\qquad\bh(0)=\bh_{0}\in L^{2}(\Zp), \label{Dynamics}
\end{equation}
whose right-hand side $F(\bh)$ is Lipschitz on $L^{2}(\Zp)$, with constant $1+q$. By the Picard--Lindel\"{o}f theorem in Banach spaces, (\ref{Dynamics}) has a unique solution $\bh\in C^{1}([0,\infty),L^{2}(\Zp))$, and the solution operators $S(t):\bh_{0}\mapsto\bh(t)$ form a continuous semiflow on $L^{2}(\Zp)$.

We recall the terminology we use; see, e.g., \cite{Hale}, \cite{Temam}. An \emph{equilibrium} is a point $\bh^{\ast}$ with $F(\bh^{\ast})=0$; the equilibria of (\ref{Dynamics}) are exactly the states of the network. An equilibrium is \emph{stable} (in the sense of Lyapunov) if solutions starting close to it remain close to it for all $t\geq0$, and \emph{asymptotically stable} if, in addition, they converge to it; the set of initial conditions whose solutions converge to $\bh^{\ast}$ is its \emph{basin of attraction}. It is \emph{exponentially stable} with rate $\gamma>0$ if $\Vert S(t)\bh_{0}-\bh^{\ast}\Vert_{2}\leq Ce^{-\gamma t}\Vert\bh_{0}-\bh^{\ast}\Vert_{2}$ for $\bh_{0}$ near $\bh^{\ast}$, and \emph{globally} exponentially stable if this holds for all $\bh_{0}\in L^{2}(\Zp)$. The number $1/\gamma$ is a \emph{relaxation time}. An \emph{attractor} is a compact invariant set that attracts a neighborhood of itself; the \emph{global attractor} is the compact invariant set that attracts every bounded subset of $L^{2}(\Zp)$. For finite networks, where $F$ is $C^{1}$ when $\phi$ is, an equilibrium is \emph{linearly stable} if the spectrum of the Jacobian $DF(\bh^{\ast})$ lies in $\{\mathrm{Re}\,z<0\}$, and \emph{hyperbolic} if no eigenvalue has zero real part; then the number of eigenvalues with positive real part is the dimension of its unstable manifold, and a hyperbolic equilibrium is asymptotically stable if and only if this number is zero. A \emph{bifurcation} is a change in the number or the stability type of the equilibria as a parameter varies.

\begin{proposition}
\label{Prop_attractor}Let $\btheta\in\X_{\mathrm{stable}}$, $\bx\in L^{2}(\Zp)$, $q=L_{\phi}\Vert\T_{\bW}\Vert$, and let $\bh^{\ast}=\bh(\btheta,\bx)$. Then, for every $\bh_{0}\in L^{2}(\Zp)$,
\[
\Vert S(t)\bh_{0}-\bh^{\ast}\Vert_{2}\leq e^{-(1-q)t}\,\Vert\bh_{0}-\bh^{\ast}\Vert_{2},\qquad t\geq0.
\]
Consequently, $\bh^{\ast}$ is globally exponentially stable with rate at least $1-q$, $\{\bh^{\ast}\}$ is the global attractor of (\ref{Dynamics}), and the relaxation time is at most $(1-q)^{-1}$. In particular, no bifurcation occurs in $\X_{\mathrm{stable}}$: for every $\btheta\in\X_{\mathrm{stable}}$ the global attractor is a single point, which depends locally Lipschitz-continuously on $(\btheta,\bx)$ by Theorem \ref{Theorem_continuity}.
\end{proposition}

\begin{proof}
Let $\bh(t)=S(t)\bh_{0}$ and $\boldsymbol{e}(t)=\bh(t)-\bh^{\ast}$. Since $F(\bh^{\ast})=0$,
\[
\frac{d}{dt}\Vert\boldsymbol{e}\Vert_{2}^{2}=2\langle\boldsymbol{e},F(\bh)-F(\bh^{\ast})\rangle=-2\Vert\boldsymbol{e}\Vert_{2}^{2}+2\langle\boldsymbol{e},\T_{\bW}(\phi\circ\bh-\phi\circ\bh^{\ast})\rangle\leq-2(1-q)\Vert\boldsymbol{e}\Vert_{2}^{2},
\]
by the Cauchy--Schwarz inequality, Lemma \ref{Lemma-0A} and (\ref{Nemytskii}). Gronwall's inequality gives the estimate. It implies that $S(t)B\rightarrow\{\bh^{\ast}\}$ uniformly for every bounded set $B$, and $\{\bh^{\ast}\}$ is compact and invariant, so it is the global attractor. The last assertion follows from Theorem \ref{Theorem_continuity}.
\end{proof}

The factor $(1-q)^{-1}$ thus plays two roles: it bounds the relaxation time of the dynamics (Proposition \ref{Prop_attractor}) and the sensitivity of the attractor to the input and the parameters (Theorem \ref{Theorem_continuity} and Corollary \ref{Cor_robust}). Corollary \ref{Cor_sharp} below shows that, for radial kernels, both bounds are attained at the threshold where the attractor bifurcates.

\begin{remark}
\label{Remark_robust}Theorem \ref{Theorem_continuity} and Corollary \ref{Cor_robust} have the following interpretation. In $\X_{\mathrm{stable}}$ the network computes a function of its input: each input determines exactly one state and one output, independently of the initial condition of the dynamics $\partial_{t}\bh=-\bh+\T_{\bW}(\phi\circ\bh)+\bu$, and small changes of the input or of the parameters produce proportionally small changes of the output and of the loss. Hence learning is well posed in $\X_{\mathrm{stable}}$. For the finite networks of Section \ref{Section_5} and $\phi\in C^{1}$, the state is moreover differentiable in the parameters by the implicit function theorem, since $\Vert\T_{\bW}D\phi(\bh)\Vert\leq q<1$, where $D\phi(\bh)$ is the multiplication by $\phi^{\prime}\circ\bh$; thus gradients can be computed by implicit differentiation through the fixed point, as for deep equilibrium models \cite{Bai et al}. The constants contain the factor $(1-q)^{-1}$, which diverges at the boundary $q=1$ of $\X_{\mathrm{stable}}$; this is the analogue of the divergence of the susceptibility at a critical point, and Corollary \ref{Cor_sharp} shows that the factor cannot be improved. Beyond the boundary, recurrent networks can have several states (Section \ref{Section_7}), and the network can store information in its attractors.
\end{remark}

\section{\label{Section_5}$p$-Adic discrete DNNs and exact discretization}

\subsection{$p$-Adic discrete DNNs}
For an integrable function $F\in L^{1}(\Zp\times\Zp)$, the approximations by averages (\ref{Average_1})--(\ref{Average_2}) provide an approximation $F^{(l)}$ obtained by averaging $F$ over the balls of radius $p^{-l}$; it is a coarse-grained version of $F$ at level $l$ of the tree. Furthermore, $F^{(l)}\in\D^{l}(\Zp\times\Zp)\subset\D^{L+\Delta}(\Zp\times\Zp)$ for $l\leq L+\Delta$, where $L$ and $\Delta$ are the positive integers of Definition \ref{Definition-3} below.

\begin{definition}
\label{Definition-3}Set
\[
\btheta=\left\{L,\Delta,\phi,\varphi,\Win^{(L)},\Wout^{(L+\Delta)},\bW^{(L+\Delta)},\bxi^{(L+\Delta)},\xiout^{(L+\Delta)}\right\},
\]
where $L$, $\Delta$ are positive integers, $\phi,\varphi\in\Lip$, $\Win^{(L)}\in\D^{L}(\Zp\times\Zp)$, $\Wout^{(L+\Delta)},\bW^{(L+\Delta)}\in\D^{L+\Delta}(\Zp\times\Zp)$, and $\bxi^{(L+\Delta)},\xiout^{(L+\Delta)}\in\D^{L+\Delta}(\Zp)$. A $p$-adic discrete DNN with parameters $\btheta$ is a dynamical system with input $\bx^{(L)}\in\D^{L}(\Zp)$ and output
\[
\by^{(L+\Delta)}(x)=\int_{\Zp}\Wout^{(L+\Delta)}(x,y)\varphi\bigl(\bh^{(L+\Delta)}(y)\bigr)\,dy+\xiout^{(L+\Delta)}(x).
\]
The hidden states (pre-activations) $\bh^{(l)}\in\D^{l}(\Zp)$, $L\leq l\leq L+\Delta$, are given by
\begin{align}
\bh^{(L)}(x) & =\int_{\Zp}\Win^{(L)}(x,y)\bx^{(L)}(y)\,dy+\bxi^{(L)}(x),\nonumber\\
\bh^{(l)}(x) & =\int_{\Zp}\bW^{(l)}(x,y)\phi\bigl(\bh^{(l-1)}(y)\bigr)\,dy+\int_{\Zp}\Win^{(L)}(x,y)\bx^{(L)}(y)\,dy+\bxi^{(l)}(x), \label{Formula_0}
\end{align}
for $L<l\leq L+\Delta$, where, for $l<L+\Delta$, $\bW^{(l)}\in\D^{l}(\Zp\times\Zp)$ and $\bxi^{(l)}\in\D^{l}(\Zp)$ denote the averages (\ref{Average_1})--(\ref{Average_2}) of $\bW^{(L+\Delta)}$ and $\bxi^{(L+\Delta)}$ at level $l$. The integer $L$ is determined by the input, and $\Delta$ is the depth of the network.
\end{definition}

All the functions in Definition \ref{Definition-3} are test functions, so all the integrals are finite sums (Lemmas \ref{Lemma-1A} and \ref{Lemma-2A}), and the hidden states are computed recursively; they are uniquely determined by the input and the parameters, without any condition on the size of the weights.

\begin{remark}
\label{Remark_perceptron}The discrete networks of \cite[Definition 4.1]{Non-ARCH-DNNs} are layered perceptrons in post-activation form,
\[
X^{(l)}(j)=\sigma_{M}\Bigl(\sum_{k\in G_{l}}w(j,k)X^{(l-1)}(\Lambda_{l,l-1}(k))+\theta(j)\Bigr),\qquad j\in G_{l},
\]
with input $X^{(L)}$ and output $X^{(L+\Delta)}$, where $\sigma_{M}$ is a bounded activation function. With the change of variables $\bh^{(l)}(j)=\sum_{k}w(j,k)X^{(l-1)}(\Lambda_{l,l-1}(k))+\theta(j)$, so that $X^{(l)}=\sigma_{M}(\bh^{(l)})$, they take the pre-activation form $\bh^{(l)}(j)=\sum_{k}w(j,k)\sigma_{M}(\bh^{(l-1)}(\Lambda_{l,l-1}(k)))+\theta(j)$ of (\ref{Formula_DNN}). The networks of Definition \ref{Definition-3} differ from them in four respects: the input term $\T_{\Win^{(L)}}\bx^{(L)}$ enters every layer; the output is a linear readout $\T_{\Wout}(\varphi\circ\bh^{(L+\Delta)})+\xiout$; the weights carry the factor $p^{-l}$ (mean-field scaling); and the kernels of the lower layers are averages of a single kernel, not its values at the points of $G_{l}$. The averages are what make the networks of increasing depth consistent (Remark \ref{Remark_single_kernel}).
\end{remark}

\begin{lemma}
\label{Lemma-6}With the notation of Definition \ref{Definition-3}, for $L<l\leq L+\Delta$ and $I\in G_{l}$,
\begin{multline}
\bh^{(l)}(I)=\sum_{K\in G_{l}}p^{-l}\bW^{(l)}(I,K)\,\phi\bigl(\bh^{(l-1)}(\Lambda_{l,l-1}(K))\bigr)\\
+\sum_{K\in G_{L}}p^{-L}\Win^{(L)}(\Lambda_{l,L}(I),K)\,\bx^{(L)}(K)+\bxi^{(l)}(I), \label{Formula_DNN}
\end{multline}
where $\bW^{(l)}(I,K)$ and $\bxi^{(l)}(I)$ are the values of $\bW^{(l)}\in\D^{l}(\Zp\times\Zp)$ and $\bxi^{(l)}\in\D^{l}(\Zp)$, in the notation (\ref{Expansion_values}), given by (\ref{Average_values}) below; moreover $\bh^{(L)}(I)=\sum_{K\in G_{L}}p^{-L}\Win^{(L)}(I,K)\bx^{(L)}(K)+\bxi^{(L)}(I)$ for $I\in G_{L}$, and
\[
\by^{(L+\Delta)}(I)=\sum_{K\in G_{L+\Delta}}p^{-L-\Delta}\Wout^{(L+\Delta)}(I,K)\,\varphi\bigl(\bh^{(L+\Delta)}(K)\bigr)+\xiout^{(L+\Delta)}(I)
\]
for $I\in G_{L+\Delta}$.
\end{lemma}

\begin{proof}
For $l\leq L+\Delta$, the kernel $\bW^{(l)}$ and the bias $\bxi^{(l)}$ are the averages (\ref{Average_1})--(\ref{Average_2}) of $\bW^{(L+\Delta)}$ and $\bxi^{(L+\Delta)}$ at level $l$, so they belong to $\D^{l}(\Zp\times\Zp)$ and $\D^{l}(\Zp)$. Their values can be computed from the values of $\bW^{(L+\Delta)}$ and $\bxi^{(L+\Delta)}$: by (\ref{Identity-Basic}), the ball $I+p^{l}\Zp$ is the disjoint union of the $p^{L+\Delta-l}$ balls $R+p^{L+\Delta}\Zp$ with $R\in G_{L+\Delta}$ and $\Lambda_{L+\Delta,l}(R)=I$, each of measure $p^{-L-\Delta}$, and $\bW^{(L+\Delta)}$, $\bxi^{(L+\Delta)}$ are constant on these balls and on their products. Hence, for $I,K\in G_{l}$,
\begin{equation}
\begin{gathered}
\bW^{(l)}(I,K)=p^{-2(L+\Delta-l)}\sum_{\substack{R,S\in G_{L+\Delta}\\ \Lambda_{L+\Delta,l}(R)=I,\ \Lambda_{L+\Delta,l}(S)=K}}\bW^{(L+\Delta)}(R,S),\\
\bxi^{(l)}(I)=p^{-(L+\Delta-l)}\sum_{\substack{R\in G_{L+\Delta}\\ \Lambda_{L+\Delta,l}(R)=I}}\bxi^{(L+\Delta)}(R),
\end{gathered}\label{Average_values}
\end{equation}
that is, $\bW^{(l)}(I,K)$ and $\bxi^{(l)}(I)$ are the arithmetic means of the values of $\bW^{(L+\Delta)}$ and $\bxi^{(L+\Delta)}$ over the vertices of level $L+\Delta$ below $(I,K)$ and $I$, respectively. Since $\bW^{(l)}\in\D^{l}(\Zp\times\Zp)$ and $\bh^{(l-1)}\in\D^{l-1}(\Zp)$, Lemma \ref{Lemma-1A} with $m=l-1$ applies to the first integral in (\ref{Formula_0}), and Lemma \ref{Lemma-2A} to the second. The formulas for $\bh^{(L)}$ and $\by^{(L+\Delta)}$ follow in the same way from Lemmas \ref{Lemma-2A} and \ref{Lemma-1A}.
\end{proof}

\begin{remark}
\label{Remark_tree}The computation (\ref{Formula_DNN}) is organized by the tree. Each $K\in G_{l}$ is written uniquely as $K=J+kp^{l-1}$ with $J=\Lambda_{l,l-1}(K)\in G_{l-1}$ and $k\in\{0,\ldots,p-1\}$, so
\begin{multline*}
\bh^{(l)}(I)=\sum_{J\in G_{l-1}}\Bigl(\sum_{k=0}^{p-1}p^{-l}\bW^{(l)}(I,J+kp^{l-1})\Bigr)\phi\bigl(\bh^{(l-1)}(J)\bigr)\\
+\sum_{K\in G_{L}}p^{-L}\Win^{(L)}(\Lambda_{l,L}(I),K)\,\bx^{(L)}(K)+\bxi^{(l)}(I).
\end{multline*}
Thus the neurons of layer $l$ are the $p^{l}$ vertices of level $l$ of the tree, and the neuron $I$ of layer $l$ receives the activations of the vertices $J$ of level $l-1$ through the effective weights $\sum_{k}p^{-l}\bW^{(l)}(I,J+kp^{l-1})$, obtained by summing the kernel over the $p$ children of $J$. Iterating this decomposition, the hidden state at level $L+\Delta$ is computed by a tree-like structure with $\Delta$ layers, from the input at level $L$; the widths $p^{L},\ldots,p^{L+\Delta}$ grow along the tree.
\end{remark}

\begin{remark}
\label{Remark_single_kernel}In a standard DNN, each layer has its own weight matrix. In Definition \ref{Definition-3}, all layers are determined by a single kernel $\bW^{(L+\Delta)}$, through its coarse-grained versions $\bW^{(l)}$, the averages over the balls of level $l$, $L<l\leq L+\Delta$. This is possible because the neurons of all layers are points of the same space $\Zp$, and it is one of the main reasons for using $p$-adic numbers: since averages of averages are averages, $(\bW^{(L+\Delta)})^{(l)}=\bW^{(l)}$, networks of increasing depth built from the same kernel on $\Zp\times\Zp$ are consistent discretizations of one continuous object. Theorem \ref{Theorem_depth} below shows that, in the stable region, the hidden states converge as the depth tends to infinity to the state of the continuous network of Definition \ref{Definition_DNN}. In this sense the distinction between $p$-adic discrete DNNs and recurrent $p$-adic continuous DNNs disappears in the limit.
\end{remark}

\subsection{Infinite-depth limits}

\begin{theorem}
\label{Theorem_depth}Consider a $p$-adic continuous DNN as in Definition \ref{Definition_DNN}, with parameters $\btheta=\{\phi,\varphi,\Win,\Wout,\bW,\bxi,\xiout\}$ in $\X_{\mathrm{stable}}$, so $q=L_{\phi}\Vert\T_{\bW}\Vert<1$, with input $\bx\in L^{2}(\Zp)$, hidden state $\bh=\bh(\btheta,\bx)$ and output $\by=\by(\btheta,\bx)$. For $L,\Delta\geq1$, consider the $p$-adic discrete DNN of Definition \ref{Definition-3} with input $\bx^{(L)}$ and parameters
\[
\btheta_{L,\Delta}=\left\{L,\Delta,\phi,\varphi,\Win^{(L)},\Wout^{(L+\Delta)},\bW^{(L+\Delta)},\bxi^{(L+\Delta)},\xiout^{(L+\Delta)}\right\},
\]
where $\bx^{(L)}$, $\Win^{(L)}$, $\Wout^{(L+\Delta)}$, $\bW^{(L+\Delta)}$, $\bxi^{(L+\Delta)}$, $\xiout^{(L+\Delta)}$ are the averages (\ref{Average_1})--(\ref{Average_2}) of $\bx,\Win,\Wout,\bW,\bxi,\xiout$. Then:
\begin{enumerate}
\item[(i)] the hidden states $\bh^{(l)}$, $L\leq l\leq L+\Delta$, do not depend on $\Delta$;
\item[(ii)] for each $L$, $\bh^{(L+\Delta)}\rightarrow\bh_{L}$ and $\by^{(L+\Delta)}\rightarrow\by_{L}$ in $L^{2}(\Zp)$ as $\Delta\rightarrow\infty$, where $\bh_{L}$ and $\by_{L}$ are the state and the output of the continuous network with parameters $\btheta$ and input term $\bu_{L}=\T_{\Win^{(L)}}\bx^{(L)}$ in place of $\bu=\T_{\Win}\bx$;
\item[(iii)] $\Vert\bh_{L}-\bh\Vert_{2}\leq(1-q)^{-1}\Vert\bu_{L}-\bu\Vert_{2}$ and $\Vert\by_{L}-\by\Vert_{2}\leq L_{\varphi}\Vert\Wout\Vert_{2}\Vert\bh_{L}-\bh\Vert_{2}$, and $\Vert\bu_{L}-\bu\Vert_{2}\rightarrow0$ as $L\rightarrow\infty$; consequently
\[
\lim_{L\rightarrow\infty}\lim_{\Delta\rightarrow\infty}\Vert\bh^{(L+\Delta)}-\bh\Vert_{2}=0,\qquad\lim_{L\rightarrow\infty}\lim_{\Delta\rightarrow\infty}\Vert\by^{(L+\Delta)}-\by\Vert_{2}=0;
\]
\item[(iv)] if $\Win\in\D^{L_{0}}(\Zp\times\Zp)$ and $\bx\in\D^{L_{0}}(\Zp)$, then $\bu_{L}=\bu$ for $L\geq L_{0}$, and for every such $L$
\[
\lim_{\Delta\rightarrow\infty}\Vert\bh^{(L+\Delta)}-\bh\Vert_{2}=0,\qquad\lim_{\Delta\rightarrow\infty}\Vert\by^{(L+\Delta)}-\by\Vert_{2}=0.
\]
\end{enumerate}
\end{theorem}

\begin{proof}
(i) The kernels used at level $l\leq L+\Delta$ are $(\bW^{(L+\Delta)})^{(l)}=\bW^{(l)}$ and $(\bxi^{(L+\Delta)})^{(l)}=\bxi^{(l)}$ (Section \ref{Section_Dl}), the averages of $\bW$ and $\bxi$ at level $l$, which do not depend on $\Delta$. Intuitively, the independence of $\Delta$ comes from the fact that averages of averages are averages.

(ii) By (\ref{Projected_operator}), $\Vert\T_{\bW^{(l)}}\Vert=\Vert\boldsymbol{P}_{l}\T_{\bW}\boldsymbol{P}_{l}\Vert\leq\Vert\T_{\bW}\Vert$, so $L_{\phi}\Vert\T_{\bW^{(l)}}\Vert\leq q$ for all $l$. The continuous network with input term $\bu_{L}$ has the same kernel $\bW$, so it has a unique state $\bh_{L}=\T_{\bW}(\phi\circ\bh_{L})+\bu_{L}+\bxi$ by Theorem \ref{Theorem_A}. The recursion (\ref{Formula_0}) reads $\bh^{(l)}=\T_{\bW^{(l)}}(\phi\circ\bh^{(l-1)})+\bu_{L}+\bxi^{(l)}$, so $\boldsymbol{e}_{l}=\bh^{(l)}-\bh_{L}$ satisfies
\[
\boldsymbol{e}_{l}=\T_{\bW^{(l)}}\bigl(\phi\circ\bh^{(l-1)}-\phi\circ\bh_{L}\bigr)+\T_{\bW^{(l)}-\bW}(\phi\circ\bh_{L})+\bxi^{(l)}-\bxi,
\]
and, by Lemma \ref{Lemma-0A} and (\ref{Nemytskii}),
\[
\Vert\boldsymbol{e}_{l}\Vert_{2}\leq q\,\Vert\boldsymbol{e}_{l-1}\Vert_{2}+\delta_{l},\qquad\delta_{l}=L_{\phi}\Vert\bW^{(l)}-\bW\Vert_{2}\Vert\bh_{L}\Vert_{2}+\Vert\bxi^{(l)}-\bxi\Vert_{2}.
\]
By Lemma \ref{Lemma_Pl}, $\delta_{l}\rightarrow0$. Iterating, $\Vert\boldsymbol{e}_{l}\Vert_{2}\leq q^{\,l-L}\Vert\boldsymbol{e}_{L}\Vert_{2}+\sum_{j=L+1}^{l}q^{\,l-j}\delta_{j}$, and both terms tend to zero as $l\rightarrow\infty$, since $q<1$ and $\delta_{j}\rightarrow0$. For the outputs,
\[
\by^{(l)}-\by_{L}=\T_{\Wout^{(l)}}\bigl(\varphi\circ\bh^{(l)}-\varphi\circ\bh_{L}\bigr)+\T_{\Wout^{(l)}-\Wout}(\varphi\circ\bh_{L})+\xiout^{(l)}-\xiout,
\]
which tends to zero by the same estimates and Lemma \ref{Lemma_Pl}.

(iii) The first two inequalities follow as in the proof of Corollary \ref{Cor_robust}(i), since $\bh_{L}$ and $\bh$ solve the state equation with the same $\bW$, $\bxi$ and $\phi$ and input terms $\bu_{L}$ and $\bu$. By (\ref{Projected_operator}), $\bu_{L}=\boldsymbol{P}_{L}\T_{\Win}\boldsymbol{P}_{L}\bx$, so
\begin{align*}
\Vert\bu_{L}-\bu\Vert_{2} & \leq\Vert\boldsymbol{P}_{L}\T_{\Win}(\boldsymbol{P}_{L}\bx-\bx)\Vert_{2}+\Vert\boldsymbol{P}_{L}\T_{\Win}\bx-\T_{\Win}\bx\Vert_{2}\\
 & \leq\Vert\Win\Vert_{2}\Vert\boldsymbol{P}_{L}\bx-\bx\Vert_{2}+\Vert\boldsymbol{P}_{L}\bu-\bu\Vert_{2},
\end{align*}
which tends to zero by Lemma \ref{Lemma_Pl}.

To obtain the double limits, we write $\bh^{(L+\Delta)}-\bh=(\bh^{(L+\Delta)}-\bh_{L})+(\bh_{L}-\bh)$. Iterating the recursion of (ii) from $l=L$ to $l=L+\Delta$ and using the first inequality of (iii),
\begin{equation}
\Vert\bh^{(L+\Delta)}-\bh\Vert_{2}\leq q^{\Delta}\Vert\boldsymbol{e}_{L}\Vert_{2}+\sum_{j=L+1}^{L+\Delta}q^{\,L+\Delta-j}\delta_{j}+\frac{\Vert\bu_{L}-\bu\Vert_{2}}{1-q}, \label{Error_depth}
\end{equation}
where $\boldsymbol{e}_{L}=\bh^{(L)}-\bh_{L}$ and $\delta_{j}=L_{\phi}\Vert\bW^{(j)}-\bW\Vert_{2}\Vert\bh_{L}\Vert_{2}+\Vert\bxi^{(j)}-\bxi\Vert_{2}$. Similarly, since $\Vert\T_{\Wout^{(l)}}\Vert\leq\Vert\Wout^{(l)}\Vert_{2}\leq\Vert\Wout\Vert_{2}$,
\begin{multline*}
\Vert\by^{(L+\Delta)}-\by\Vert_{2}\leq L_{\varphi}\Vert\Wout\Vert_{2}\Bigl(\Vert\bh^{(L+\Delta)}-\bh_{L}\Vert_{2}+\Vert\bh_{L}-\bh\Vert_{2}\Bigr)\\
+L_{\varphi}\Vert\Wout^{(L+\Delta)}-\Wout\Vert_{2}\Vert\bh_{L}\Vert_{2}+\Vert\xiout^{(L+\Delta)}-\xiout\Vert_{2}.
\end{multline*}
In (\ref{Error_depth}), the first two terms tend to zero as $\Delta\rightarrow\infty$ for fixed $L$, by (ii). The last term does not depend on $\Delta$: although $L+\Delta\rightarrow\infty$, the input is presented at the fixed level $L$, so the input term $\bu_{L}=\T_{\Win^{(L)}}\bx^{(L)}$ does not change with $\Delta$, and $\Vert\bu_{L}-\bu\Vert_{2}$ is the discretization error of the input. This term tends to zero only as $L\rightarrow\infty$. Hence the double limits. In fact, the convergence also holds when $L$ and $\Delta$ tend to infinity simultaneously, in any manner: by (\ref{Bound_h}), $\Vert\bh_{L}\Vert_{2}\leq(\Vert\T_{\Win}\Vert\Vert\bx\Vert_{2}+\Vert\bxi\Vert_{2})/(1-q)$, since $\Vert\bu_{L}\Vert_{2}=\Vert\boldsymbol{P}_{L}\T_{\Win}\boldsymbol{P}_{L}\bx\Vert_{2}\leq\Vert\T_{\Win}\Vert\Vert\bx\Vert_{2}$, so $\Vert\boldsymbol{e}_{L}\Vert_{2}$ is bounded independently of $L$, the sum in (\ref{Error_depth}) is at most $(1-q)^{-1}\sup_{j>L}\delta_{j}$, and $\sup_{j>L}\delta_{j}\rightarrow0$ as $L\rightarrow\infty$ by Lemma \ref{Lemma_Pl}. The same holds for the outputs.

(iv) If $\Win$ and $\bx$ belong to $\D^{L_{0}}$ and $L\geq L_{0}$, then $\Win^{(L)}=\Win$ and $\bx^{(L)}=\bx$, so $\bu_{L}=\bu$, $\bh_{L}=\bh$ and $\by_{L}=\by$, and (ii) gives the assertion.
\end{proof}

\begin{remark}
In Theorem \ref{Theorem_depth}, the input is presented to the network at the fixed level $L$, so for a general input the limit $\Delta\rightarrow\infty$ recovers the continuous network driven by the input term $\bu_{L}$, and the limit $L\rightarrow\infty$ removes the remaining discretization error of the input; when the input and the input kernel are test functions, part (iv) shows that the single limit $\Delta\rightarrow\infty$ suffices. Since $\Zp\times\Zp$ is compact, continuous kernels belong to $L^{2}$, so the theorem applies in particular to continuous parameters. The networks in Theorem \ref{Theorem_depth} have increasing depth $\Delta$ and width $p^{L+\Delta}$, with weights $p^{-l}\bW^{(l)}(I,K)$; this is a mean-field scaling. Outside the stability region the hidden states $\bh^{(L+\Delta)}$ need not converge as $\Delta\rightarrow\infty$; the relation between their limit points and the multiple states of Section \ref{Section_7} is an open problem.
\end{remark}

When the parameters are test functions, their averages reproduce them at all sufficiently fine levels, and Theorem \ref{Theorem_depth} takes the following sharp form.

\begin{corollary}
\label{Cor_Picard}Let $L\leq l_{0}\leq L+\Delta$, let $\bW,\Wout\in\D^{l_{0}}(\Zp\times\Zp)$, $\bxi,\xiout\in\D^{l_{0}}(\Zp)$, $\Win\in\D^{L}(\Zp\times\Zp)$, $\bx\in\D^{L}(\Zp)$ and $\phi,\varphi\in\Lip$, and assume $q=L_{\phi}\Vert\T_{\bW}\Vert<1$. Let $\bh$ and $\by$ be the state and the output of the continuous network with these parameters, and consider the $p$-adic discrete DNN with parameters $\btheta_{\Delta}=\{L,\Delta,\phi,\varphi,\Win,\Wout,\bW,\bxi,\xiout\}$ and input $\bx$ (note that $\D^{l_{0}}\subset\D^{L+\Delta}$). Then:
\begin{enumerate}
\item[(i)] for $l_{0}<l\leq L+\Delta$, the hidden states satisfy $\bh^{(l)}=\Phi_{\btheta,\bx}(\bh^{(l-1)})$, where $\Phi_{\btheta,\bx}$ is the map (\ref{Fixed_point_map}) with kernels $\bW$, $\Win$ and bias $\bxi$; that is, the layers of the network above level $l_{0}$ are the Picard iterates of the continuous network;
\item[(ii)] the hidden state and the output satisfy
\begin{align*}
\Vert\bh^{(L+\Delta)}-\bh\Vert_{2} & \leq q^{L+\Delta-l_{0}}\,\Vert\bh^{(l_{0})}-\bh\Vert_{2},\\
\Vert\by^{(L+\Delta)}-\by\Vert_{2} & \leq L_{\varphi}\Vert\Wout\Vert_{2}\,q^{L+\Delta-l_{0}}\,\Vert\bh^{(l_{0})}-\bh\Vert_{2};
\end{align*}
\item[(iii)] in particular, given $\epsilon>0$, there is $\Delta_{0}\geq l_{0}-L$ such that for every $\Delta\geq\Delta_{0}$ the $p$-adic discrete DNN with parameters $\btheta_{\Delta}$ satisfies $\Vert\bh^{(L+\Delta)}-\bh\Vert_{2}<\epsilon$ and $\Vert\by^{(L+\Delta)}-\by\Vert_{2}<\epsilon$.
\end{enumerate}
\end{corollary}

The assumption $L\leq l_{0}$ is not a restriction: if $l_{0}<L$, then $\D^{l_{0}}\subset\D^{L}$, and we may replace $l_{0}$ by $L$.

\begin{proof}
A function in $\D^{l_{0}}$ is constant on the balls of radius $p^{-l}$ for $l\geq l_{0}$, so its averages (\ref{Average_1})--(\ref{Average_2}) at level $l\geq l_{0}$ reproduce it. Hence $\bW^{(l)}=\bW$ and $\bxi^{(l)}=\bxi$ for $l_{0}\leq l\leq L+\Delta$, $\Wout^{(L+\Delta)}=\Wout$ and $\xiout^{(L+\Delta)}=\xiout$, and, since $\Win\in\D^{L}$ and $\bx\in\D^{L}$, also $\Win^{(L)}=\Win$ and $\bx^{(L)}=\bx$. For $l_{0}<l\leq L+\Delta$, (\ref{Formula_0}) becomes $\bh^{(l)}=\T_{\bW}(\phi\circ\bh^{(l-1)})+\T_{\Win}\bx+\bxi=\Phi_{\btheta,\bx}(\bh^{(l-1)})$, which is (i). Since $\Phi_{\btheta,\bx}$ is a contraction with constant $q$ and $\bh=\Phi_{\btheta,\bx}(\bh)$ (proof of Theorem \ref{Theorem_A}), $\Vert\bh^{(l)}-\bh\Vert_{2}\leq q\Vert\bh^{(l-1)}-\bh\Vert_{2}$ for $l_{0}<l\leq L+\Delta$, and iterating gives the first inequality in (ii). Since $\by^{(L+\Delta)}-\by=\T_{\Wout}(\varphi\circ\bh^{(L+\Delta)}-\varphi\circ\bh)$, the second inequality follows from Lemma \ref{Lemma-0A} and (\ref{Nemytskii}). Part (iii) follows from (ii), since $q<1$ and $\Vert\bh^{(l_{0})}-\bh\Vert_{2}$ does not depend on $\Delta$ by Theorem \ref{Theorem_depth}(i).
\end{proof}

\begin{remark}
\label{Remark_old_TheoremA}Corollary \ref{Cor_Picard}(i) shows that propagating the input through the levels of the tree is the same as iterating the fixed-point map of the continuous network.
\end{remark}

\begin{remark}[$p$-Adic discrete CNNs]
\label{Theorem_exact}Theorem \ref{Theorem_depth} concerns the limit in which the size of the network, controlled by $L$ and $\Delta$, tends to infinity. For fixed $L$ and $\Delta$, the relevant finite objects are the discretizations of the $p$-adic continuous DNNs. Let $\btheta\in\X$, $L,\Delta\geq1$, $\Win\in\D^{L}(\Zp\times\Zp)$, $\bW\in\D^{L+\Delta}(\Zp\times\Zp)$, $\bxi\in\D^{L+\Delta}(\Zp)$ and $\bx\in L^{2}(\Zp)$. By Lemma \ref{Lemma_ON}, $\T_{\bW}$ maps $L^{2}(\Zp)$ into $\D^{L+\Delta}(\Zp)$ and $\T_{\Win}$ maps $L^{2}(\Zp)$ into $\D^{L}(\Zp)\subset\D^{L+\Delta}(\Zp)$, so the right side of (\ref{State_equation}) lies in $\D^{L+\Delta}(\Zp)$ for every $\bh\in L^{2}(\Zp)$. Hence:
\begin{enumerate}
\item[(i)] every state of (\ref{State_equation}) belongs to $\D^{L+\Delta}(\Zp)$;
\item[(ii)] by Lemmas \ref{Lemma-1A} and \ref{Lemma-2A}, a function $\bh\in\D^{L+\Delta}(\Zp)$ is a state of (\ref{State_equation}) if and only if its values satisfy
\begin{multline}
\bh(I)=\sum_{K\in G_{L+\Delta}}p^{-L-\Delta}\bW(I,K)\phi(\bh(K))\\
+\sum_{K\in G_{L}}p^{-L}\Win(\Lambda_{L+\Delta,L}(I),K)\bx^{(L)}(K)+\bxi(I),\qquad I\in G_{L+\Delta}, \label{Formula_1}
\end{multline}
where $\bx^{(L)}=\boldsymbol{P}_{L}\bx$; the network only sees the average of the input, since $\T_{\Win}\bx=\T_{\Win}\boldsymbol{P}_{L}\bx$;
\item[(iii)] if $\btheta\in\X_{\mathrm{stable}}$, the Picard iterates starting at $\bh_{0}=0$ stay in $\D^{L+\Delta}(\Zp)$ and are the iterates of the finite map (\ref{Formula_1});
\item[(iv)] if $\btheta\in\X_{\mathrm{stable}}$ and $\bW,\Win,\bxi$ are only in $L^{2}$, then $\bh(\btheta,\bx)$ is a limit of states of systems (\ref{Formula_1}): choose test functions $\bW_{k},\Win^{(k)},\bxi_{k}$ converging to $\bW,\Win,\bxi$ in $L^{2}$; the corresponding parameters $\btheta_{k}$ are in $\X_{\mathrm{stable}}$ for $k$ large (Lemma \ref{Lemma_open}), $\bh(\btheta_{k},\bx)\rightarrow\bh(\btheta,\bx)$ (Theorem \ref{Theorem_continuity}), and each $\bh(\btheta_{k},\bx)$ satisfies (\ref{Formula_1}) for suitable $L_{k},\Delta_{k}$, by (ii).
\end{enumerate}
In matrix form, (\ref{Formula_1}) reads $\bh=M\phi(\bh)+M_{\mathrm{in}}\bx^{(L)}+\bxi$, with $\phi$ acting componentwise, $M=[p^{-L-\Delta}\bW(I,K)]_{I,K\in G_{L+\Delta}}$ (Lemma \ref{Lemma_ON}) and $M_{\mathrm{in}}=[p^{-L}\Win(\Lambda_{L+\Delta,L}(I),K)]_{I\in G_{L+\Delta},K\in G_{L}}$. This is the equation of the stationary states of a $p$-adic discrete cellular neural network \cite[Definition 1]{Zambrano-Zuniga-1} on $G_{L+\Delta}$, with feedback operator $p^{-L-\Delta}\bW(I,K)$, feedforward operator $p^{-L}\Win(\Lambda_{L+\Delta,L}(I),K)$, input $\bx^{(L)}$, threshold $\bxi$ and activation $\phi$ (bounded in \cite{Zambrano-Zuniga-1}, only Lipschitz here), and (ii) is the discretization of $p$-adic continuous cellular neural networks of \cite[Section 3.3]{Zambrano-Zuniga-1}, restricted to stationary states. A $p$-adic continuous DNN whose parameters are test functions is called a \emph{finite $p$-adic DNN}. The states of (\ref{Formula_1}) are defined by a fixed-point equation at the level $L+\Delta$, and not by the recursion (\ref{Formula_0}) of Definition \ref{Definition-3}; when the parameters are test functions and $\btheta\in\X_{\mathrm{stable}}$, the hidden state $\bh^{(L+\Delta)}$ of the $p$-adic discrete DNN converges to the state of (\ref{Formula_1}) as $\Delta\rightarrow\infty$ (Corollary \ref{Cor_Picard}).
\end{remark}

\section{\label{Section_6}Feedforward and layered networks}
\label{Section_layered}
We recall that an open compact subset of $\Zp$ is a finite disjoint union of balls. Let $U_{0},\ldots,U_{\Delta}\subset\Zp$, $\Delta\geq1$, be pairwise disjoint open compact subsets such that
\[
\Zp=\bigsqcup_{k=0}^{\Delta}U_{k}.
\]
Then $L^{2}(\Zp)=\bigoplus_{k=0}^{\Delta}L^{2}(U_{k})$, and any $\bh\in L^{2}(\Zp)$ is uniquely determined by its restrictions $\bh_{k}=\bh|_{U_{k}}\in L^{2}(U_{k})$, $k=0,1,\ldots,\Delta$. We call $\bW\in L^{2}(\Zp\times\Zp)$ \emph{feedforward} with respect to this partition if $\bW$ is supported on
\[
\bigsqcup_{0\leq j<k\leq\Delta}U_{k}\times U_{j},
\]
that is, $\bW=0$ almost everywhere outside this set, and \emph{layered} if $\bW$ is supported on
\[
\bigsqcup_{k=1}^{\Delta}U_{k}\times U_{k-1}.
\]
A feedforward kernel defines a continuous feedforward network whose $k$-th layer is the continuum of neurons $U_{k}$; the neurons of $U_{k}$ receive connections only from the neurons of $U_{0},\ldots,U_{k-1}$, with weights $\bW|_{U_{k}\times U_{j}}$, $j<k$. Connections with $j<k-1$ are skip connections, as in residual networks; a layered kernel has none, so the neurons of $U_{k}$ receive connections only from those of $U_{k-1}$. A $p$-adic continuous DNN is \emph{feedforward}, respectively \emph{layered}, if its kernel is feedforward, respectively layered, with respect to some partition as above, and \emph{recurrent} otherwise. Thus every layered network is feedforward, and a recurrent network has \emph{feedback}: for every partition as above, there are connections, on a set of positive measure, inside a layer, from a neuron to itself, or from a layer to an earlier one.

Consider the equation $\bh=\T_{\bW}(\phi\circ\bh)+\boldsymbol{b}$, with $\boldsymbol{b}\in L^{2}(\Zp)$; in (\ref{Eq-B}), $\boldsymbol{b}=\T_{\Win}\bx+\bxi$. If $\bW$ is feedforward, no connection arrives at $U_{0}$, so every solution satisfies $\bh_{0}=\boldsymbol{b}|_{U_{0}}$. The \emph{input} of the network is $\bh_{0}\in L^{2}(U_{0})$, its \emph{output} is $\bh_{\Delta}\in L^{2}(U_{\Delta})$, its \emph{hidden states} are $\bh_{k}\in L^{2}(U_{k})$, $k=1,\ldots,\Delta-1$, and $\Delta$ is the \emph{depth} of the network. The restrictions $\boldsymbol{b}|_{U_{k}}$, $k=1,\ldots,\Delta$, play the role of biases.

\begin{proposition}
\label{Theorem_feedforward}Let $\bW$ be feedforward, $\phi\in\Lip$ and $\boldsymbol{b}\in L^{2}(\Zp)$. Then $\bh=\T_{\bW}(\phi\circ\bh)+\boldsymbol{b}$ has exactly one solution $\bh\in L^{2}(\Zp)$, without any restriction on $\Vert\bW\Vert_{2}$. It is given recursively by $\bh_{0}=\boldsymbol{b}|_{U_{0}}$ and
\[
\bh_{k}(x)=\int_{U_{0}\sqcup\cdots\sqcup U_{k-1}}\bW(x,y)\phi(\bh(y))\,dy+\boldsymbol{b}(x),\qquad x\in U_{k},\ k=1,\ldots,\Delta.
\]
If $\bW$ is layered, the integral is over $U_{k-1}$, and $\bh_{k}$ depends only on $\bh_{k-1}$.
\end{proposition}

\begin{proof}
Since $\bW$ is feedforward, $\bW(x,\cdot)=0$ for $x\in U_{0}$, and, for $x\in U_{k}$, $k\geq1$, $\bW(x,\cdot)$ vanishes outside $U_{0}\sqcup\cdots\sqcup U_{k-1}$. Hence $\bh$ is a solution if and only if $\bh_{0}=\boldsymbol{b}|_{U_{0}}$ and $\bh_{k}$ satisfies the displayed formula for $k=1,\ldots,\Delta$. The right-hand side of this formula involves only $\bh_{0},\ldots,\bh_{k-1}$, so the formula determines $\bh_{1},\ldots,\bh_{\Delta}$ successively, and in a unique way, from $\bh_{0}$. Moreover, if $\bh_{0},\ldots,\bh_{k-1}$ are square integrable, so is $\bh_{k}$, by Lemma \ref{Lemma-0A} applied to the function equal to $\bh$ on $U_{0}\sqcup\cdots\sqcup U_{k-1}$ and to $0$ elsewhere. By induction on $k$, the equation has exactly one solution, and it belongs to $L^{2}(\Zp)$.
\end{proof}

\begin{corollary}
\label{Cor_feedforward}Multiplicity of states cannot occur in feedforward $p$-adic DNNs, in particular in layered ones, whatever the size of the weights. A necessary condition for a $p$-adic continuous DNN to have more than one state is that it is recurrent, that is, that it has feedback.
\end{corollary}

The following criterion shows that a network is recurrent without examining all partitions.

\begin{lemma}
\label{Lemma_nilpotent}If $\bW$ is feedforward with respect to a partition $U_{0},\ldots,U_{\Delta}$, then $\T_{\bW}^{\Delta+1}=0$. Consequently, if $\T_{\bW}$ is not nilpotent, the network is recurrent. In particular, if $\bW(x,y)=k(x-y)$ is radial with $k\geq0$ and $\lambda=\int_{\Zp}k(y)\,dy>0$, the network is recurrent.
\end{lemma}

\begin{proof}
We use induction on $\Delta$. Recall that $\bW(x,y)$ is the weight from $y$ to $x$, where $x$ is the neuron that receives the signal, and that $\bW$ vanishes almost everywhere outside $\bigsqcup_{0\leq j<k\leq\Delta}U_{k}\times U_{j}$.

\emph{Case $\Delta=1$.} Here $\bW$ is supported on $U_{1}\times U_{0}$. For every $f\in L^{2}(\Zp)$, $\T_{\bW}f$ vanishes on $U_{0}$, because $\bW(x,\cdot)=0$ for $x\in U_{0}$, and, for $x\in U_{1}$, $\T_{\bW}f(x)=\int_{U_{0}}\bW(x,y)f(y)\,dy$ depends only on $f|_{U_{0}}$. Applying this to $g=\T_{\bW}f$, which vanishes on $U_{0}$, we get $\T_{\bW}g=0$, so $\T_{\bW}^{2}=0$.

\emph{Inductive step.} Let $\Delta\geq2$, and assume the assertion for partitions of $\Zp$ into $\Delta$ open compact sets. Set $U^{\prime}=U_{0}\sqcup\cdots\sqcup U_{\Delta-1}$, $\bW^{\prime}=\bW1_{U^{\prime}\times U^{\prime}}$ and $\bW^{\prime\prime}=\bW1_{U_{\Delta}\times U^{\prime}}$. Since no connection leaves $U_{\Delta}$, $\bW=\bW^{\prime}+\bW^{\prime\prime}$ almost everywhere. The kernel $\bW^{\prime}$ is feedforward with respect to the partition $U_{0},\ldots,U_{\Delta-2},U_{\Delta-1}\sqcup U_{\Delta}$ of $\Zp$ into $\Delta$ sets, so $\T_{\bW^{\prime}}^{\Delta}=0$ by the induction hypothesis. Write $L^{2}(\Zp)=L^{2}(U^{\prime})\oplus L^{2}(U_{\Delta})$ and $f=(f^{\prime},f_{\Delta})$. Then
\[
\T_{\bW}(f^{\prime},f_{\Delta})=(Af^{\prime},Bf^{\prime}),\qquad A=\T_{\bW^{\prime}}|_{L^{2}(U^{\prime})},\quad B=\T_{\bW^{\prime\prime}}|_{L^{2}(U^{\prime})},
\]
and, by a direct computation, $\T_{\bW}^{n}(f^{\prime},f_{\Delta})=(A^{n}f^{\prime},BA^{n-1}f^{\prime})$ for $n\geq1$. Since $A^{\Delta}=0$, taking $n=\Delta+1$ gives $\T_{\bW}^{\Delta+1}=0$.

For the radial kernel, $\T_{\bW}1=k\ast1=\lambda\cdot1$, so $\T_{\bW}^{n}1=\lambda^{n}\neq0$ for all $n$.
\end{proof}

\begin{remark}
Proposition \ref{Theorem_feedforward} shows how a layer structure survives in a continuum: the layers are open compact subsets of $\Zp$, that is, finite unions of balls, and the width of each layer is the continuum of points of such a set. Whether sequences of finite networks with suitably scaled weights converge to layered continuous networks is discussed in Section \ref{Section_continuous_limits}.
\end{remark}

\section{\label{Section_8}Critical organization in a toy model}
By \emph{critical organization} we mean a transition, as a parameter of the network, such as the total coupling of the kernel, crosses a threshold, from a unique state to several states. We first describe it in the simplest $p$-adic network, the edge detector of \cite{Zambrano-Zuniga-2}, in which the transition and the organization of the states can be computed explicitly: for $a<1$ there is exactly one state, and for $a>1$ the states are labelled by sign patterns and ordered by the poset of Section \ref{Section_order_sign}. The idea of organizing the states of $p$-adic networks by an order on sign patterns goes back to \cite{Zambrano-Zuniga-2}; here we use a corrected, signed version of that order (Remark \ref{Remark_ZZ2}). The same poset describes the states of the hierarchical recurrent networks of Section \ref{Section_hierarchical}; Examples \ref{Example_toy} and \ref{Example_hierarchical} illustrate this connection in the simplest case.

\subsection{$p$-Adic cellular neural networks}
$p$-Adic cellular neural networks (CNNs) were introduced by Zambrano-Luna and the author \cite{Zambrano-Zuniga-1,Zambrano-Zuniga-2}; they generalize the classical CNNs of Chua and Yang \cite{Chua-Tamas}, \cite{Slavova} by arranging cells on the tree $\Zp$ rather than on a Euclidean lattice. In \cite{Zambrano-Zuniga-2}, an edge detector for grayscale images was introduced:
\begin{equation}
\partial_{t}\bh(z,t)=-\bh(z,t)+a\phi(\bh(z,t))+\bu(z),\qquad\bu=\Win\ast\bx+\bxi, \label{CNN_1}
\end{equation}
with $\phi(s)=\frac{1}{2}(|s+1|-|s-1|)$, $a>0$, $\Win\in L^{1}(\Zp)$ and $\bx,\bxi\in C(\Zp)$, so that $\bu\in C(\Zp)$. Numerical experiments in \cite{Zambrano-Zuniga-2} show that, starting from $\bh(z,0)=0$, the output $\phi(\bh(z,t))$ approaches a black-and-white image approximating the edges of $\bx$. The stationary states satisfy
\begin{equation}
\bh(z)=a\phi(\bh(z))+\bu(z),\qquad z\in\Zp. \label{Map_B}
\end{equation}

Equation (\ref{Map_B}) has no integral kernel: it is a scalar equation at each point, and the input $\bu$ is the only place where the $p$-adic structure enters. Formally it is the case $\bW=a\,\delta(x-y)$, which is not in $L^{2}(\Zp\times\Zp)$; hence Theorem \ref{Theorem_A} does not apply directly, although the same contraction argument works pointwise for $a<1$. Its coupled, finite-resolution analogue is the kernel $a\boldsymbol{P}_{l}$ of Theorem \ref{Theorem_hierarchical}.

\subsection{\label{Section_order_sign}An order on sign patterns}
Let $V=\{-1,0,1\}$ with the partial order $0\prec1$, $0\prec-1$, and $1,-1$ incomparable. For a set $S$, we denote by $V^{S}$ the set of all functions $\sigma:S\rightarrow V$. We order $V^{S}$ componentwise: $\sigma\sqsubseteq\sigma^{\prime}$ if and only if $\sigma(s)\in\{0,\sigma^{\prime}(s)\}$ for all $s\in S$. The support of $\sigma$ is $\mathrm{supp}\,\sigma=\{s;\sigma(s)\neq0\}$. When $S=\varnothing$, $V^{S}$ has exactly one element, the empty function.

We recall some notions on partially ordered sets; see \cite[Chapters 1--2]{Davey-Priestley}. Let $(P,\sqsubseteq)$ be a partially ordered set and $\sigma,\sigma^{\prime}\in P$. A \emph{lower bound} of $\sigma$ and $\sigma^{\prime}$ is an element $\tau$ with $\tau\sqsubseteq\sigma$ and $\tau\sqsubseteq\sigma^{\prime}$; their \emph{meet} $\sigma\wedge\sigma^{\prime}$, if it exists, is the greatest lower bound (infimum) of $\sigma$ and $\sigma^{\prime}$. Dually, their \emph{join} $\sigma\vee\sigma^{\prime}$, if it exists, is the least upper bound (supremum). $P$ is a \emph{lattice} if any two elements have a meet and a join, and a \emph{meet-semilattice} if any two elements have a meet. The \emph{principal ideal} generated by $\sigma$ is $\{\tau\in P;\tau\sqsubseteq\sigma\}$. The \emph{Boolean lattice} of subsets of a finite set $A$ is the set of all subsets of $A$ ordered by inclusion; there, the meet is the intersection and the join is the union.

\begin{lemma}
\label{Lemma_poset}(i) $(V^{S},\sqsubseteq)$ is a partially ordered set with least element $0$; its maximal elements are the $\sigma$ with $\mathrm{supp}\,\sigma=S$.
(ii) Any two elements have a meet, $(\sigma\wedge\sigma^{\prime})(s)=\sigma(s)$ if $\sigma(s)=\sigma^{\prime}(s)$ and $0$ otherwise; they have a join if and only if there is no $s$ with $\sigma(s)=-\sigma^{\prime}(s)\neq0$.
(iii) For each $\sigma$, the principal ideal $\{\tau;\tau\sqsubseteq\sigma\}$ is isomorphic to the Boolean lattice of subsets of $\mathrm{supp}\,\sigma$.
(iv) If $S\neq\varnothing$, then $V^{S}$ is not a lattice.
\end{lemma}

\begin{proof}
All assertions reduce to the corresponding ones for $V$, which are immediate; for (iv), if $s\in S$, two maximal elements that differ only by the sign at $s$ have no common upper bound, hence no join.
\end{proof}

\subsection{The scalar equation}
For $t\in\R$ let $S_{a}(t)$ be the set of solutions $s\in\R$ of $s=a\phi(s)+t$.

\begin{lemma}
\label{Lemma_scalar}
\begin{enumerate}
\item[(i)] If $0<a<1$, $S_{a}(t)$ has exactly one element: $a+t$ if $t\geq1-a$, $t-a$ if $t\leq a-1$, and $t/(1-a)$ if $|t|<1-a$.
\item[(ii)] If $a=1$, $S_{1}(t)=\{1+t\}$ for $t>0$, $\{t-1\}$ for $t<0$, and $S_{1}(0)=[-1,1]$.
\item[(iii)] If $a>1$, $S_{a}(t)=\{a+t\}$ if $t>a-1$, $\{t-a\}$ if $t<1-a$, $\{a+t,\,t-a\}$ if $|t|=a-1$, and $\{a+t,\,t-a,\,t/(1-a)\}$ if $|t|<a-1$. The three values lie in $[1,\infty)$, $(-\infty,-1]$ and $(-1,1)$, respectively.
\end{enumerate}
\end{lemma}

\begin{proof}
A solution $s\geq1$ satisfies $s=a+t$, which is admissible iff $a+t\geq1$. A solution $s\leq-1$ satisfies $s=t-a$, admissible iff $t\leq a-1$. A solution $|s|<1$ satisfies $s=as+t$: if $a\neq1$, $s=t/(1-a)$, admissible iff $|t|<|1-a|$; if $a=1$ this requires $t=0$, and then every $s\in(-1,1)$ is a solution. Collecting the cases gives the statement.
\end{proof}

\subsection{The states}
A state of (\ref{Map_B}) is a function $\bh:\Zp\rightarrow\R$ satisfying (\ref{Map_B}) at every point. Since (\ref{Map_B}) is pointwise, $\bh$ is a state if and only if $\bh(z)\in S_{a}(\bu(z))$ for all $z$.

\begin{theorem}
\label{Theorem_toy}Let $\bu\in C(\Zp)$.
\begin{enumerate}
\item[(i)] If $0<a<1$, there is exactly one state; it is given by Lemma \ref{Lemma_scalar}(i) and is continuous.
\item[(ii)] If $a=1$, there is exactly one state if and only if $\bu$ has no zeros. If $N=\{z;\bu(z)=0\}\neq\varnothing$, the states are the functions equal to $\bu+\mathrm{sign}(\bu)$ on $\Zp\smallsetminus N$ and with arbitrary values in $[-1,1]$ on $N$; in $L^{1}(\Zp)$, the state is unique if and only if $N$ has measure zero.
\item[(iii)] If $a>1$, the states are the functions $\bh(z)\in S_{a}(\bu(z))$. Writing $F=\{z;|\bu(z)|<a-1\}$, $F_{\pm}=\{z;\pm\bu(z)>a-1\}$ and $E=\{z;|\bu(z)|=a-1\}$, a state is determined by a labelling $\lambda:\Zp\rightarrow V$ with $\lambda=1$ on $F_{+}$, $\lambda=-1$ on $F_{-}$, $\lambda\in\{\pm1\}$ on $E$, and $\lambda$ arbitrary on $F$, through $\bh=a+\bu$ where $\lambda=1$, $\bh=\bu-a$ where $\lambda=-1$, and $\bh=\bu/(1-a)$ where $\lambda=0$.
\end{enumerate}
\end{theorem}

\begin{proof}
Immediate from Lemma \ref{Lemma_scalar}. In (i), continuity follows because the three formulas agree on the boundaries $\bu=\pm(1-a)$.
\end{proof}

\subsection{The order structure of the states}
Assume $a>1$ and $E=\varnothing$ (that is, $|\bu(z)|\neq a-1$ for all $z$), and let $\mathcal{M}_{a}$ be the set of states. By Theorem \ref{Theorem_toy}(iii), $\lambda\mapsto\bh$ is a bijection from $V^{F}$ onto $\mathcal{M}_{a}$ (the labels on $F_{\pm}$ being forced). We order $\mathcal{M}_{a}$ by transporting the order of Lemma \ref{Lemma_poset}: $\bh\sqsubseteq\bh^{\prime}$ if $\lambda\sqsubseteq\lambda^{\prime}$, that is, if $I_{+}\subseteq I_{+}^{\prime}$ and $I_{-}\subseteq I_{-}^{\prime}$. We call a state \emph{bistable} if $\lambda(z)\neq0$ for all $z$, equivalently if $|\bh(z)|\geq1$ everywhere.

\begin{theorem}
\label{Theorem_order}Assume $a>1$ and $E=\varnothing$. Then $(\mathcal{M}_{a},\sqsubseteq)$ is isomorphic to $(V^{F},\sqsubseteq)$. Consequently:
\begin{enumerate}
\item[(i)] it has a least element, the state with $\bh=\bu/(1-a)$ on $F$, and its maximal elements are the bistable states;
\item[(ii)] let $\bh,\bh^{\prime}$ be states with labellings $\lambda,\lambda^{\prime}$. They have a meet, the state with labelling $\lambda\wedge\lambda^{\prime}$, equal to $\lambda(z)$ where $\lambda(z)=\lambda^{\prime}(z)$ and to $0$ elsewhere. They have a join if and only if there is no $z\in F$ with $\lambda(z)=-\lambda^{\prime}(z)\neq0$, that is, no point at which one of the two states takes the value $a+\bu(z)$ and the other the value $\bu(z)-a$; in that case the join is the state with labelling $\lambda\vee\lambda^{\prime}$, equal to $\lambda(z)$ where $\lambda(z)\neq0$ and to $\lambda^{\prime}(z)$ elsewhere;
\item[(iii)] the states below a given state $\bh$ form a Boolean lattice isomorphic to the subsets of $F\cap\{|\bh|\geq1\}$;
\item[(iv)] $(\mathcal{M}_{a},\sqsubseteq)$ is a lattice if and only if $F=\varnothing$, in which case there is only one state.
\end{enumerate}
Moreover, the measurable states form a sub-meet-semilattice, and so do the states in $L^{2}(\Zp)$.
\end{theorem}

\begin{proof}
The isomorphism holds by construction, and (i)--(iv) follow from Lemma \ref{Lemma_poset}; for (iv), if $F=\varnothing$, then $V^{F}$ has one element, so there is only one state. A state is measurable if and only if $I_{\pm}$ are measurable, and meets correspond to intersections of these sets. Every state satisfies $|\bh|\leq a+|\bu|+|\bu|/(a-1)$, so measurable states are in $L^{2}(\Zp)$.
\end{proof}

\begin{example}
\label{Example_toy}Let $p=2$, $a>1$ and $\bu=0$. Then $E=F_{\pm}=\varnothing$ and $F=\mathbb{Z}_{2}$, and by Theorem \ref{Theorem_toy}(iii) the states are the functions $\bh=a\lambda$, $\lambda:\mathbb{Z}_{2}\rightarrow V$. The states in $\D^{1}(\mathbb{Z}_{2})$ are constant on the two balls $2\mathbb{Z}_{2}$ and $1+2\mathbb{Z}_{2}$, so they correspond to the sign patterns $\lambda\in V^{S}$, $S=G_{1}=\{0,1\}$: there are $3^{2}=9$ of them. In the order of Theorem \ref{Theorem_order}, the least element is $\lambda=(0,0)$, that is, $\bh=0$; above it are the four patterns $(\pm1,0)$, $(0,\pm1)$; and the maximal elements are the four bistable patterns $(\pm1,\pm1)$. For instance, $(1,0)\sqsubseteq(1,1)$ and $(1,0)\sqsubseteq(1,-1)$, the meet of $(1,1)$ and $(1,-1)$ is $(1,0)$, and these two states have no join. The Hasse diagram is shown in Figure \ref{Fig_regimes}(c). On $\D^{1}(\mathbb{Z}_{2})\simeq\R^{2}$ the dynamics (\ref{CNN_1}) is $\dot{s}_{I}=-s_{I}+a\phi(s_{I})$, $I\in\{0,1\}$; the value $0$ is repelling, with eigenvalue $a-1>0$, and the values $\pm a$ are attracting, with eigenvalue $-1$. Hence the state $\lambda$ has an unstable manifold of dimension equal to the number of zeros of $\lambda$: the order of Lemma \ref{Lemma_poset} is graded by instability.
\end{example}

\begin{remark}
\label{Remark_ZZ2}Comparison with \cite{Zambrano-Zuniga-2}. In the notation of \cite{Zambrano-Zuniga-2}, $\bu=(B\ast U)+Z$, and the sets $I_{\pm}=\lambda^{-1}(\pm1)$ of Theorem \ref{Theorem_toy}(iii) are those of \cite[Definition 1]{Zambrano-Zuniga-2}. Three statements of \cite{Zambrano-Zuniga-2} need correction. First, \cite[Lemma 1(ii)]{Zambrano-Zuniga-2} asserts that for $a=1$ the state in $L^{1}(\Zp)$ is unique; by Theorem \ref{Theorem_toy}(ii), this holds if and only if $\{\bu=0\}$ has measure zero. Second, with the strict inequalities of \cite[Definition 1]{Zambrano-Zuniga-2}, the states taking the values $\pm1$ at points of $E$ are excluded, in contrast with \cite[Theorem 1]{Zambrano-Zuniga-2}; this matters only when $E$ has positive measure. Third, the relation \cite[(4.1)]{Zambrano-Zuniga-2}, which orders the states by the inclusion of $I_{+}\cup I_{-}$, is a preorder, not a partial order, since two states with the same support and opposite signs are related in both directions; the proposed meet $(I_{+}\cup I_{+}^{\prime},I_{-}\cup I_{-}^{\prime})$ is not a state when $I_{+}\cap I_{-}^{\prime}\neq\varnothing$; and, for this relation, two distinct bistable states have no common lower bound. Hence \cite[Theorem 2]{Zambrano-Zuniga-2} holds only in the degenerate case $F=\varnothing$, although its assertion that the bistable states are the extremal elements is correct (the order of \cite{Zambrano-Zuniga-2} is reversed with respect to ours). Theorem \ref{Theorem_order} gives the corrected order: a meet-semilattice whose principal ideals are Boolean and whose maximal elements are the bistable states, but not a lattice.
\end{remark}

\begin{remark}
For the pointwise dynamics (\ref{CNN_1}) with $a>1$ and $E=\varnothing$, at each $z\in F$ the values $a+\bu(z)$ and $\bu(z)-a$ are attracting equilibria of the scalar equation and $\bu(z)/(1-a)$ is repelling. Thus the attractors of the edge detector are the bistable states, in agreement with Theorem \ref{Theorem_hierarchical}(v).
\end{remark}

\section{\label{Section_7}Critical organization of recurrent $p$-adic DNNs}

By Theorem \ref{Theorem_A}, multiplicity of states can only occur for $q\geq1$, and, by Corollary \ref{Cor_feedforward}, it requires recurrence. In this section we study recurrent networks with zero or small input. The radial kernels with $k\geq0$ considered below are recurrent by Lemma \ref{Lemma_nilpotent}: since $k(x-y)=k(y-x)$, every connection is reciprocal, so these networks have feedback. We first show that the contraction threshold of Theorem \ref{Theorem_A} is not, in general, a transition in the sense of Section \ref{Section_8}, and then prove sharp transitions for radial kernels.

\subsection{The contraction threshold is not a bifurcation}

\begin{lemma}
\label{Prop_counterexample}(i) Let $\phi=\tanh$, so that $L_{\phi}=1$, $\bW(x,y)=-c$ with $c>0$, $\Win=0$ and $\bxi\in L^{2}(\Zp)$. Then $L_{\phi}\Vert\bW\Vert_{2}=L_{\phi}\opnorm{\T_{\bW}}=c$, and for every $c>0$ the network has exactly one state.

(ii) Let $\phi(s)=s$, so that $L_{\phi}=1$, $\bW(x,y)=c$ with $c\neq1$, $\Win=0$ and $\bxi\in L^{2}(\Zp)$. Then $L_{\phi}\opnorm{\T_{\bW}}=|c|$, and for every $c\neq1$ the network has exactly one state, $\bh=\bxi+\frac{c}{1-c}\int_{\Zp}\bxi(y)\,dy$.
\end{lemma}

\begin{proof}
(i) For $f\in L^{2}(\Zp)$, $\T_{\bW}f=-c\int_{\Zp}f(y)\,dy=-c\langle f,1\rangle1$, where $1$ is the constant function, which has norm one since $\Zp$ has measure one. Hence $\T_{\bW}$ has rank one and norm $c$, and $\Vert\bW\Vert_{2}=c$. The state equation (\ref{State_equation}) reads
\[
\bh=\bxi-c\int_{\Zp}\tanh(\bh(y))\,dy.
\]
If $\bh$ is a state, set $s=\int_{\Zp}\tanh(\bh(y))\,dy\in\R$; then $\bh=\bxi-cs$, and substituting this expression in the definition of $s$ gives
\begin{equation}
s=G(s):=\int_{\Zp}\tanh(\bxi(y)-cs)\,dy. \label{Equation_G}
\end{equation}
Conversely, if $s\in\R$ satisfies (\ref{Equation_G}), then $\bh:=\bxi-cs$ satisfies $\int_{\Zp}\tanh(\bh(y))\,dy=G(s)=s$, so $\bh$ is a state. Since $c>0$, different solutions $s$ give different states, and a state determines $s$. Hence the states are in one-to-one correspondence with the solutions of (\ref{Equation_G}).

The function $G$ is defined for every $s\in\R$, since $|\tanh|\leq1$ and $\Zp$ has measure one; thus $|G(s)|\leq1$. It is continuous, by the dominated convergence theorem. It is strictly decreasing: if $s_{1}<s_{2}$, then $\bxi(y)-cs_{1}>\bxi(y)-cs_{2}$ for every $y$, because $c>0$, so $\tanh(\bxi(y)-cs_{1})>\tanh(\bxi(y)-cs_{2})$ for every $y$, and integrating gives $G(s_{1})>G(s_{2})$. Therefore $F(s)=s-G(s)$ is continuous and strictly increasing, and $F(s)\rightarrow\pm\infty$ as $s\rightarrow\pm\infty$, because $G$ is bounded. By the intermediate value theorem, $F$ has exactly one zero, that is, (\ref{Equation_G}) has exactly one solution, and the network has exactly one state, for every $c>0$.

(ii) Here $\T_{\bW}f=c\langle f,1\rangle1$, so $\opnorm{\T_{\bW}}=|c|$. A state satisfies $\bh=c\int_{\Zp}\bh(y)\,dy+\bxi$. Integrating over $\Zp$ gives $\int_{\Zp}\bh=c\int_{\Zp}\bh+\int_{\Zp}\bxi$, so $\int_{\Zp}\bh=(1-c)^{-1}\int_{\Zp}\bxi$, and $\bh=\bxi+\frac{c}{1-c}\int_{\Zp}\bxi$. Conversely, this function is a state, as a direct substitution shows.
\end{proof}

\begin{remark}
\label{Remark_constant}Lemma \ref{Prop_counterexample} shows that crossing $L_{\phi}\opnorm{\T_{\bW}}=1$ need not create new states. The relevant threshold depends on the kernel class, as the next results show.
\end{remark}

\subsection{Radial kernels}
A kernel is \emph{radial} if $\bW(x,y)=k(x-y)$ with $k(y)=w(|y|_{p})$ for a function $w$ on $\{p^{-j}\}_{j\geq0}$. Thus $\T_{\bW}f=k\ast f$ is a convolution operator. Since the Haar measure $dx$ is invariant under translations and $\int_{\Zp}dx=1$,
\[
\Vert\bW\Vert_{2}^{2}=\int_{\Zp}\int_{\Zp}|k(x-y)|^{2}\,dy\,dx=\int_{\Zp}\Vert k\Vert_{L^{2}(\Zp)}^{2}\,dx=\Vert k\Vert_{L^{2}(\Zp)}^{2}.
\]

\begin{theorem}
\label{Theorem_radial}Consider a $p$-adic continuous DNN (Definition \ref{Definition_DNN}) with $\Win=0$ and $\bxi=0$, so that $\bu=0$ for every input $\bx$, with $\bW(x,y)=k(x-y)$ radial, $k\geq0$ and $k\in L^{2}(\Zp)$, and with an activation function $\phi$ which is odd, bounded, of class $C^{1}$, strictly concave on $[0,\infty)$, and satisfies $\phi^{\prime}(0)=1$; for example, $\phi=\tanh$. Thus the states are the solutions of $\bh=\T_{\bW}(\phi\circ\bh)$. Set $\lambda=\int_{\Zp}k(y)\,dy>0$. Then:
\begin{enumerate}
\item[(i)] $\opnorm{\T_{\bW}}=\lambda$;
\item[(ii)] if $\lambda\leq1$, $\bh=0$ is the only state;
\item[(iii)] if $\lambda>1$, the constant functions $0$, $\alpha_{\lambda}$ and $-\alpha_{\lambda}$ are states, where $\alpha_{\lambda}>0$ is the unique positive solution of $\alpha=\lambda\phi(\alpha)$;
\item[(iv)] if $\lambda>1$, then, for the dynamics $\partial_{t}\bh=-\bh+\T_{\bW}(\phi\circ\bh)$, the linearization $-I+\phi^{\prime}(\alpha)\T_{\bW}$ at the constant states $\alpha=\pm\alpha_{\lambda}$ has spectrum in $(-\infty,0)$, while at $\alpha=0$ it has the positive eigenvalue $\lambda-1$, with the constants as eigenfunctions.
\end{enumerate}
\end{theorem}

\begin{proof}
We first note that $L_{\phi}=1$ and $|\phi(s)-\phi(t)|<|s-t|$ for $s\neq t$. Indeed, $\phi^{\prime}$ is even, because $\phi$ is odd, and strictly decreasing on $[0,\infty)$, because $\phi$ is strictly concave there. Moreover $\phi^{\prime}\geq0$: if $\phi^{\prime}(r_{0})<0$ for some $r_{0}\geq0$, concavity gives $\phi(s)\leq\phi(r_{0})+\phi^{\prime}(r_{0})(s-r_{0})\rightarrow-\infty$ as $s\rightarrow\infty$, which contradicts boundedness. Hence $0\leq\phi^{\prime}(r)<\phi^{\prime}(0)=1$ for $r\neq0$, and, for $t<s$, $0\leq\phi(s)-\phi(t)=\int_{t}^{s}\phi^{\prime}(r)\,dr<s-t$; the Lipschitz constant $L_{\phi}=1$ is a supremum that is not attained. (i) By Young's inequality on the compact group $\Zp$, $\Vert k\ast f\Vert_{2}\leq\Vert k\Vert_{1}\Vert f\Vert_{2}=\lambda\Vert f\Vert_{2}$, and $k\ast1=\lambda$. (ii) Clearly $0$ is a state. For $\lambda<1$ uniqueness follows from Theorem \ref{Theorem_A}. For $\lambda=1$, let $\bh_{1},\bh_{2}$ be states with $\bh_{1}\neq\bh_{2}$ in $L^{2}(\Zp)$. Then $\Vert\bh_{1}-\bh_{2}\Vert_{2}>0$, that is, the set $\{x;\bh_{1}(x)\neq\bh_{2}(x)\}$ has positive measure; on this set $|\phi(\bh_{1}(x))-\phi(\bh_{2}(x))|<|\bh_{1}(x)-\bh_{2}(x)|$, and outside it both sides vanish. Hence
\[
\Vert\bh_{1}-\bh_{2}\Vert_{2}=\Vert k\ast(\phi\circ\bh_{1}-\phi\circ\bh_{2})\Vert_{2}\leq\Vert\phi\circ\bh_{1}-\phi\circ\bh_{2}\Vert_{2}<\Vert\bh_{1}-\bh_{2}\Vert_{2},
\]
a contradiction. (iii) For a constant $\alpha$, $k\ast\phi(\alpha)=\lambda\phi(\alpha)$. The function $g(\alpha)=\lambda\phi(\alpha)-\alpha$ is odd, strictly concave on $[0,\infty)$, with $g(0)=0$, $g^{\prime}(0)=\lambda-1>0$ and $g(\alpha)\rightarrow-\infty$ as $\alpha\rightarrow\infty$, since $\phi$ is bounded; hence it has exactly one positive zero $\alpha_{\lambda}$, and $g^{\prime}(\alpha_{\lambda})<0$, that is, $0\leq\lambda\phi^{\prime}(\alpha_{\lambda})<1$. (iv) $\T_{\bW}$ is self-adjoint ($k$ is even, being radial), with spectrum in $[-\lambda,\lambda]$ by (i). Hence the spectrum of $-I+\phi^{\prime}(\alpha_{\lambda})\T_{\bW}$ lies in $[-1-\lambda\phi^{\prime}(\alpha_{\lambda}),-1+\lambda\phi^{\prime}(\alpha_{\lambda})]\subset(-\infty,0)$, and the same holds at $-\alpha_{\lambda}$ because $\phi^{\prime}$ is even. At $\alpha=0$, $\phi^{\prime}(0)=1$ and $\T_{\bW}1=\lambda$.
\end{proof}

\begin{remark}
\label{Remark_meaning_iv}Part (iv) is what turns the count of states in (ii)--(iii) into a bifurcation of attractors. The states are the equilibria of the dynamics $\partial_{t}\bh=-\bh+\T_{\bW}(\phi\circ\bh)$ (Section \ref{Section_attractors}), and critical organization concerns which equilibria attract trajectories, not only how many equilibria there are. For $\lambda<1$ the zero state is globally attracting (Proposition \ref{Prop_attractor}). For $\lambda>1$, by (iv), the zero state is unstable along the constants, while $\pm\alpha_{\lambda}$ are linearly stable: at $\lambda=1$ one attractor splits into two, which is a pitchfork bifurcation. The eigenvalue $\lambda-1$ at $0$ also controls the relaxation time and the susceptibility, which diverge as $\lambda\rightarrow1^{-}$ (Corollary \ref{Cor_sharp}); and the number of linearly stable states is the memory of Section \ref{Section_exp_regimes} (see also Theorem \ref{Theorem_hierarchical}(v)).

Part (iv) also has a meaning for the $p$-adic discrete DNNs. Their layers are the Picard iterates of $\bh\mapsto\T_{\bW}(\phi\circ\bh)$ (the computation in the proof of Corollary \ref{Cor_Picard}(i) does not use $q<1$), and the derivative of this map at a constant state $\alpha$ is $\phi^{\prime}(\alpha)\T_{\bW}$. Its spectrum is contained in $[-\lambda\phi^{\prime}(\alpha_{\lambda}),\lambda\phi^{\prime}(\alpha_{\lambda})]\subset(-1,1)$ at $\pm\alpha_{\lambda}$, and it contains $\lambda>1$ at $0$. Thus, above the threshold, propagation through the layers is locally contracting near $\pm\alpha_{\lambda}$, while small inputs near $0$ with nonzero mean are amplified from layer to layer.

Finally, (iv) is a statement about the spectrum of the linearization (the G\^{a}teaux derivative). On $L^{2}(\Zp)$, the superposition operator $\bh\mapsto\phi\circ\bh$ is Fr\'{e}chet differentiable only if $\phi$ is affine \cite{Appell-Zabrejko}, so (iv) does not by itself imply nonlinear asymptotic stability in $L^{2}(\Zp)$. However, convolution commutes with translations, so the dynamics leaves each finite-dimensional space $\D^{l}(\Zp)$ invariant; there the right-hand side is of class $C^{1}$, and, by the principle of linearized stability, $\pm\alpha_{\lambda}$ are asymptotically stable and $0$ is unstable for the dynamics restricted to $\D^{l}(\Zp)$, for every $l$.
\end{remark}

\begin{remark}
\label{Remark_radial_phi}(i) The hypotheses on $\phi$ in Theorem \ref{Theorem_radial} are satisfied by $\tanh s$, $s/(1+|s|)$, $\frac{2}{\pi}\arctan(\frac{\pi}{2}s)$ and $\frac{\sqrt{\pi}}{2}\mathrm{erf}(s)$. They imply $L_{\phi}=\phi^{\prime}(0)=1$, so the transition from one state to several states occurs exactly at $\lambda=L_{\phi}\opnorm{\T_{\bW}}=1$, and the zero state loses stability there. If $\phi^{\prime}(0)=\gamma>0$ instead, applying the theorem to $\gamma^{-1}\phi$ and to the kernel $\gamma k$ shows that the threshold is $\gamma\lambda=1$. In particular, the normalization $\phi^{\prime}(0)=1$ cannot be replaced by $L_{\phi}<1$: in that case Theorem \ref{Theorem_A} gives a unique state for all $\lambda<1/L_{\phi}$, and (iii) fails for $1<\lambda<1/L_{\phi}$.
(ii) Strict concavity cannot be dropped. For the saturating function $\phi(s)=\frac{1}{2}(|s+1|-|s-1|)$, which is odd, bounded and has $\phi^{\prime}(0)=L_{\phi}=1$, but is linear on $[-1,1]$, every constant $c\in[-1,1]$ is a state when $\lambda=1$; hence (ii) fails at $\lambda=1$.
(iii) For signed kernels the theorem fails (Lemma \ref{Prop_counterexample}).
\end{remark}

\begin{corollary}
\label{Cor_sharp}Under the assumptions of Theorem \ref{Theorem_radial}, assume $\lambda<1$, so that $q=\lambda$, and consider the network with input term $\bu$, $\bh=\T_{\bW}(\phi\circ\bh)+\bu$, whose state for $\bu=0$ is $\bh^{\ast}=0$.
\begin{enumerate}
\item[(i)] The linearization (G\^{a}teaux derivative) $-I+\T_{\bW}$ of the right-hand side of (\ref{Dynamics}) at $\bh^{\ast}=0$ has spectrum in $[-1-\lambda,-1+\lambda]$, and $-1+\lambda$ is an eigenvalue, with the constant functions as eigenfunctions. Hence the linear relaxation rate at $0$ is exactly $1-q$; since constant initial conditions $s_{0}$ evolve by $\dot{s}=-s+\lambda\phi(s)$, whose solutions decay like $e^{-(1-\lambda)t}$ as $s_{0}\rightarrow0$, the rate in Proposition \ref{Prop_attractor} cannot be improved.
\item[(ii)] For $\bu=\varepsilon c$, with $c\in\R\smallsetminus\{0\}$ and $\varepsilon\in\R$, the state $\bh_{\varepsilon}$ is constant, and
\[
\lim_{\varepsilon\rightarrow0}\frac{\Vert\bh_{\varepsilon}-0\Vert_{2}}{\Vert\varepsilon c\Vert_{2}}=\frac{1}{1-\lambda}=\frac{1}{1-q}.
\]
Hence the bound $\Vert\bh(\bu^{\prime})-\bh(\bu)\Vert_{2}\leq(1-q)^{-1}\Vert\bu^{\prime}-\bu\Vert_{2}$, which is the content of the proof of Corollary \ref{Cor_robust}(i), cannot be improved.
\end{enumerate}
As $\lambda\rightarrow1^{-}$, the relaxation time and the susceptibility at $0$ diverge like $(1-\lambda)^{-1}$; at $\lambda=1$ the attractor $0$ loses its stability, and for $\lambda>1$ it splits into the attractors $\pm\alpha_{\lambda}$ (Theorem \ref{Theorem_radial}).
\end{corollary}

\begin{proof}
(i) Since $\phi^{\prime}(0)=1$, $|\phi(s)|\leq|s|$ and $\phi(\varepsilon\boldsymbol{v})/\varepsilon\rightarrow\boldsymbol{v}$ pointwise, we have $\phi(\varepsilon\boldsymbol{v})/\varepsilon\rightarrow\boldsymbol{v}$ in $L^{2}(\Zp)$ by dominated convergence, the linearization at $0$ is $-I+\T_{\bW}$, and the assertion follows from Theorem \ref{Theorem_radial}(i) and (iv). (ii) The constant function $s_{\varepsilon}$, where $s=s_{\varepsilon}$ solves $s=\lambda\phi(s)+\varepsilon c$, is a state because $\T_{\bW}1=\lambda$; it is the only state by Theorem \ref{Theorem_A}, since $q=\lambda<1$. The function $g(s)=s-\lambda\phi(s)$ is of class $C^{1}$ with $g^{\prime}(0)=1-\lambda>0$, so, by the inverse function theorem, $s_{\varepsilon}=g^{-1}(\varepsilon c)$ is of class $C^{1}$ in $\varepsilon$ with $s_{0}=0$ and $ds_{\varepsilon}/d\varepsilon|_{\varepsilon=0}=c/(1-\lambda)$. Since $\Vert\bh_{\varepsilon}\Vert_{2}=|s_{\varepsilon}|$ and $\Vert\varepsilon c\Vert_{2}=|\varepsilon c|$, the limit follows.
\end{proof}

\subsection{\label{Section_hierarchical}Hierarchical kernels}
We now take the radial kernel that couples strongly the neurons inside each ball of radius $p^{-l}$ and weakly all neurons:
\begin{equation}
\bW_{a,b}^{(l)}(x,y)=a\,p^{l}\Om{l}{x-y}+b,\qquad\T_{\bW_{a,b}^{(l)}}=a\boldsymbol{P}_{l}+b\boldsymbol{P}_{0}, \label{Hierarchical_kernel}
\end{equation}
where $\boldsymbol{P}_{l}$ is the projection (\ref{Projection}) and $\boldsymbol{P}_{0}f=\int_{\Zp}f$. Its operator norm is $\max\{|a+b|,|a|\}$ and its Hilbert--Schmidt norm is $(a^{2}p^{l}+2ab+b^{2})^{1/2}\geq|a|p^{l/2}-|b|$; so the criterion $L_{\phi}\Vert\bW\Vert_{2}<1$ deteriorates as $l$ grows, while $L_{\phi}\opnorm{\T_{\bW}}<1$ does not.

For $a>1$ let $\psi_{a}(s)=s-a\tanh s$. It has a local maximum $M_{a}>0$ at $-s_{a}$ and a local minimum $-M_{a}$ at $s_{a}$, where $\cosh^{2}s_{a}=a$, and for $|c|<M_{a}$ the equation $\psi_{a}(s)=c$ has exactly three solutions $\psi_{a,-1}^{-1}(c)<-s_{a}<\psi_{a,0}^{-1}(c)<s_{a}<\psi_{a,1}^{-1}(c)$, depending smoothly on $c$. For $\sigma\in V^{G_{l}}$ we write $\bh_{\sigma}\in\D^{l}(\Zp)$ for the function with $\bh_{\sigma}(I)=\psi_{a,\sigma(I)}^{-1}(0)$; thus $\bh_{\sigma}(I)\in\{-\alpha_{a},0,\alpha_{a}\}$ with sign $\sigma(I)$.

\begin{theorem}
\label{Theorem_hierarchical}Let $\phi=\tanh$, zero input, and the kernel (\ref{Hierarchical_kernel}) with $l\geq0$.
\begin{enumerate}
\item[(i)] Every state lies in $\D^{l}(\Zp)$, and $\bh\in\D^{l}(\Zp)$ is a state if and only if
\begin{equation}
\psi_{a}(\bh(I))=b\,m(\bh),\qquad m(\bh):=p^{-l}\sum_{J\in G_{l}}\tanh\bh(J),\qquad I\in G_{l}. \label{Hier_system}
\end{equation}
\item[(ii)] If $a\geq0$, $b\geq0$ and $a+b\leq1$, then $\bh=0$ is the only state.
\item[(iii)] If $b=0$ and $a>1$, the states are exactly the $3^{p^{l}}$ functions $\bh_{\sigma}$, $\sigma\in V^{G_{l}}$.
\item[(iv)] If $a>1$, there is $b_{0}(a)>0$, independent of $l$, such that for $|b|<b_{0}(a)$ the network has exactly $3^{p^{l}}$ states $\bh_{\sigma}^{b}$, $\sigma\in V^{G_{l}}$, where $\bh_{\sigma}^{b}(I)=\psi_{a,\sigma(I)}^{-1}(b\,m_{\sigma})$ for a unique $m_{\sigma}\in(-1,1)$.
\item[(v)] Under the assumptions of (iv), consider the dynamics $\partial_{t}\bh=-\bh+\T_{\bW}(\tanh\bh)$ on the invariant space $\D^{l}(\Zp)\simeq\R^{p^{l}}$. The state $\bh_{\sigma}^{b}$ is hyperbolic, its unstable manifold has dimension $|G_{l}\smallsetminus\mathrm{supp}\,\sigma|$, and it is asymptotically stable if and only if $\mathrm{supp}\,\sigma=G_{l}$. In particular, exactly $2^{p^{l}}$ states are asymptotically stable, and the map $\sigma\mapsto\bh_{\sigma}^{b}$ identifies the set of states with the poset $(V^{G_{l}},\sqsubseteq)$ of Lemma \ref{Lemma_poset}, graded by the number of stable directions.
\end{enumerate}
\end{theorem}

\begin{proof}
(i) The range of $a\boldsymbol{P}_{l}+b\boldsymbol{P}_{0}$ is contained in $\D^{l}(\Zp)$, so every state lies there. For $\bh\in\D^{l}(\Zp)$, $\boldsymbol{P}_{l}\tanh\bh=\tanh\bh$ and $\boldsymbol{P}_{0}\tanh\bh=m(\bh)$, which gives (\ref{Hier_system}).

(ii) For $a,b\geq0$ the kernel is radial and nonnegative with $\lambda=a+b$; apply Theorem \ref{Theorem_radial}.

(iii) For $b=0$, (\ref{Hier_system}) decouples into $\psi_{a}(\bh(I))=0$ for each $I$, whose solutions are $0,\pm\alpha_{a}$.

(iv) Fix $c_{1}\in(0,M_{a})$ and let $\kappa(c_{1})=\min\{|\psi_{a}^{\prime}(s)|;\ s\in\psi_{a}^{-1}([-c_{1},c_{1}])\}>0$; the minimum is positive because $\psi_{a}^{-1}([-c_{1},c_{1}])$ is a compact set not containing $\pm s_{a}$. Set $b_{0}(a)=\sup_{c_{1}\in(0,M_{a})}\min\{c_{1},\kappa(c_{1})\}>0$, and let $|b|<b_{0}(a)$; choose $c_{1}$ with $|b|<\min\{c_{1},\kappa(c_{1})\}$. If $\bh$ is a state, $|b\,m(\bh)|<c_{1}<M_{a}$, so each $\bh(I)$ lies on one of the three branches: $\bh(I)=\psi_{a,\sigma(I)}^{-1}(b\,m(\bh))$ for a unique $\sigma\in V^{G_{l}}$. Conversely, for given $\sigma$, the function defined in this way is a state if and only if $m=m(\bh)$ solves
\[
m=\Phi_{\sigma}(m):=p^{-l}\sum_{I\in G_{l}}\tanh\bigl(\psi_{a,\sigma(I)}^{-1}(bm)\bigr),\qquad m\in[-1,1].
\]
$\Phi_{\sigma}$ maps $[-1,1]$ into $(-1,1)$, and $|\Phi_{\sigma}^{\prime}(m)|\leq|b|/\kappa(c_{1})<1$ since $\tanh^{\prime}\leq1$ and $|(\psi_{a,j}^{-1})^{\prime}|\leq1/\kappa(c_{1})$ on $[-c_{1},c_{1}]$. By the contraction principle, $\Phi_{\sigma}$ has a unique fixed point $m_{\sigma}$. Hence the states are in bijection with $V^{G_{l}}$.

(v) In the basis $\{\Om{l}{x-I}\}$, the vector field is $F_{I}(\bh)=-\bh(I)+a\tanh\bh(I)+b\,m(\bh)$, which is smooth, so the classical linearization theorem applies. Its Jacobian at a state is $J=-\mathrm{diag}(\psi_{a}^{\prime}(\bh(I)))+bp^{-l}\mathbf{1}\mathbf{1}^{T}D$, with $D=\mathrm{diag}(\tanh^{\prime}\bh(I))$ positive definite. The matrix $D^{1/2}JD^{-1/2}$ is symmetric, and differs from the diagonal matrix $-\mathrm{diag}(\psi_{a}^{\prime}(\bh(I)))$ by $bp^{-l}D^{1/2}\mathbf{1}\mathbf{1}^{T}D^{1/2}$, of norm at most $|b|$. On the outer branches $\psi_{a}^{\prime}\geq\kappa(c_{1})$ and on the middle branch $\psi_{a}^{\prime}\leq-\kappa(c_{1})$. By Weyl's inequality and $|b|<\kappa(c_{1})$, $J$ has exactly $|G_{l}\smallsetminus\mathrm{supp}\,\sigma|$ positive eigenvalues and the rest negative.
\end{proof}

\begin{example}
\label{Example_hierarchical}Let $p=2$, $l=1$, $a>1$ and $b=0$, so that $\T_{\bW}=a\boldsymbol{P}_{1}$. By Theorem \ref{Theorem_hierarchical}(iii), the states are the $3^{2}=9$ functions $\bh_{\sigma}$, $\sigma\in V^{S}$, $S=G_{1}=\{0,1\}$, with $\bh_{\sigma}=\alpha_{a}\sigma(I)$ on the ball $I+2\mathbb{Z}_{2}$, where $\alpha_{a}=a\tanh\alpha_{a}$. By Theorem \ref{Theorem_hierarchical}(v), the four states with $\sigma=(\pm1,\pm1)$ are asymptotically stable, the four states with exactly one zero are saddles with one unstable direction, and $\bh_{(0,0)}=0$ has two unstable directions. This is exactly the structure of Example \ref{Example_toy}, with $\alpha_{a}$ in place of $a$: the poset of states is $(V^{S},\sqsubseteq)$, with the Hasse diagram of Figure \ref{Fig_regimes}(c). By Theorem \ref{Theorem_hierarchical}(iv), this structure persists for $|b|<b_{0}(a)$, when all neurons are coupled; then the states $\bh_{\sigma}^{b}$ are no longer constant multiples of the sign pattern, but they are still indexed by $V^{S}$ and graded by instability in the same way. In contrast with the toy model, where each point evolves independently, here the order structure survives a coupling between all neurons.
\end{example}

\begin{remark}
Theorem \ref{Theorem_hierarchical} shows in a rigorous model how hierarchy and criticality are connected. Below the threshold there is one state. Above it, the states are indexed by sign patterns on the leaves of the tree $G_{l}$, the attracting ones are the $2^{p^{l}}$ fully polarized patterns, and the remaining states are saddles whose instability is measured by the number of unpolarized leaves. The kernel $a\boldsymbol{P}_{l}$ is the restriction to resolution $l$ of the pointwise model of Section \ref{Section_8}: $\boldsymbol{P}_{l}\rightarrow I$ strongly as $l\rightarrow\infty$. The same proof applies to inputs $\bu\in\D^{l}(\Zp)$ with $\max_{I}|\bu(I)|$ small, and to more general radial kernels that are small perturbations of $a\boldsymbol{P}_{l}$ in operator norm.
\end{remark}

\begin{remark}
\label{Remark_chi}The threshold in Theorems \ref{Theorem_radial}--\ref{Theorem_hierarchical} is of a different nature from the edge-of-chaos threshold $\chi_{1}=\sigma_{w}^{2}\mathbb{E}[\phi^{\prime}(h)^{2}]=1$ of deep networks at initialization \cite{Poole et al}, \cite{Schoenholz-DIP}, which concerns the propagation of correlations through random layers, and from the transition at $g=1$ for random recurrent networks \cite{Sompolinsky et al}, \cite{Wainrib-Touboul}, which concerns the spectral radius of a random matrix. For random recurrent networks, our sufficient condition $2gL_{\phi}<1$ (Section \ref{Section_continuous_limits}) is a factor two from that transition. A relation between these thresholds and those proved here is part of Conjecture \ref{Conj_random}.
\end{remark}

\section{\label{Section_experiments}Numerical experiments}

We illustrate the results of Section \ref{Section_7} and test Conjectures \ref{Conj_tree} and \ref{Conj_random} on finite $p$-adic DNNs. By Remark \ref{Theorem_exact}, when the parameters are test functions of level $l$, the states of the $p$-adic continuous DNN are exactly the solutions of a system of $p^{l}$ equations. We consider radial kernels $\T_{\bW}=\sum_{j=0}^{l}c_{j}\boldsymbol{P}_{j}$ and zero input, where $\boldsymbol{P}_{j}f(x)=p^{j}\int_{x+p^{j}\Zp}f(y)\,dy$ is the orthogonal projection (\ref{Projection}) of $L^{2}(\Zp)$ onto $\D^{j}(\Zp)$, that is, the average of $f$ over the ball of radius $p^{-j}$ containing $x$. Since $\D^{j}(\Zp)\subset\D^{l}(\Zp)$ for $j\leq l$, each $\boldsymbol{P}_{j}$ maps $\D^{l}(\Zp)$ into itself, its restriction to $\D^{l}(\Zp)\simeq\R^{p^{l}}$, in the coordinates $\bh\mapsto(\bh(I))_{I\in G_{l}}$, is given by the matrix $P_{j}=[(P_{j})_{I,K}]_{I,K\in G_{l}}$, with $(P_{j})_{I,K}=p^{j-l}$ if $I\equiv K\bmod p^{j}$ and $0$ otherwise. Thus the state equation reads $\bh=\sum_{j}c_{j}P_{j}\tanh\bh$ on $\R^{p^{l}}$. For general radial kernels we take the profile $c_{j}=\lambda\,2^{-j}/\sum_{i=0}^{l}2^{-i}$, so that $\lambda=\sum_{j}c_{j}=\Vert\T_{\bW}\Vert$ is the total coupling of Theorem \ref{Theorem_radial}. States are computed by Newton's method from many initial conditions, and their linear stability is read from the eigenvalues of the Jacobian $-I+\T_{\bW}\,\mathrm{diag}(1-\tanh^{2}\bh)$ of the dynamics (\ref{Dynamics}). Counts obtained from many initial conditions are lower bounds; in Section \ref{Section_exp_hierarchical} the count is exact. The numerical techniques used are standard: the Benettin algorithm for Lyapunov exponents \cite{Benettin et al}, Newton's method \cite{Deuflhard} and the classical Runge--Kutta method \cite{Hairer et al}; the code is written in Python with NumPy and SciPy \cite{Harris et al}, \cite{Virtanen et al}. The simulation code, the data, and the scripts that produce the figures are available at \url{https://wazuniga.github.io}. Claude (Anthropic; model \mbox{claude-opus-5-5}) was used to help implement the simulation code of this section; the author checked the code and its results, as described in the Acknowledgments.

\subsection{\label{Section_exp_hierarchical}Exact count for hierarchical kernels}
For the kernel $a\boldsymbol{P}_{l}+b\boldsymbol{P}_{0}$ of Theorem \ref{Theorem_hierarchical}, the state equation on $\D^{l}(\Zp)$ reduces to $\psi_{a}(\bh(I))=bm$, $I\in G_{l}$, with $m=p^{-l}\sum_{K}\tanh\bh(K)$. Hence all values $\bh(I)$ solve $\psi_{a}(s)=c$ for the same $c=bm$, and a state is determined by the numbers $(n_{-},n_{0},n_{+})$ of neurons on each of the three branches of $\psi_{a}$, together with a root $c$ of a scalar equation; we call such a pair a \emph{class}, and each class contains a multinomial number of states, which differ by a permutation of the neurons. This enumerates all states exactly, for every $b$, up to the numerical computation of the roots. We take $p=2$, $l=4$ ($N=16$ neurons, $3^{16}$ sign patterns), $a\in\{1.5,2,3\}$ and $b$ on a grid of step $0.025$ in $[-3,3]$; as a check of accuracy, for $a=2$, $b=0.5$, the states of all $153$ classes satisfy the state equation up to an error (residual) $\max_{I}|\bh(I)-a\tanh\bh(I)-bm|$ below $10^{-8}$.

We recall the terminology used in this section. A \emph{state} is a solution of the state equation, that is, an equilibrium of the dynamics (\ref{Dynamics}) restricted to $\D^{l}(\Zp)\simeq\R^{N}$, with $N=p^{l}$. A state is \emph{stable} if all the eigenvalues of the Jacobian of the dynamics at the state have negative real part; it is then asymptotically stable, that is, an attractor. The other states are unstable; they are not observed as long-time limits of the dynamics, but they separate the basins of attraction of the stable states. By Theorem \ref{Theorem_hierarchical}(v), for $|b|<b_{0}(a)$ the unstable states are saddles, and the stable states are the $2^{N}$ fully polarized sign patterns.

Figure \ref{Fig_hierarchical} shows the number of states and of stable states, normalized by $3^{N}$ and $2^{N}$. The counts of Theorem \ref{Theorem_hierarchical}, $3^{N}$ states of which exactly $2^{N}$ are stable, hold on a range of $b$ much larger than the range $|b|<b_{0}(a)$ covered by the proof: for $a=1.5,2,3$, the proof gives $b_{0}(a)\approx0.19,0.44,0.82$, while the counts hold for $b\in[-0.35,0.25]$, $[-0.75,0.60]$ and $[-1.575,1.45]$, respectively. Beyond this range the counts decrease in steps, as families of sign patterns with given numbers $(n_{-},n_{0},n_{+})$ disappear. The mechanism is the following: a pattern keeps its three branches only while $|bm|<M_{a}$, so a strong global coupling of either sign first eliminates the imbalanced patterns, those with large $|m|$. The two signs then act differently. Global inhibition ($b<0$) drives the network towards balanced patterns and removes mainly unstable states; the stable states that survive are the nearly balanced ones, and they are still a large fraction of $2^{N}$ (for $b=-2$, about $79\%$ for $a=1.5$ and $92\%$ for $a=2$). Global excitation ($b>0$) drives the network towards consensus and destabilizes most states: for $a=1.5$ and $b\geq1$, the only stable states are the exactly balanced patterns, $n_{+}=n_{-}=8$, and the two consensus states, that is, $\binom{16}{8}+2$ states, about $20\%$ of $2^{N}$. Hence the number of stable states, the memory of the network, is maximal when the local coupling $a$ dominates and the global coupling is weak; a strong global coupling of either sign competes with the hierarchical organization, as in the experiments of Section \ref{Section_exp_cascade} and in the hypothesis of Conjecture \ref{Conj_tree}.

\begin{figure}[ht]
\centering
\includegraphics[width=\textwidth]{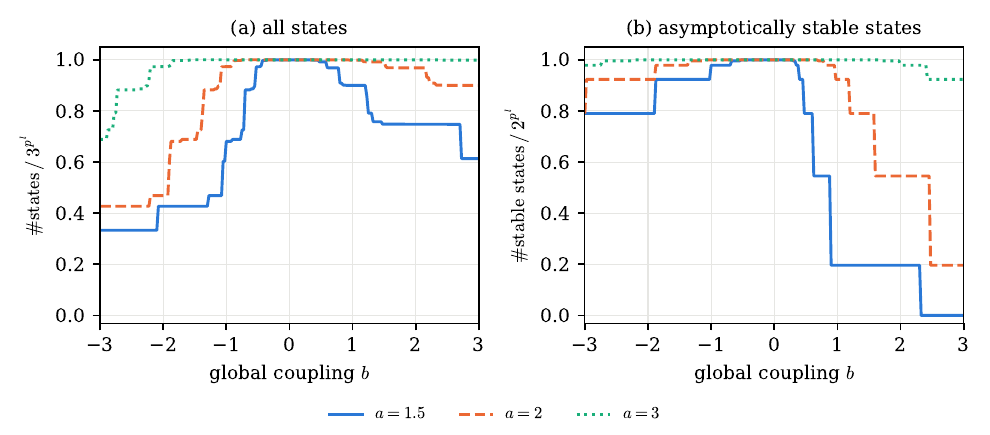}
\caption{\label{Fig_hierarchical}Exact number of states (a) and of asymptotically stable states (b) for the hierarchical kernel $a\boldsymbol{P}_{l}+b\boldsymbol{P}_{0}$, $p=2$, $l=4$ ($3^{16}$ sign patterns), normalized by $3^{p^{l}}$ and $2^{p^{l}}$, as functions of the global coupling $b$, for $a=1.5$ (solid blue), $a=2$ (dashed orange) and $a=3$ (dotted green). Theorem \ref{Theorem_hierarchical} guarantees the value $1$ in both panels for $|b|<b_{0}(a)$, with $b_{0}(a)\approx0.19$, $0.44$ and $0.82$, respectively; the computed counts keep the value $1$ on much larger intervals of $b$, which are asymmetric.}
\end{figure}

\subsection{\label{Section_exp_cascade}A cascade for general radial kernels}
Conjecture \ref{Conj_tree} concerns general radial kernels. A naive reading suggests a cascade: with the profile of Section \ref{Section_experiments}, the coupling inside the balls of level $k$ is $A_{k}=\sum_{j\geq k}c_{j}=\lambda a_{k}$, with $a_{k}=\sum_{j\geq k}2^{-j}/\sum_{i=0}^{l}2^{-i}$, which exceeds one for $\lambda>\lambda_{k}=1/a_{k}$, and one might expect $3^{p^{k}}$ states, of which $2^{p^{k}}$ are stable, once $\lambda>\lambda_{k}$. We computed the states for $p=2$, $l=3$ ($8$ neurons, $\lambda_{k}=1,2.14,5,15$) and $p=3$, $l=2$ ($9$ neurons, $\lambda_{k}=1,2.33,7$), for $\lambda\in[0.6,30]$, starting Newton's method from all $3^{N}$ sign patterns at two amplitudes and from $4000$ random initial conditions.

Figure \ref{Fig_cascade} shows that the first transition occurs exactly at $\lambda_{0}=1$, but the later ones are delayed and partial. For $p=2$, $l=3$, we found $17$ states, of which $4$ are stable, at $\lambda=5.3$, $49$ states ($8$ stable) at $\lambda=10.6$, and $1137$ states ($8$ stable) at $\lambda=30$, instead of $3^{8}=6561$ and $2^{8}=256$; for $p=3$, $l=2$, we found $4491$ states ($62$ stable) at $\lambda=30$, instead of $3^{9}$ and $2^{9}$. This does not contradict Conjecture \ref{Conj_tree}, whose hypothesis $\beta_{k}<M_{A_{k}}$ bounds the coupling between the balls of level $k$. Since $M_{A}=A\tanh s_{A}-s_{A}<A$ for every $A>1$, the hypothesis implies $\beta_{k}<A_{k}$; for this profile, $\beta_{k}=\lambda(1-a_{k})$ and $A_{k}=\lambda a_{k}$, so the hypothesis requires $a_{k}>1/2$. But $a_{1}=7/15$ for $p=2$ and $a_{1}=3/7$ for $p=3$, so the hypothesis fails at the levels $k\geq1$ for every $\lambda$. The experiment shows that the hypothesis on the coupling between the balls is essential: a strong coupling between the balls suppresses the cascade.

\begin{figure}[ht]
\centering
\includegraphics[width=\textwidth]{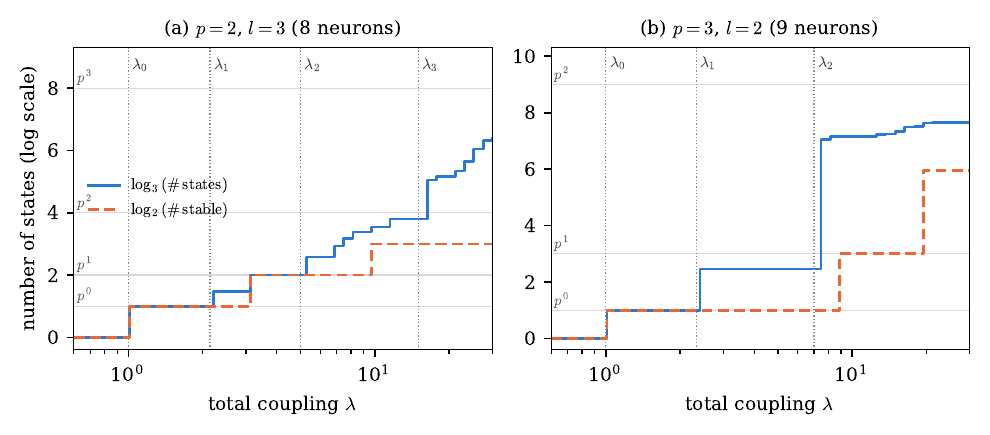}
\caption{\label{Fig_cascade}Number of states ($\log_{3}$, solid) and of linearly stable states ($\log_{2}$, dashed) for the radial kernel with profile $c_{j}\propto2^{-j}$, as functions of the total coupling $\lambda$, for (a) $p=2$, $l=3$ and (b) $p=3$, $l=2$. Vertical lines: the thresholds $\lambda_{k}=1/a_{k}$; horizontal lines: the values $p^{k}$ predicted by a naive cascade.}
\end{figure}

\subsection{\label{Section_exp_random}Random recurrent networks on the tree}
We consider the random networks of Conjecture \ref{Conj_random}, $\dot{\bh}=-\bh+M\tanh\bh$, where $M$ is a $p^{l}\times p^{l}$ matrix with independent centered Gaussian entries of variance $p^{-l}v(|I-K|_{p})$, $p=2$, for two profiles normalized so that $\rho(v)^{2}=\sum_{K\in G_{l}}p^{-l}v(|K|_{p})$ takes prescribed values: the flat profile $v$ constant, which is the model of \cite{Sompolinsky et al}, \cite{Wainrib-Touboul}, and the local profile $v(d)\propto(d+p^{-l})^{-1}$, concentrated on small $p$-adic distances. For each $\rho(v)$ and $l=3,4,5$ ($8$, $16$, $32$ neurons), we computed the equilibria of $8$ samples of $M$ by Newton's method from $3000$ random initial conditions; the counts are lower bounds. For $l=6$ ($64$ neurons), we computed the largest Lyapunov exponent of the dynamics, for $M$ and for its symmetric part $(M+M^{T})/\sqrt{2}$, which has the same variance profile, averaged over $6$ samples.

Figure \ref{Fig_random}(a) shows that for the flat profile and $\rho(v)<1$ the zero state was the only equilibrium in every sample; for the local profile, a few additional equilibria appear for small networks at $\rho(v)=0.75$, a finite-size effect of the large self-coupling of this profile. For $\rho(v)>1$ the number of equilibria grows with $\rho(v)$ and, for the local profile, with the size: at $\rho(v)=2.25$ the mean numbers found are $4$, $25$ and $103$ for $8$, $16$ and $32$ neurons. For the flat profile the growth is slow at these sizes ($3$, $7$ and $8$), in agreement with the small exponential rate of \cite{Wainrib-Touboul} near $\rho=2$. These observations are consistent with the first part of Conjecture \ref{Conj_random}. On the other hand, we did not observe that the equilibria are approximately constant on balls when $v$ concentrates on small distances: for $32$ neurons and $\rho(v)=2$, the mean fraction $\Vert\boldsymbol{P}_{4}\bh\Vert_{2}^{2}/\Vert\bh\Vert_{2}^{2}$ of the energy of the nonzero equilibria carried by the functions constant on the balls of level $4$ is $0.45$ for the local profile and $0.51$ for the flat one.

Figure \ref{Fig_random}(b) illustrates Section \ref{Section_gradient}: for symmetric couplings the largest Lyapunov exponent is negative for every $\rho(v)$, and the trajectories converge to equilibria. For asymmetric couplings and the flat profile, the largest exponent approaches zero for $\rho(v)\gtrsim1.75$, and in most samples the trajectory does not converge to an equilibrium; this indicates non-stationary attractors. A positive exponent, that is, chaos, is not established for $64$ neurons.

\begin{figure}[ht]
\centering
\includegraphics[width=\textwidth]{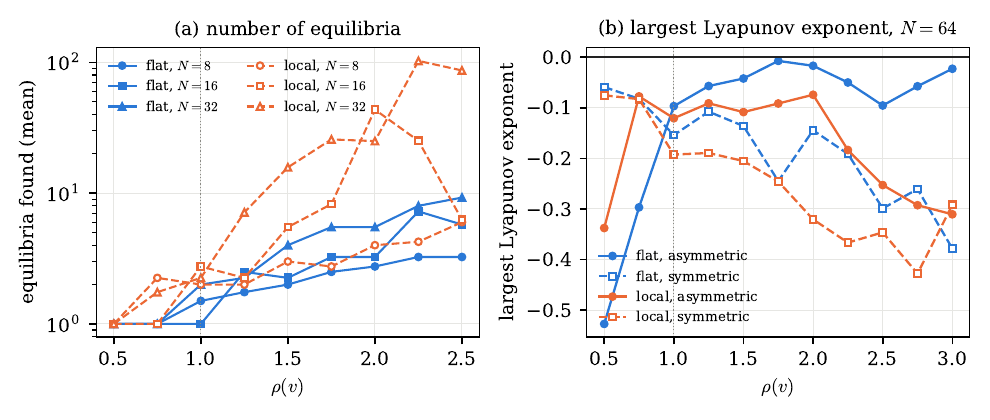}
\caption{\label{Fig_random}Random recurrent networks on the tree, $p=2$. (a) Mean number of equilibria found (lower bounds, $8$ samples) as a function of $\rho(v)$, for the flat and local profiles and $8$, $16$, $32$ neurons. (b) Largest Lyapunov exponent for $64$ neurons, for the asymmetric matrix $M$ and for its symmetric part ($6$ samples).}
\end{figure}

\subsection{\label{Section_exp_regimes}The three regimes}
We compute, at a stable state $\bh$, two quantities obtained from the linearization $J=-I+\T_{\bW}D\phi(\bh)$ of the dynamics (\ref{Dynamics}): the relaxation time $\tau=1/(-\max\mathrm{Re}\,\sigma(J))$, where $\sigma(J)$ is the spectrum of $J$, and the susceptibility $\chi=\Vert(I-\T_{\bW}D\phi(\bh))^{-1}\Vert$, the norm of the linear response of the state to a perturbation of the input (Section \ref{Section_attractors}). We measure memory by $\log_{2}$ of the number of linearly stable states. Figure \ref{Fig_regimes} shows these quantities for the radial kernel with the profile of Section \ref{Section_experiments} and for the hierarchical kernel $a\boldsymbol{P}_{l}+b\boldsymbol{P}_{0}$ with $b=0.05$, $p=2$, $l=4$, at the stable state reached from a small random initial condition; for these kernels $\tau$ and $\chi$ agree within the precision of the computation.

For the radial kernel, $\tau$ and $\chi$ peak at $\lambda=1$, where the memory jumps from $0$ to $1$ bit. For the hierarchical kernel there are two critical points: at $a=1-b$ the constant mode becomes unstable and two polarized stable states appear, which is the critical point of the coarsest level of the tree; at $a=1$ the fine-scale modes become unstable, and the number of stable states grows from $2$ to $12872$ at $a=1.04$ and to $2^{16}=65536$ for $a\geq1.15$, so that the memory reaches $p^{l}=16$ bits. Thus, in hierarchical networks, criticality depends on the scale; for general radial kernels the critical points of the levels are the thresholds $\lambda_{k}$ of Section \ref{Section_exp_cascade}, and their effect depends on the coupling between the balls. These observations are discussed in Section \ref{Section_brain}.

\begin{figure}[ht]
\centering
\includegraphics[width=\textwidth]{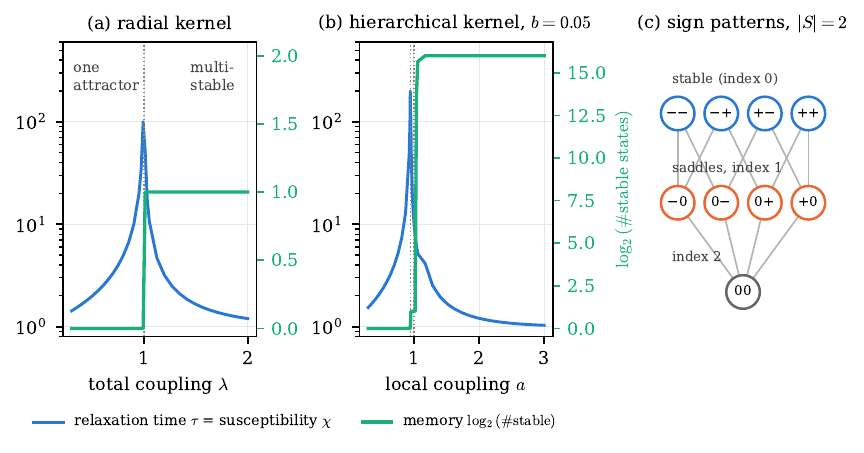}
\caption{\label{Fig_regimes}The three regimes. (a) Radial kernel with profile $c_{j}\propto2^{-j}$, $p=2$, $l=4$: relaxation time $\tau$ and susceptibility $\chi$ (left axis, they coincide) at the stable state reached from a small random initial condition, and memory $\log_{2}(\#\text{stable states})$ (right axis), as functions of the total coupling $\lambda$. (b) The same for the hierarchical kernel $a\boldsymbol{P}_{l}+b\boldsymbol{P}_{0}$, $b=0.05$, $p=2$, $l=4$; the dotted lines are $a=1-b$ and $a=1$. (c) Hasse diagram of the poset $V^{S}$ of sign patterns for $|S|=2$, with the unstable dimension of the corresponding states.}
\end{figure}

\section{\label{Section_conjectures}Discussion, conjectures and open problems}

\subsection{\label{Section_continuous_limits}Continuous limits of finite networks}
A continuous limit of a family $\btheta_{N}$ of finite networks with an increasing number of neurons is a $p$-adic continuous DNN whose states and outputs are the limits in $L^{2}(\Zp)$ of those of $\btheta_{N}$. The existence of such a limit is a property of the family, not of its members: a single finite network is exactly a $p$-adic continuous DNN with test functions as parameters (Remark \ref{Theorem_exact}), and it is the coarse-graining of many continuous networks, which differ by components invisible at its level.

\emph{The stable region.} By Theorem \ref{Theorem_continuity}, if $\btheta_{N}\rightarrow\btheta$ in $\X$ and $\btheta\in\X_{\mathrm{stable}}$, then $\bh(\btheta_{N},\bx)\rightarrow\bh(\btheta,\bx)$ and $\by(\btheta_{N},\bx)\rightarrow\by(\btheta,\bx)$; conversely, every $\btheta\in\X_{\mathrm{stable}}$ is such a limit (Remark \ref{Theorem_exact}(iv)). Thus, in the stable region, the continuous limits of finite networks whose kernels converge in $L^{2}$ are exactly the $p$-adic continuous DNNs of Definition \ref{Definition_DNN}. In particular, the families of Definition \ref{Definition-3}, built by coarse-graining a single kernel, always converge (Theorem \ref{Theorem_depth}).

\emph{Families without a continuous limit.} The hypothesis that the kernels converge cannot be dropped. Consider the network (\ref{Formula_1}) in its matrix form $\bh=M\phi(\bh)+M_{\mathrm{in}}\bx^{(L)}+\bxi$, with $M=[p^{-L-\Delta}\bW(I,K)]_{I,K\in G_{L+\Delta}}$, and assume, as for standard networks at initialization, that the entries of $M$ are independent Gaussian variables with distribution $\mathcal{N}(0,g^{2}/n)$, where $n=p^{L+\Delta}$ is the width. The corresponding kernel is $\bW(I,K)=p^{L+\Delta}M_{I,K}$. Then $\mathbb{E}\Vert\bW\Vert_{2}^{2}=g^{2}n\rightarrow\infty$, while $\Vert M\Vert\rightarrow2g$ almost surely \cite[Chapter 4]{Vershynin}. Hence the stability condition (\ref{Stability}) holds uniformly in the width when $2gL_{\phi}<1$, but the kernels do not converge in $L^{2}(\Zp\times\Zp)$, and there is no deterministic continuous limit; the known infinite-width limits of these networks are limits of laws \cite{Lee et al}, \cite{Matthews et al}. Bounded families can also fail to converge because of oscillations at small scales: if $\Win=a_{N}(x)$, where the $a_{N}\in\D^{N+1}(\Zp)$ are orthonormal $p$-adic wavelets of Haar type \cite{KKZuniga}, and $\bW=0$, $\bxi=0$, then the states $\bigl(\int_{\Zp}\bx\,dy\bigr)a_{N}$ have no convergent subsequence when $\int_{\Zp}\bx\,dy\neq0$. Finally, a family whose neurons are not labelled canonically by the tree may converge for one labelling and not for another, and outside the stable region the convergence of the sets of states is open.

\emph{What the $p$-adic structure adds.} Continuum descriptions of networks are used for neural fields on Euclidean domains \cite{Amari}, \cite{Veltz-Faugeras} and, through embeddings into $L^{2}([0,1])$ by step functions, for mean-field limits of wide networks \cite{Mei et al}, \cite{Sirignano-Spiliopoulos}; moreover, $L^{2}(\Zp)$ is isometric to $L^{2}([0,1])$ \cite[Theorem 7.2]{Non-ARCH-DNNs}. For integral operators with general kernels, exact discretization is not specific to $\Zp$: on any measure space, a kernel $\bW(x,y)$ that is constant on the products $C_{i}\times C_{j}$ of the cells of a finite partition maps functions constant on the cells into themselves. For convolutional networks, whose kernels depend on $x-y$, the situation is different. If $H$ is a compact open subgroup of a locally compact abelian group and $k$, $f$ are constant on the cosets of $H$, then, by the invariance of the Haar measure, $(k\ast f)(x+h)=(k\ast f)(x)$ for $h\in H$, so $k\ast f$ is again constant on the cosets of $H$. Without such a subgroup this fails: on $\R$, which has no nontrivial compact subgroups, the convolution of two step functions is not a step function, and its support is in general larger than the supports of the factors. The group $\Zp$ has the nested chain of compact open subgroups $p^{l}\Zp$, $l\geq0$, whose cosets are the balls; hence convolutional networks admit exact discretizations at all resolutions, which are consistent (averages of averages are averages), and the radial kernels $w(|x-y|_{p})$ are at the same time convolutions and hierarchical, and act diagonally along the tree (Section \ref{Section_7}). In this sense, exact discretization of convolutional networks is specific to $\Zp$, and more generally to totally disconnected groups with a hierarchy of compact open subgroups, such as the dyadic group of Walsh analysis, which carries the same tree of cosets as $\mathbb{Z}_{2}$. The same mechanism makes finite networks exact restrictions of the continuum theory in the $p$-adic Buice--Cowan theory \cite{Zuniga-BC}; there, the continuous limit is the deterministic mean-field limit, whose validity near the critical point is controlled by a Ginzburg criterion (Section \ref{Section_BC}).

\subsection{\label{Section_landscape}$p$-Adic DNNs in the landscape of DNNs and RNNs}
Following \cite[Section 2.2]{Segadlo et al}, standard deep feedforward networks (standard DNNs) and standard recurrent networks (standard RNNs) are described by pre-activations $\bh^{(k)}\in\R^{n_{k}}$ governed by
\begin{equation}
\bh^{(k)}=\bW^{(k)}\phi(\bh^{(k-1)})+\Win^{(k)}\bx^{(k)}+\bxi^{(k)},\qquad k=1,\ldots,\Delta, \label{Net-1}
\end{equation}
with $\bh^{(0)}=\Win^{(0)}\bx^{(0)}+\bxi^{(0)}$ and output $\by=\Wout\phi(\bh^{(\Delta)})+\xiout$. For standard DNNs, $\bW^{(k)}$ is an $n_{k}\times n_{k-1}$ matrix and the weights of different layers are independent. For standard RNNs, $k$ is a time step, the weights do not depend on $k$, and the network processes a sequence of inputs $\bx^{(k)}$. We describe how the networks of this paper fit into this landscape.

\emph{Feedforward networks.} A layered $p$-adic continuous DNN is a standard DNN whose layers are continua of neurons, the open compact sets $U_{k}$ (Proposition \ref{Theorem_feedforward}); a feedforward one may also have skip connections, as residual networks. Their state is unique for all weights, which is why critical organization requires recurrence (Corollary \ref{Cor_feedforward}).

\emph{Recurrent networks.} Since the input of (\ref{State_equation}) does not depend on time, the states of a recurrent $p$-adic continuous DNN are the fixed points of a standard RNN driven by a constant input. The closest standard models are therefore equilibrium (implicit) models, whose output is defined as a fixed point: deep equilibrium models \cite{Bai et al}, \cite{Bai-MDEQ}, monotone operator equilibrium networks \cite{Winston-Kolter} and implicit deep learning \cite{El Ghaoui et al}. The condition $L_{\phi}\Vert\T_{\bW}\Vert<1$ plays the role of their well-posedness conditions, and Theorem \ref{Theorem_continuity} and Corollary \ref{Cor_robust} are the corresponding robustness statements. Outside $\X_{\mathrm{stable}}$, the multiplicity of states of Sections \ref{Section_8}--\ref{Section_7} is the regime of attractor networks, such as Hopfield networks \cite{Hopfield-1982}, \cite{Hopfield-1984}, \cite{Ramsauer et al}, in which stable states store patterns. Finally, as a continuum of neurons, (\ref{State_equation}) is a neural field equation \cite{Amari}, \cite{Veltz-Faugeras}.

\emph{$p$-Adic discrete DNNs.} Unrolling a standard RNN with constant input for $T$ steps gives a weight-tied standard DNN of depth $T$, and deep equilibrium models are its infinite-depth limit. The $p$-adic discrete DNNs are a multi-resolution version of this construction: the layer of level $l$ has $p^{l}$ neurons, and its weights are the coarse-grained kernel $\bW^{(l)}$ of a single kernel, with the factor $p^{-l}$ of mean-field scaling. The identification with the unrolled recurrent network is exact above level $l_{0}$ (Corollary \ref{Cor_Picard}), and Theorem \ref{Theorem_depth} gives the infinite-depth limit; related ideas are multiscale deep equilibrium models \cite{Bai-MDEQ} and mean-field limits of wide networks \cite{Chizat-Bach}, \cite{Mei et al}, \cite{Sirignano-Spiliopoulos}.

\emph{What is not covered.} The Gaussian-process and neural tangent kernel limits \cite{Lee et al}, \cite{Matthews et al}, \cite{Jacot et al} are limits of laws whose kernels diverge in $L^{2}(\Zp\times\Zp)$ (Section \ref{Section_continuous_limits}). A kernel $\bW(x,y)$ plays the role of a graphon, as in graphon neural networks \cite{Ruiz et al}, to which the $p$-adic setting adds exact hierarchical coarse-graining. Finally, sequence processing is not treated: the framework describes recurrent networks with time-dependent input, $\bh_{t}=\T_{\bW}(\phi\circ\bh_{t-1})+\T_{\Win}\bx_{t}+\bxi$, but all our results concern equilibria and constant inputs.

\subsection{\label{Section_criticality_summary}Critical organization}
\subsubsection{What is proved}
The results on criticality, summarized in Section \ref{Section_main_results}, are the following: the contraction threshold is not a critical point in general (Lemma \ref{Prop_counterexample}); nonnegative radial kernels have a sharp pitchfork bifurcation at total coupling $\lambda=1$ (Theorem \ref{Theorem_radial}); above the threshold, hierarchical kernels have exactly $3^{p^{l}}$ states, $2^{p^{l}}$ of them stable, ordered as a tree-indexed poset (Theorem \ref{Theorem_hierarchical}, Lemma \ref{Lemma_poset}); and critical organization requires recurrence (Corollary \ref{Cor_feedforward}). These results concern equilibria of deterministic networks; the relation with the edge-of-chaos transition of random networks is the subject of Conjectures \ref{Conj_tree} and \ref{Conj_random}.

\subsubsection{\label{Section_gradient}Critical organization and strange attractors}
It has been suggested that critical organization in neural networks is related to strange attractors. For the networks studied here this is not the case. Assume that the kernel is symmetric, $\bW(x,y)=\bW(y,x)$, that $\bu$ does not depend on time, and that $\phi$ is $C^{1}$ and strictly increasing, as for $\phi=\tanh$; radial kernels, and in particular the hierarchical kernels of Theorem \ref{Theorem_hierarchical}, are symmetric. Set $G(s)=\int_{0}^{s}r\phi^{\prime}(r)\,dr$ and
\[
E(\bh)=\int_{\Zp}G(\bh(x))\,dx-\frac{1}{2}\langle\phi\circ\bh,\T_{\bW}(\phi\circ\bh)\rangle-\langle\bu,\phi\circ\bh\rangle.
\]
Along solutions of (\ref{Dynamics}), using the symmetry of $\T_{\bW}$,
\[
\frac{d}{dt}E(\bh)=\bigl\langle(\phi^{\prime}\circ\bh)\,\partial_{t}\bh,\ \bh-\T_{\bW}(\phi\circ\bh)-\bu\bigr\rangle=-\int_{\Zp}\phi^{\prime}(\bh(x))\bigl(\partial_{t}\bh(x)\bigr)^{2}dx\leq0,
\]
with equality only at equilibria; this is the Lyapunov function of Cohen--Grossberg and Hopfield \cite{Cohen-Grossberg}, \cite{Hopfield-1984}. When the parameters are test functions of level $l$, the dynamics leaves $\D^{l}(\Zp)$ invariant, and by LaSalle's invariance principle every bounded trajectory in $\D^{l}(\Zp)$ converges to the set of equilibria. In $L^{2}(\Zp)$, for bounded $\phi$ the compactness of $\T_{\bW}$ gives a global attractor, and for gradient systems the global attractor is the union of the equilibria and their unstable manifolds \cite{Hale}. Hence, in these networks, critical organization is multistability: many fixed-point attractors, organized hierarchically, as in the associative-memory regime of Hopfield networks; there is no chaotic invariant set and no strange attractor. For the edge detector of Section \ref{Section_8}, the dynamics is a scalar equation at each point, and every trajectory converges to an equilibrium. Strange attractors require asymmetric couplings; the natural setting is that of random recurrent networks with independent, hence asymmetric, weights (Conjecture \ref{Conj_random}), where a transition to chaos is known for constant variance profiles \cite{Sompolinsky et al}. The experiments of Section \ref{Section_exp_random} agree with this picture: with symmetric couplings the trajectories converge to equilibria, whereas with asymmetric couplings they may converge to non-stationary attractors.

\subsection{\label{Section_brain}The critical brain hypothesis from the multistability perspective}
The critical brain hypothesis states that the brain operates near the boundary of a phase transition between order and chaos, where its ability to process, transmit and store information is optimal; analogous questions have been raised for artificial networks \cite{Criticall-0,Criticall-1,Criticall-2}. For the recurrent $p$-adic DNNs studied here, which have symmetric kernels and therefore gradient dynamics (Section \ref{Section_gradient}), the relevant transition is not between order and chaos but between a unique attractor and many attractors, and the results of Sections \ref{Section_4} and \ref{Section_7} give a precise picture of it, illustrated in Figure \ref{Fig_regimes}.

\emph{Three regimes.} Below the critical point there is one globally attracting state (Proposition \ref{Prop_attractor}): the network computes a function of its input, robustly and with fast relaxation, but it stores nothing. At the critical point the unique state loses stability (Theorems \ref{Theorem_radial} and \ref{Theorem_hierarchical}), and the relaxation time and the susceptibility diverge (Corollary \ref{Cor_sharp}): small inputs produce large and long-lived responses. Above the critical point the network is multistable: it stores patterns as attractors, as in Hopfield networks \cite{Hopfield-1982}, \cite{Hopfield-1984}, and an input must overcome the barriers between basins to change the state. The critical point is thus where the network is still sensitive to its inputs and starts to acquire memory; this is the multistability analogue of the boundary between order and chaos. For hierarchical kernels, criticality depends on the scale: a network can be critical at coarse levels of the tree and subcritical at fine ones (Section \ref{Section_exp_regimes}).

\emph{The landscape.} The dynamics descends the Lyapunov function $E$ of Section \ref{Section_gradient}, which has one well below the critical point, a flat well at it, and many wells above it. For hierarchical kernels, the Hasse diagram of the poset $V^{S}$ (Lemma \ref{Lemma_poset}, Figure \ref{Fig_regimes}(c)) can be read as the skeleton of this landscape: its maximal elements are the stable states, and two stable states that differ at one leaf lie above a common saddle of index one. This is an interpretation: whether the order of Theorem \ref{Theorem_hierarchical} describes exactly the connections between the equilibria by unstable manifolds is not proved here.

\emph{Dynamics on the landscape.} The order on the states of the edge detector introduced in \cite{Zambrano-Zuniga-2} was motivated by the observation that the energy landscapes of many complex systems, such as spin glasses and proteins, are organized hierarchically, and that their dynamics can be described as a random walk on the landscape; see \cite[Chapter 4]{KKZuniga}, \cite{Spin-Glasses}. The results of this paper suggest a similar picture for recurrent $p$-adic DNNs. The deterministic dynamics (\ref{Dynamics}) descends $E$ and remains in one basin. If a small noise of intensity $\varepsilon$ is added, the network jumps between minima of $E$, and, as $\varepsilon\rightarrow0$, these jumps are expected to reduce to a Markov chain on the stable states, with transitions through saddles of index one and rates of order $\exp(-\Delta E/\varepsilon^{2})$, determined by the heights $\Delta E$ of the barriers. For hierarchical kernels, Theorem \ref{Theorem_hierarchical}(v) and Lemma \ref{Lemma_poset} identify the possible transitions: between two bistable patterns that differ at one leaf. For the kernel $a\boldsymbol{P}_{l}+b\boldsymbol{P}_{0}$ the barriers are, up to $O(b)$, the same for all leaves, and the walk resembles a random walk on the hypercube $\{\pm1\}^{p^{l}}$. For radial kernels with strong couplings at several levels (Section \ref{Section_tree_indexed}), flipping a whole ball of level $k$ should cost a barrier that grows with the coupling $A_{k}$ inside it, so that the barriers, and hence the transition rates, would be organized by the tree: the effective dynamics would be an ultrametric random walk on the stable states, a discrete analogue of $p$-adic diffusion on hierarchical energy landscapes \cite{KKZuniga}. Making this precise requires a stochastic version of (\ref{Dynamics}) compatible with the Lyapunov structure of Section \ref{Section_gradient}, since the dynamics is a gradient flow of $E$ only with respect to a metric weighted by $\phi^{\prime}(\bh)$; we leave this problem for future work.

The two pictures are complementary: symmetric hierarchical networks exhibit an edge of multistability, which concerns memory, while asymmetric random networks exhibit an edge of chaos, which concerns dynamics (Section \ref{Section_gradient} and Conjecture \ref{Conj_random}).

\subsection{\label{Section_BC}Relation with the $p$-adic Buice--Cowan theory}
The Buice--Cowan theory \cite{Buice-Cowan} is a statistical field theory of neural activity, derived from a stochastic spike model, whose mean-field (zero-loop) limit is a Wilson--Cowan equation and whose critical behavior belongs to the directed percolation class. In its $p$-adic version \cite{Zuniga-BC}, the neurons are the points of a $p$-adic space and the mean-field equation is the hierarchical Wilson--Cowan equation
\[
\partial_{t}\boldsymbol{a}=-\alpha\boldsymbol{a}+f(w\ast\boldsymbol{a}+I),
\]
where $w\geq0$ is a radial kernel, $\alpha>0$, and $f$ is a nonnegative firing rate with $f(0)=0$, $f^{\prime}(0)>0$ and $f^{\prime\prime}(0)<0$. Here the activation acts on the input, while in (\ref{Dynamics}) it acts inside the integral operator. The two forms are equivalent: if $\boldsymbol{a}$ is a solution, then $\bh=w\ast\boldsymbol{a}+I$ solves $\partial_{t}\bh=-\alpha\bh+w\ast f(\bh)+\alpha I$, which, after rescaling time by $\alpha$, is (\ref{Dynamics}) with $\T_{\bW}=\alpha^{-1}w\ast$, $\phi=f$ and $\bu=I$; the equilibria correspond bijectively through $\boldsymbol{a}=\alpha^{-1}f(\bh)$; and the linearizations $-\alpha+f^{\prime}(\bh)\,w\ast$ and $-\alpha+w\ast f^{\prime}(\bh)$ have the same nonzero spectrum. Thus the mean-field theory of \cite{Zuniga-BC} is a recurrent $p$-adic DNN with a radial kernel, written in rate form.

\emph{The stable region is the subcritical phase.} Since $w\geq0$, $\Vert\T_{\bW}\Vert=\alpha^{-1}\int w$, and, since $f$ is concave on the inputs $s\geq0$, its Lipschitz constant there is $f^{\prime}(0)$. Hence
\[
q=\frac{f^{\prime}(0)\int w}{\alpha}=1-\frac{\mu}{\alpha},
\]
where $\mu=\alpha-f^{\prime}(0)\int w$ is the distance to criticality of \cite{Zuniga-BC}, and $q$ is the branching ratio of the spike model, the mean number of spikes triggered by one spike. The condition $q<1$ of Theorem \ref{Theorem_A} is the subcritical (quiescent) phase $\mu>0$, $q=1$ is the mean-field critical point, and the factor $(1-q)^{-1}$ of Corollaries \ref{Cor_robust} and \ref{Cor_sharp} is, up to the factor $\alpha^{-1}$, the mean avalanche size $1/\mu$.

\emph{The Ginzburg criterion.} The validity criterion of mean-field theory in \cite{Zuniga-BC}, a $p$-adic Ginzburg criterion, is of a different nature: it does not decide the phase, but whether the deterministic picture survives the fluctuations of the spike counts near $\mu=0$. It holds as $\mu\rightarrow0$ if and only if $N>2\sigma$, where $N$ is the dimension of the neuron space and $\sigma$ the decay exponent of the coupling up the hierarchy. In that case the divergence of $(1-q)^{-1}$ and the critical slowing down described here have the exact exponents; otherwise fluctuations change the exponents near the critical point, which is moreover shifted to $q>1$. In both theories, the decisive parameter is how the coupling is distributed over the levels of the tree, as for the couplings $A_{k}$, $\beta_{k}$ of Conjecture \ref{Conj_tree} and $b$ in Section \ref{Section_exp_hierarchical}.

\emph{Signed activations and firing rates.} The two theories describe two forms of the critical brain hypothesis. With an odd activation, as here, the transition is a pitchfork bifurcation and leads to multistability, with states indexed by sign patterns in $V^{S}$ (Section \ref{Section_brain}). With a nonnegative firing rate, the transition is from an absorbing quiescent state to an active state, and the fluctuations produce avalanches in the directed percolation class. For instance, for the kernel $a\boldsymbol{P}_{l}$ and a concave firing rate with $af^{\prime}(0)>\alpha$, each ball has a quiescent and an active equilibrium, so the equilibria are indexed by $\{0,1\}^{G_{l}}$, a Boolean lattice graded by instability, but only the state with all balls active is stable. The hierarchy then organizes the saddles, while multistability, that is, memory, requires signed activations.

\subsection{\label{Section_tree_indexed}Tree-indexed states for general radial kernels}
Let $k(y)=w(|y|_{p})$ be a radial function in $L^{2}(\Zp)$, and let $\bW(x,y)=k(x-y)$. In contrast with the general kernels $\bW(x,y)$ of Sections \ref{Section_4}--\ref{Section_6}, the operator $\T_{\bW}$ is then a convolution operator, $\T_{\bW}f=k\ast f$. Since $p^{j}\Om{j}{y}$ equals $p^{j}$ for $|y|_{p}\leq p^{-j}$ and $0$ otherwise, we have $k(y)=\sum_{j\geq0}c_{j}p^{j}\Om{j}{y}$ with
\[
c_{0}=w(1),\qquad c_{j}=p^{-j}\bigl(w(p^{-j})-w(p^{-j+1})\bigr),\quad j\geq1,
\]
and hence, by (\ref{Projection}),
\begin{align*}
\T_{\bW}f(x) & ={\textstyle\int_{\Zp}}k(x-y)f(y)\,dy=\sum_{j\geq0}c_{j}p^{j}{\textstyle\int_{\Zp}}\Om{j}{x-y}f(y)\,dy\\
& =\sum_{j\geq0}c_{j}\boldsymbol{P}_{j}f(x),
\end{align*}
that is, $\T_{\bW}=\sum_{j\geq0}c_{j}\boldsymbol{P}_{j}$, the series converging in operator norm when $\sum_{j}|c_{j}|<\infty$. The coefficients $c_{j}$ are all nonnegative if and only if $w(1)\geq0$ and $w$ is nonincreasing in $|y|_{p}$, that is, if the coupling between two neurons does not decrease as they get closer in the tree. We assume this in this subsection; then $\sum_{j}c_{j}=\int_{\Zp}k=\lambda$. Let $A_{l}=\sum_{j\geq l}c_{j}$ be the coupling inside balls of radius $p^{-l}$ and $\beta_{l}=\sum_{j<l}c_{j}$ the coupling between them. On $\D^{l}(\Zp)$, which is invariant, the state equation reads $\psi_{A_{l}}(\bh)=\sum_{j<l}c_{j}\boldsymbol{P}_{j}\tanh\bh$, and the proof of Theorem \ref{Theorem_hierarchical}(iv), with the scalar $m$ replaced by a function in $\D^{l-1}(\Zp)$ and the sup norm, shows: if $A_{l}>1$ and $\beta_{l}<b_{0}(A_{l})$, there are exactly $3^{p^{l}}$ states in $\D^{l}(\Zp)$.

\begin{conjecture}
\label{Conj_tree}Let $\phi=\tanh$, zero input, and $\T_{\bW}=\sum_{j}c_{j}\boldsymbol{P}_{j}$ with $c_{j}\geq0$, that is, a radial kernel $w(|x-y|_{p})$ with $w\geq0$ nonincreasing in $|x-y|_{p}$. If $A_{l}>1$ and $\beta_{l}<M_{A_{l}}$, the network has at least $3^{p^{l}}$ states, of which at least $2^{p^{l}}$ are linearly stable.
\end{conjecture}

The states predicted by Conjecture \ref{Conj_tree} lie in $\D^{l}(\Zp)$; determining all the states of these networks, including those outside $\D^{l}(\Zp)$, is a further problem. The experiments of Section \ref{Section_exp_cascade} show that the hypothesis $\beta_{l}<M_{A_{l}}$ is essential: when the coupling between the balls is strong, the transitions at the levels $k\geq1$ are delayed and partial, and the counts stay far below $3^{p^{k}}$ and $2^{p^{k}}$.

\subsection{Random recurrent networks on the tree}
Let $M^{(l)}$ be a $p^{l}\times p^{l}$ random matrix with independent centered Gaussian entries of variance $p^{-l}v(|I-K|_{p})$, for a nonnegative profile $v$, and consider the equilibria of $\dot{\bh}=-\bh+M^{(l)}\tanh(\bh)$. For $v$ constant this is the model of \cite{Sompolinsky et al}, \cite{Wainrib-Touboul}.

\begin{conjecture}
\label{Conj_random}Let $\rho(v)^{2}=\lim_{l\rightarrow\infty}\sum_{K\in G_{l}}p^{-l}v(|K|_{p})$. With probability tending to one as $l\rightarrow\infty$: if $\rho(v)<1$, $\bh=0$ is the only equilibrium; if $\rho(v)>1$, the number of equilibria grows exponentially in $p^{l}$.
\end{conjecture}

Conjecture \ref{Conj_random} is the analogue, for ultrametric variance profiles, of the results of \cite{Wainrib-Touboul} for $v$ constant, and the experiments of Section \ref{Section_exp_random} are consistent with it. A proof would connect the critical organization of Section \ref{Section_7} with the transition to chaos of random recurrent networks; combined with Corollary \ref{Cor_feedforward}, it would locate critical organization, in the sense of this paper, in the recurrent part of an architecture.

\subsection{Further open problems}
Besides Conjectures \ref{Conj_tree} and \ref{Conj_random}, the following questions remain open.
\begin{enumerate}
\item Determine the exact range of $b$ for which Theorem \ref{Theorem_hierarchical} holds; by Section \ref{Section_exp_hierarchical}, it is much larger than $|b|<b_{0}(a)$, and asymmetric in $b$.
\item For random recurrent networks on the tree: determine whether, when the variance profile $v$ concentrates on small $p$-adic distances, the equilibria are approximately constant on balls of a level determined by $v$ and organized by sign patterns on the corresponding tree, as in Theorem \ref{Theorem_hierarchical} (the experiments of Section \ref{Section_exp_random} do not show this effect for networks with at most $32$ neurons); and determine whether, with asymmetric couplings, these networks have chaotic attractors for $\rho(v)>1$, as for constant profiles \cite{Sompolinsky et al}, and how the variance profile affects this transition.
\item Show that, for radial kernels with strong couplings at several levels of the tree, the dynamics (\ref{Dynamics}) perturbed by a small noise reduces, as the noise vanishes, to an ultrametric random walk on the stable states, with transition rates determined by the levels of the tree (Section \ref{Section_brain}).
\item Extend the results of this paper to recurrent $p$-adic DNNs with time-dependent inputs, $\bh_{t}=\T_{\bW}(\phi\circ\bh_{t-1})+\T_{\Win}\bx_{t}+\bxi$, that is, to sequence processing (Section \ref{Section_landscape}).
\end{enumerate}

\section*{Acknowledgments}
During the preparation of this work, the author used Claude (Anthropic; model \mbox{claude-opus-5-5}) to improve the clarity and readability of the language, to propose revisions of the exposition and of some proofs, and to implement the simulation code of Section \ref{Section_experiments}. After using this tool, the author reviewed and edited the content as needed and took full responsibility for the publication. The author was partially supported by the Debnath Endowment.


\begin{thebibliography}{99}

\bibitem{A-K-S}S. Albeverio, A. Yu. Khrennikov, V. M. Shelkovich. Theory of $p$-adic distributions: linear and nonlinear models. Cambridge University Press, 2010.

\bibitem{Albeverio et al}S. Albeverio, A. Khrennikov, B. Tirozzi. $p$-Adic dynamical systems and neural networks. Math. Models Methods Appl. Sci. 9 (1999), no. 9, 1417--1437.

\bibitem{Amari}S. Amari. Dynamics of pattern formation in lateral-inhibition type neural fields. Biol. Cybern. 27 (1977), 77--87.

\bibitem{Appell-Zabrejko}J. Appell, P. P. Zabrejko. Nonlinear superposition operators. Cambridge Tracts in Mathematics 95. Cambridge University Press, 1990.

\bibitem{Bai et al}S. Bai, J. Z. Kolter, V. Koltun. Deep equilibrium models. Advances in Neural Information Processing Systems 32 (NeurIPS 2019).

\bibitem{Bai-MDEQ}S. Bai, V. Koltun, J. Z. Kolter. Multiscale deep equilibrium models. Advances in Neural Information Processing Systems 33 (NeurIPS 2020).

\bibitem{Benettin et al}G. Benettin, L. Galgani, A. Giorgilli, J.-M. Strelcyn. Lyapunov characteristic exponents for smooth dynamical systems and for Hamiltonian systems; a method for computing all of them. Part 1: Theory. Meccanica 15 (1980), 9--20.

\bibitem{Brain-1}M. Khona, I. R. Fiete. Attractor and integrator networks in the brain. Nat. Rev. Neurosci. 23 (2022), 744--766.

\bibitem{Brain-orga-1}E. Bullmore, O. Sporns. Complex brain networks: graph theory and connectomics. Nat. Rev. Neurosci. 10 (2009), 186--198.

\bibitem{Brain-orga-2}C. C. Hilgetag, A. Goulas. `Hierarchy' in the organization of brain networks. Philos. Trans. R. Soc. Lond. B Biol. Sci. 375 (2020), 20190319.

\bibitem{Buice-Cowan}M. A. Buice, J. D. Cowan. Field-theoretic approach to fluctuation effects in neural networks. Phys. Rev. E 75 (2007), 051919.

\bibitem{Chizat-Bach}L. Chizat, F. Bach. On the global convergence of gradient descent for over-parameterized models using optimal transport. Advances in Neural Information Processing Systems 31 (NeurIPS 2018).

\bibitem{Chua-Tamas}L. O. Chua, T. Roska. Cellular neural networks and visual computing: foundations and applications. Cambridge University Press, 2002.

\bibitem{Cohen-Grossberg}M. A. Cohen, S. Grossberg. Absolute stability of global pattern formation and parallel memory storage by competitive neural networks. IEEE Trans. Syst. Man Cybern. SMC-13 (1983), 815--826.

\bibitem{Criticall-0}S. Bornholdt, T. R\"{o}hl. Self-organized critical neural networks. Phys. Rev. E 67 (2003), 066118.

\bibitem{Criticall-1}M. I. Katsnelson, V. Vanchurin, T. Westerhout. Self-organized criticality in neural networks. arXiv:2107.03402.

\bibitem{Criticall-2}S. Vock, C. Meisel. Critical dynamics governs deep learning. arXiv:2507.08527.

\bibitem{Davey-Priestley}B. A. Davey, H. A. Priestley. Introduction to lattices and order. Second edition. Cambridge University Press, 2002.

\bibitem{Deuflhard}P. Deuflhard. Newton methods for nonlinear problems: affine invariance and adaptive algorithms. Springer Series in Computational Mathematics 35. Springer, 2004.

\bibitem{El Ghaoui et al}L. El Ghaoui, F. Gu, B. Travacca, A. Askari, A. Y. Tsai. Implicit deep learning. arXiv:1908.06315.

\bibitem{Hairer et al}E. Hairer, S. P. N{\o}rsett, G. Wanner. Solving ordinary differential equations I: nonstiff problems. Second edition. Springer Series in Computational Mathematics 8. Springer, 1993.

\bibitem{Hale}J. K. Hale. Asymptotic behavior of dissipative systems. Mathematical Surveys and Monographs 25. American Mathematical Society, Providence, RI, 1988.

\bibitem{Harris et al}C. R. Harris et al. Array programming with NumPy. Nature 585 (2020), 357--362.

\bibitem{Hopfield-1982}J. J. Hopfield. Neural networks and physical systems with emergent collective computational abilities. Proc. Natl. Acad. Sci. USA 79 (1982), 2554--2558.

\bibitem{Hopfield-1984}J. J. Hopfield. Neurons with graded response have collective computational properties like those of two-state neurons. Proc. Natl. Acad. Sci. USA 81 (1984), 3088--3092.

\bibitem{Jacot et al}A. Jacot, F. Gabriel, C. Hongler. Neural tangent kernel: convergence and generalization in neural networks. Advances in Neural Information Processing Systems 31 (NeurIPS 2018).

\bibitem{IEEE}L. Chen, K. Aihara. Strange attractors in chaotic neural networks. IEEE Trans. Circuits Syst. I 47 (2000), 1455--1468.

\bibitem{KKZuniga}A. Khrennikov, S. Kozyrev, W. A. Z\'{u}\~{n}iga-Galindo. Ultrametric equations and its applications. Encyclopedia of Mathematics and its Applications 168. Cambridge University Press, 2018.

\bibitem{Khrennikov-Nilson}A. Y. Khrennikov, M. Nilson. $p$-Adic neural networks. In: $p$-Adic Deterministic and Random Dynamics. Mathematics and Its Applications 574. Springer, Dordrecht, 2004.

\bibitem{Khrennikov-Tirozzi}A. Khrennikov, B. Tirozzi. Learning of $p$-adic neural networks. CMS Conf. Proc. 29, Amer. Math. Soc., Providence, RI, 2000, 395--401.

\bibitem{Laplacian1}W. A. Z\'{u}\~{n}iga-Galindo. Reaction-diffusion equations on complex networks and Turing patterns, via $p$-adic analysis. J. Math. Anal. Appl. 491 (2020), 124239.

\bibitem{Laplacian2}W. A. Z\'{u}\~{n}iga-Galindo. Ultrametric diffusion, rugged energy landscapes and transition networks. Phys. A 597 (2022), 127221.

\bibitem{Lee et al}J. Lee, Y. Bahri, R. Novak, S. S. Schoenholz, J. Pennington, J. Sohl-Dickstein. Deep neural networks as Gaussian processes. International Conference on Learning Representations (ICLR), 2018.

\bibitem{Matthews et al}A. G. de G. Matthews, J. Hron, M. Rowland, R. E. Turner, Z. Ghahramani. Gaussian process behaviour in wide deep neural networks. International Conference on Learning Representations (ICLR), 2018.

\bibitem{Mei et al}S. Mei, A. Montanari, P.-M. Nguyen. A mean field view of the landscape of two-layer neural networks. Proc. Natl. Acad. Sci. USA 115 (2018), E7665--E7671.

\bibitem{Neal}R. M. Neal. Bayesian learning for neural networks. Lecture Notes in Statistics 118. Springer, New York, 1996.

\bibitem{Non-ARCH-DNNs}W. A. Z\'{u}\~{n}iga-Galindo. Deep neural networks: a formulation via non-Archimedean analysis. J. Fourier Anal. Appl. 32 (2026), Paper No. 77. https://doi.org/10.1007/s00041-026-10290-y

\bibitem{Poole et al}B. Poole, S. Lahiri, M. Raghu, J. Sohl-Dickstein, S. Ganguli. Exponential expressivity in deep neural networks through transient chaos. Advances in Neural Information Processing Systems 29 (NeurIPS 2016).

\bibitem{Ramsauer et al}H. Ramsauer, B. Sch\"{a}fl, J. Lehner, P. Seidl, M. Widrich, T. Adler, L. Gruber, M. Holzleitner, M. Pavlovi\'{c}, G. K. Sandve, V. Greiff, D. Kreil, M. Kopp, G. Klambauer, J. Brandstetter, S. Hochreiter. Hopfield networks is all you need. International Conference on Learning Representations (ICLR), 2021.

\bibitem{Reed-Simon}M. Reed, B. Simon. Methods of modern mathematical physics. I. Functional analysis. Second edition. Academic Press, New York, 1980.

\bibitem{Roberts et al}D. A. Roberts, S. Yaida. The principles of deep learning theory: an effective theory approach to understanding neural networks. Cambridge University Press, 2022.

\bibitem{Ruiz et al}L. Ruiz, L. F. O. Chamon, A. Ribeiro. Graphon neural networks and the transferability of graph neural networks. Advances in Neural Information Processing Systems 33 (NeurIPS 2020).

\bibitem{Schoenholz-DIP}S. S. Schoenholz, J. Gilmer, S. Ganguli, J. Sohl-Dickstein. Deep information propagation. International Conference on Learning Representations (ICLR), 2017.

\bibitem{Segadlo et al}K. Segadlo, B. Epping, A. van Meegen, D. Dahmen, M. Kr\"{a}mer, M. Helias. Unified field theoretical approach to deep and recurrent neuronal networks. J. Stat. Mech. (2022), 103401.

\bibitem{Sirignano-Spiliopoulos}J. Sirignano, K. Spiliopoulos. Mean field analysis of neural networks: a law of large numbers. SIAM J. Appl. Math. 80 (2020), 725--752.

\bibitem{Slavova}A. Slavova. Cellular neural networks: dynamics and modelling. Mathematical Modelling: Theory and Applications 16. Kluwer, Dordrecht, 2003.

\bibitem{Sompolinsky et al}H. Sompolinsky, A. Crisanti, H.-J. Sommers. Chaos in random neural networks. Phys. Rev. Lett. 61 (1988), 259--262.

\bibitem{Spin-Glasses}M. M\'{e}zard, G. Parisi, M. A. Virasoro. Spin glass theory and beyond. World Scientific, 1987.

\bibitem{Taibleson}M. H. Taibleson. Fourier analysis on local fields. Princeton University Press, 1975.

\bibitem{Temam}R. Temam. Infinite-dimensional dynamical systems in mechanics and physics. Second edition. Applied Mathematical Sciences 68. Springer-Verlag, New York, 1997.

\bibitem{V-V-Z}V. S. Vladimirov, I. V. Volovich, E. I. Zelenov. $p$-Adic analysis and mathematical physics. World Scientific, Singapore, 1994.

\bibitem{Veltz-Faugeras}R. Veltz, O. Faugeras. Local/global analysis of the stationary solutions of some neural field equations. SIAM J. Appl. Dyn. Syst. 9 (2010), 954--998. arXiv:0910.2247.

\bibitem{Vershynin}R. Vershynin. High-dimensional probability. Cambridge University Press, 2018.

\bibitem{Virtanen et al}P. Virtanen et al. SciPy 1.0: fundamental algorithms for scientific computing in Python. Nature Methods 17 (2020), 261--272.

\bibitem{Villegas}P. Villegas. Strange attractors in complex networks. Phys. Rev. E 111 (2025), L042301.

\bibitem{Wainrib-Touboul}G. Wainrib, J. Touboul. Topological and dynamical complexity of random neural networks. Phys. Rev. Lett. 110 (2013), 118101.

\bibitem{Weaver}N. Weaver. Lipschitz algebras. Second edition. World Scientific, 2018.

\bibitem{Winston-Kolter}E. Winston, J. Z. Kolter. Monotone operator equilibrium networks. Advances in Neural Information Processing Systems 33 (NeurIPS 2020).

\bibitem{Yang}G. Yang. Wide feedforward or recurrent neural networks of any architecture are Gaussian processes. Advances in Neural Information Processing Systems 32 (NeurIPS 2019).

\bibitem{Zambrano-Zuniga-1}B. A. Zambrano-Luna, W. A. Z\'{u}\~{n}iga-Galindo. $p$-Adic cellular neural networks. J. Nonlinear Math. Phys. 30 (2023), 34--70.

\bibitem{Zambrano-Zuniga-2}B. A. Zambrano-Luna, W. A. Z\'{u}\~{n}iga-Galindo. $p$-Adic cellular neural networks: applications to image processing. Phys. D 446 (2023), 133668.

\bibitem{Zuniga et al}W. A. Z\'{u}\~{n}iga-Galindo, C. He, B. A. Zambrano-Luna. $p$-Adic statistical field theory and convolutional deep Boltzmann machines. PTEP. Prog. Theor. Exp. Phys. (2023), 063A01.

\bibitem{Zuniga-1}W. A. Z\'{u}\~{n}iga-Galindo. A correspondence between deep Boltzmann machines and $p$-adic statistical field theories. Adv. Theor. Math. Phys. 28 (2024), 679--741.

\bibitem{Zuniga-2}W. A. Z\'{u}\~{n}iga-Galindo, B. A. Zambrano-Luna, B. Dibba. Hierarchical neural networks, $p$-adic PDEs, and applications to image processing. J. Nonlinear Math. Phys. 31 (2024), 63.

\bibitem{Zuniga-BC}W. A. Z\'{u}\~{n}iga-Galindo. $p$-Adic statistical field theory of fluctuation effects in hierarchical neural networks. arXiv:2610.36400.

\bibitem{Zuniga-DBNs}W. A. Z\'{u}\~{n}iga-Galindo. $p$-Adic statistical field theory and deep belief networks. Phys. A 612 (2023), 128492.

\bibitem{Zuniga-Entropy}W. A. Z\'{u}\~{n}iga-Galindo, B. A. Zambrano-Luna. Hierarchical Wilson--Cowan models and connection matrices. Entropy 25 (2023), 949.

\bibitem{Zuniga-Textbook}W. A. Z\'{u}\~{n}iga-Galindo. $p$-Adic analysis: stochastic processes and pseudo-differential equations. De Gruyter, 2025.

\end{thebibliography}
\end{document}